\documentclass[11pt,letterpaper]{article}

\usepackage[top=1in,bottom=1in,left=1in,right=1in]{geometry}

\usepackage[utf8x]{inputenc} 
\usepackage[T1]{fontenc}    
\usepackage{hyperref}       

\usepackage{url}            
\usepackage{booktabs}       
\usepackage{amsfonts}       
\usepackage{nicefrac}      
\usepackage{float}
\usepackage{microtype}     
\usepackage{xcolor}        
\usepackage[english]{babel}
\usepackage{amsmath,amsfonts,amsthm,bm,amssymb,mathtools}
\usepackage{enumitem}
\usepackage[capitalize,noabbrev]{cleveref}
\usepackage{wrapfig}

\DeclareMathOperator*{\argmin}{arg\,min}
\newcommand{\sqbr}[1]{\left[ #1 \right]} 
\newcommand{\crbr}[1]{\left\{#1\right\}} 
\newcommand{\nrbr}[1]{\left( #1 \right)} 
\newcommand{\norm}[1]{\left\lVert #1 \right\rVert}

\newcommand{\lmod}{\left\lvert \left\lvert}
\newcommand{\indinf}[1]{\left\lvert \left\lvert \left\lvert #1 \right\rvert \right\rvert \right\rvert_{\infty}}
\newcommand{\rmod}{\right\rvert \right\rvert}

\def\P{{\mathbb{P}}} 
\def\E{{\mathbb{E}}}

\theoremstyle{plain}
\newtheorem{theorem}{Theorem}[section]

\newtheorem{lemma}[theorem]{Lemma}

\theoremstyle{definition}

\newtheorem{assumption}[theorem]{Assumption}
\theoremstyle{remark}

\usepackage{comment}
\usepackage{soul}

\usepackage[textsize=tiny]{todonotes}

\title{A Theoretical Study of The Effects of Adversarial Attacks on Sparse Regression}

\author{
	Deepak Maurya \\ 
	Department of Computer Science\\
	Purdue University \\ 
	West Lafayette, Indiana, USA \\ 
	\texttt{dmaurya@purdue.edu} \\
	\and
	Jean Honorio \\ 
	Department of Computer Science\\
	Purdue University \\ 
	West Lafayette, Indiana, USA \\ 
	\texttt{jhonorio@purdue.edu} \\
}
\date{}
\begin{document}
\maketitle

\begin{abstract}
This paper analyzes $\ell_1$ regularized linear regression under the challenging scenario of having only adversarially corrupted data for training. We use the primal-dual witness paradigm to provide provable performance guarantees for the support of the estimated regression parameter vector to match the actual parameter. Our theoretical analysis shows the counter-intuitive result that an adversary can influence sample complexity by corrupting the irrelevant features, i.e., those corresponding to zero coefficients of the regression parameter vector, which, consequently, do not affect the dependent variable. As any adversarially robust algorithm has its limitations, our theoretical analysis identifies the regimes under which the learning algorithm and adversary can dominate over each other. It helps us to analyze these fundamental limits and address critical scientific questions of which parameters (like mutual incoherence, the maximum and minimum eigenvalue of the covariance matrix, and the budget of adversarial perturbation) play a role in the high or low probability of success of the LASSO algorithm. Also, the derived sample complexity is logarithmic with respect to the size of the regression parameter vector, and our theoretical claims are validated by empirical analysis on synthetic and real-world datasets.
\end{abstract}

\section{Introduction}
\label{sec:intro}
A well-known instance of the failure of machine learning (ML) models is when they are confronted with adversarial attacks. The vulnerability of ML models to possibly small perturbations  imperceptible to the human eye in input features such as one-pixel attacks \cite{su2019one}  may produce inaccurate predictions with high confidence \cite{szegedy2013intriguing, goodfellow2014explaining, madry2017towards}. This challenges the practical utility of ML models for critical applications demanding precisely correct predictions such as medical diagnosis \cite{lanfredi2019adversarial}, biometric verification systems \cite{fredrikson2015model}, object detection in autonomous driving \cite{eykholt2018robust}. Hence, demystifying empirical failure with theoretical analysis to design certifiable learning algorithms has been an active area of research recently. 

Most of the well-analyzed theoretical defense mechanisms \cite{zhai2019adversarially, yin2019rademacher, cohen2019certified} and empirical approaches \cite{carlini2017towards, wong2018provable, raghunathan2018semidefinite, kurakin2018adversarial} are focused on analyzing the generalization of adversarial loss functions assuming uncorrupted data is available for training. These
approaches employ the adversarial training method to improve the predictions with perturbed inputs, which may cause a diminished performance for unperturbed inputs \cite{carmon2019unlabeled, javanmard2020precise}. 
Taking the adversarial training paradigm forward, 
we consider a more challenging and interesting problem: learning from adversarially corrupted
data with \emph{no access} to noise-free measurements of regressors and dependent variable in a sparse linear regression model. 

Some of the existing works in the literature analyze adversarial loss functions in an epsilon radius ball centered around noise-free measurements \cite{xing2021adversarially, yin2019rademacher, awasthi2020adversarial,qin2021random}. Carrying this idea forward, we study the effect of such epsilon ball perturbation \cite{tong2018adversarial, balda2019perturbation, gilmer2019adversarial} on sample complexity. 
To elaborate on this, consider a sparse linear regression model which has a few non-zero coefficients in the parameter vector. 
It may seem obvious for an adversary to spend its limited epsilon budget per sample on influencing the regressors corresponding to non-zero coefficients in the regression parameter vector, as only those perturbations will affect the dependent variable. A few recent papers in the literature \cite{raghunathan2018certified,goodfellow2014explaining} have also recommended this, which design adversarial perturbation by focusing on maximizing the loss function only and ignoring the underlying sparsity in the model. But our analysis (Lemma \ref{lem:m1_bound}) shows the counter-intuitive result that an adversary can influence the sample complexity by affecting the irrelevant features (i.e., those corresponding to zero coefficients of the regression parameter vector) as well. 

Our work also examines the effect of various adversarial parameters on sample complexity under minimal assumptions on adversarial perturbation. We assume that the adversary can design perturbations (e.g., a random sub-Gaussian vector) whose covariance matrix may not necessarily be diagonal and may be dependent on the data. Our theoretical analysis shows the dependence of sample complexity on both the minimum and maximum eigenvalue of the covariance matrix of adversarial perturbations. In fact, our study shows that the adversary may increase the number of required samples for successful support recovery by decreasing the minimum eigenvalue (Lemma \ref{lem:sampleHess_psd} or Eq. \eqref{eq:mineig_Ex}) and/or increasing the maximum eigenvalue of the covariance matrix of adversarial perturbation (Theorem \ref{thm:main} or Eq. \eqref{eq:lambda_f1}). More interestingly, the adversary can also increase the sample complexity by designing the covariance matrix of adversarial perturbations such that its eigenvector corresponding to the maximum eigenvalue is parallel to the regression parameter vector (discussed after Eq. \eqref{eq:lambda_f1}).

\paragraph{Our Contributions:}
Our key contributions are summarized below: 
\begin{itemize}[leftmargin=0.45 cm]
	\setlength\itemsep{-1 pt}
	\item \textbf{Novel Problem Formulation:} To the best of our knowledge, we are the first to define the support recovery problem for a sparse linear regression model under adversarial attacks assuming the availability of corrupted data for model learning. To this end, we define a novel generative model for adversarial training data.
	\item \textbf{Identifying fundamental limitations:} As any robust algorithm has limitations, our theoretical analysis derives the conditions under which the adversary can dominate the learning algorithm. We also derive the conditions in which the learning algorithm can mitigate the adversary despite its attack of any malicious form.
	\item \textbf{Sample Complexity:} The support recovery problem is well explored for the non-adversarial regime, 
	where \cite{wainwright2009sharp} showed the sample complexity of  $\Omega(k \log(p))$, where $p$ denotes the size of the regression parameter vector and $k$ denotes the number of non-zero entries. In this work, we extend the primal-dual witness paradigm \cite{ravikumar2010high,ravikumar2011high, ravikumar2009sparse,daneshmand2014estimating} to a model 
	under adversarial attacks. 
	In this novel problem, we show the sample complexity of $ \Omega\left(k^2 \log(p)\right)$ (Theorem \ref{thm:main}). If we assume the adversarial perturbation to be Gaussian instead of sub-Gaussian, the sample complexity improves to  $ \Omega\left(k \log(p)\right)$ under the adversarial setting (Appendix \ref{sec:proof_Gaussadv}).
	\item \textbf{Theoretical tools:} Our contribution can be seen as a first step  towards the study of learning from adversarial training data. As a byproduct, we also obtain several technical results related to a new concentration inequality (Theorem \ref{thm:B_2}) and a projection matrix (Lemma \ref{lem:Proj_mat}), which could be useful for other problems.
	\item \textbf{Empirical validation:} We also verify our theoretical findings through experiments on synthetic and real-world data where the adversarial perturbation can be dependent on the data. 
\end{itemize}

\section{Preliminaries and Notation}
\label{sec:prelim}
In this section, we briefly discuss the existing work formally. Let $\left(\mathbf{x}^{\star(j)}, y^{\star(j)}\right)$ denote the $j^{th}$ \textit{noise-free} 
sample coming from the linear model: $y^{\star(j)} = \mathbf{w}^{\star \intercal}\mathbf{x}^{\star(j)}$ where $y^{\star(j)}  \in \mathbb{R}$, $\mathbf{x}^{\star(j)} \in \mathbb{R}^p$, and $\mathbf{w}^{\star} \in \mathbb{R}^p$, where $\mathcal{S} = \mathcal{S}(\mathbf{w}^{\star})$ denotes the support of $\mathbf{w}^{\star}$, i.e., the set of indices corresponding to non-zero entries of $\mathbf{w}^{\star}$. Let $k = |\mathcal{S}|$ denote the cardinality of the support and hence $p - k = |\mathcal{S}^c|$. \cite{wainwright2009sharp} considers the following model where \textit{only} $y^{\star(j)} $ is corrupted 
\begin{align}
y^{(j)}  = y^{\star(j)} + e_y^{(j)}
\end{align}
where $e_y$ is the zero-mean noise. Given $n$ samples of corrupted  ${y}^{(j)}$ and uncorrupted ${\mathbf{x}}^{\star (j)}$ for $j \in [n]$, the regression parameter vector is estimated by solving the least absolute shrinkage and selection operator (LASSO) problem: 
\begin{align}
\hat{\mathbf{w}} = \argmin_{\mathbf{w} \in \mathbb{R}^p} \frac{1}{2n} \lmod\mathbf{y} - \mathbf{X}^{\star} \mathbf{w}\rmod_2^2 +  \lambda ||\mathbf{w}||_1 \label{eq:lasso}
\end{align}
where $\lambda > 0$ is the regularization parameter, $\mathbf{y} \in \mathbb{R}^n$ and $\mathbf{X}^{\star} \in \mathbb{R}^{n \times p}$ are a collection of $n$  samples of $y^{(j)}$ and $\mathbf{x}^{\star(j)} $ for $j \in [n]$,  respectively.  In the support recovery problem, we are interested in theoretical guarantees for $\hat{\mathbf{w}} $ to be a unique solution of the optimization problem in Eq. \eqref{eq:lasso}, $\mathcal{S}(\hat{\mathbf{w}}) = \mathcal{S}(\mathbf{w}^{\star})$, and $\text{sign}(\hat{\mathbf{w}}_i) = \text{sign}(\mathbf{w}^{\star}_i)$ for $i \in \mathcal{S}$. The minimum number of samples required to ensure these properties is $n = \Omega(k \log(p))$ \cite{wainwright2009sharp}. The model discussed in this section is a special case of our novel proposed model in the next section (Eq. \eqref{eq:ex_assum} with $\mathbf{e}_\mathbf{x} = \mathbf{0}$). Our goal in this paper is to derive the sample complexity for support recovery of a more generalized model. Before proceeding on to that discussion, we introduce the notations used in this paper. 

\textbf{Notation}: We use a lowercase letter, e.g., $a$ to denote a scalar, a lowercase  bold letter such as $\mathbf{a}$ to denote a vector, and an uppercase bold letter such as $\mathbf{A}$ to denote a matrix. A vector $\mathbf{1}_m$ or $\mathbf{0}_m$ represents a vector of ones or zeros respectively, of size $m$. We denote a set with calligraphic alphabet, e.g., $\mathcal{P}$. Also, $[n]$ denotes the set $\left\{1,2,\ldots, n\right\}$. For a vector, $\mathbf{a}_i$ denotes the $i^{\text{th}}$ entry of the vector $\mathbf{a}$. For a matrix $\mathbf{A} \in \mathbb{R}^{p \times q}$, we represent the sub-matrix with rows $\mathcal{P} \subseteq [p]$ and columns $\mathcal{Q} \subseteq [q]$ as $\mathbf{A}_{\mathcal{P}\mathcal{Q}} \in \mathbb{R}^{|\mathcal{P}| \times |\mathcal{Q}|}$. For a vector vector $\mathbf{a} \in \mathbb{R}^m$, we denote the $\ell_p$ norm as $\lmod \mathbf{a}\rmod_{p} = \left(\sum_{i = 1}^{m} |\mathbf{a}_i|^p\right)^{\frac{1}{p}}$. Similarly $\lmod \mathbf{A}\rmod_{\infty}$ denotes the entrywise $\ell_{\infty}$ norm of a matrix $\mathbf{A}$ and $\lmod \mathbf{A}\rmod_{\infty}$ denotes the spectral norm. We denote the induced $\ell_{\infty}$ norm using $\lmod \left| \mathbf{A}\right|\rmod_{\infty} = \max_{i \in [m]} \lmod \mathbf{a}_i\rmod_1$ for a matrix $\mathbf{A} \in \mathbb{R}^{m \times k}$, where $\mathbf{a}_i$ denotes its $i^{\text{th}}$ row. Minimum and maximum eigenvalues of a matrix $\mathbf{A}$ are denoted by $\Lambda_{\text{min}} \left(\mathbf{A}\right)$, $\Lambda_{\text{max}} \left(\mathbf{A}\right)$. A function  $f(m) = \Omega(g(m))$ implies that there exists a constant $c_1$ such that $f(m) \geq c_1g(m), \forall m\geq m_0$. Similarly, $f(m) = \mathcal{O}(g(m))$ denotes that there exists a constant $c_2$ such that $f(m) \leq c_2g(m), \forall m\geq m_0$.  For a vector $\mathbf{w} \in \mathbb{R}^p$, $\mathcal{S}(\mathbf{w}) = \{i \in [p],\mathbf{w}_i \neq 0 \}$ denotes the support of $\mathbf{w} $ and similarly $\mathcal{S}^c(\mathbf{w}) = [p] \setminus \mathcal{S}(\mathbf{w}) $ denotes the non-support. 
\section{A Novel Generative Model For Adversarial Training Data}
\label{sec:novelgen}
In this section, we propose a novel generative model for adversarial training. In our model, we assume the adversary has attacked the system and disturbed the noise-free features $\mathbf{x}^{\star(j)}$ as shown below:
\begin{align}
{y}^{(j)}  = {y}^{\star (j)} + e_y^{(j)}, \qquad 
{\mathbf{x}}^{(j)}  = {\mathbf{x}}^{\star(j)} + \mathbf{e}^{(j)}_{\mathbf{x}}\label{eq:ex_assum}
\end{align} 
We assume only adversarially corrupted data, that is $n$ independent samples of $\left\{{y}^{(j)}, {\mathbf{x}}^{(j)} \right\}$ for $j \in [n]$ are available for training, which makes the problem challenging as compared to existing works which assume availability of uncorrupted features $\mathbf{x}^{\star (j)}$. Let $\mathbf{X} \in \mathbb{R}^{n \times p}$ be the collection of $n$ samples of $\mathbf{x}^{(j)}$ for $j \in [n]$. For brevity, we may drop the superscript $(j)$ later. Let the population covariance matrix of $\mathbf{X}^{\star}$ be denoted by $\mathbf{\Sigma}$. We assume that $\frac{\mathbf{x}^{\star}_i}{\mathbf{\sqrt{\Sigma}}_{ii}}$ is a zero mean sub-Gaussian random variable with  variance proxy parameter $\sigma^2$ for $i \in [p]$. We assume $e_y$ is a zero mean sub-Gaussian random variable with  variance proxy parameter $\sigma^2_{e_y}$.  

Proceeding to adversarial perturbation, it should be obviously bounded by some budget for each sample \cite{zhai2019adversarially, balda2019perturbation, yin2019rademacher, cohen2019certified} so that the underlying model is learnable. We model this fixed budget to be a parameter in our analysis, and hence the performance guarantees can be analyzed as a function of this parameter if interested. As the adversarial perturbation is bounded, it can be considered as sub-Gaussian random variable.  Also, it can be dependent on uncorrupted regressors, $\mathbf{x}^{\star}$ for each sample. To clarify,  $\mathbf{e}^{(i)}_{\mathbf{x}}$ can be dependent on uncorrupted regressors $\mathbf{x}^{\star (i)} $, where $i$ denotes $i^{\text{th}}$ sample. But $\mathbf{e}^{(i)}_{\mathbf{x}}$ is independent of $\mathbf{x}^{\star(j)} $, if $i \neq j$, where $j$ denotes another sample. Note that the adversary does not have control over uncorrupted regressors, $\mathbf{x}^{\star}$ but has access to $\mathbf{x}^{\star}$, which can be used to design $\mathbf{e}_{\mathbf{x}}$ in an arbitrary manner (that may make the support recovery problem most challenging).
\subsection{Notations and Assumptions}
\label{sec:assumpt}
In this section, we define a few quantities and briefly discuss their utility or effect on sample complexity. For example, assumption \ref{assum:Sigma_Xpsd} and \ref{assum:mutual_incoh} are critical to guarantee uniqueness of solution and identifiability in support recovery problem, as seen in the literature \cite{wainwright2009sharp, ravikumar2011high,ravikumar2010high,ravikumar2009sparse,daneshmand2014estimating}. We start the discussion on the first assumption, which avoids trivial cases to make the support recovery problem more challenging. 

It may look obvious for an adversary to spend its fixed budget for a sample to disturb the support entries ($\mathbf{x}_{\mathcal{S}}$) only since $y^{\star(j)} = \mathbf{w}^{\star \intercal}\mathbf{x}^{\star(j)} =  \mathbf{w}_{\mathcal{S}}^{\star \intercal}\mathbf{x}_{\mathcal{S}}^{\star(j)}$. We assume the adversary attacks the non-support entries ($\mathbf{x}_{\mathcal{S}^c}$) as well, to make the learning task of estimating $\mathcal{S}$ and $\mathcal{S}^c$ tougher, which is formalized in the following lemma.
\begin{lemma}
	\label{lem:zero_mean}
	If the adversary attacks only the support entries ($\mathbf{x}_{\mathcal{S}}$) or the non-support entries ($\mathbf{x}_{\mathcal{S}^c}$) with non-zero mean of adversarial perturbation, then the learner can guess the support trivially with probability at least  $1 - \mathcal{O}\left( \frac{1}{p} \right)$ if $n = \Omega\left(\log (p)\right)$. 
\end{lemma}
The proof of the above lemma relies on the fact that the sample mean of feature vector $\mathbf{x}^{(j)}$ for $j \in [n]$ will be close to the population mean. If only the support is attacked, the population mean is away from zero for entries in the support, and is zero for entries in the non-support.  Hence, the learner can guess the support by just computing the sample mean. In order to avoid such a trivial case, we consider the adversarial perturbation to be zero-mean in the next assumption. 
\begin{assumption}(Adversarial Perturbation)
	\label{assum:ex_subG} 
	$\mathbf{e_x}$ is a zero mean  sub-Gaussian random vector with parameters $(\mathbf{\Sigma_e}, r^2)$, where $\mathbf{\Sigma_e}$ is the population covariance matrix of the adversarial perturbation ($\mathbf{e}_{\mathbf{x}}$) and $r^2$ is the variance proxy parameter 
	for some $r \in \mathbb{R}$.
\end{assumption}
In the above definition, we require that $\mathbb{E}\left(\exp\left(\bm{\alpha}^{\intercal} \mathbf{e_x}\right)\right)\leq \exp\left(\frac{ \bm{\alpha}^{\intercal}\mathbf{\Sigma}_{\mathbf{e}} \bm{\alpha} r^2 }{2}\right)$, for all $\bm{\alpha} \in \mathbb{R}^p$ which is akin to the classical definition \cite{hsu2012tail} as discussed briefly in Appendix \ref{appn:ex_subG}.

Let the population covariance of $\mathbf{X}^{\star}$ be denoted by $\mathbf{\Sigma} = \mathbb{E}\left(\mathbf{X}^{\star \intercal} \mathbf{X}^{\star}\right)$ and the sample covariance matrix of $\mathbf{X}$ be denoted by $\hat{\mathbf{\Sigma}} = \frac{1}{n}\mathbf{X}^\intercal\mathbf{X}$. For the uniqueness of the solution to the problem stated in Eq. \eqref{eq:res_lasso}, we need a submatrix of the sample covariance matrix to be positive definite. But as $\mathbf{X}$ is assumed to be random, we assume the population covariance matrix of $\mathbf{X}^{\star}$ is positive definite as done in literature \cite{wainwright2009sharp, ravikumar2011high,ravikumar2010high,ravikumar2009sparse,daneshmand2014estimating}.
\begin{assumption}(Positive Definiteness)
	\label{assum:Sigma_Xpsd}
	The minimum eigenvalue of the population covariance matrix of $\mathbf{x}^{\star}$ fulfills $\lambda_{\text{min}}\left( \mathbf{\Sigma} \right) = C_{\text{min}} > 0 $. 
\end{assumption}

Similarly, the minimum eigenvalue of the population covariance matrix of adversarial perturbations in support $\mathcal{S}$ is denoted by $D_{\text{min}} \geq 0$. More perturbation in $\mathcal{S}$ compared to $\mathcal{S}^c$ may lead to a larger value of  $D_{\text{min}} $, which helps the learner easily ensure the uniqueness of the solution compared to smaller values of $D_{\text{min}} $. This counter-intuitive result is  discussed in sub-section \ref{sec:uniq}. The minimum eigenvalue of $\mathbf{ \Sigma}_{\mathbf{ex}}$ is denoted by $F_{\text{min}}$, which influences the sample complexity as discussed in Lemma \ref{lem:m2_bound}. Similarly, the maximum eigenvalues of matrices $\{\mathbf{\Sigma}, \mathbf{\Sigma_e}, \mathbf{ \Sigma}_{\mathbf{ex}}\}$ are denoted by $\{C_{\max}, D_{\max}, F_{\max}\}$, whose influence on sample complexity can be seen in Lemma \ref{lem:Xe_ey_lambda}. Further, we make the assumption on mutual incoherence, which basically implies that the regressors in the non-support $\mathcal{S}^c$ do not have a strong correlation with regressors in the support $\mathcal{S}$. 
\begin{assumption}(Mutual Incoherence)
	\label{assum:mutual_incoh}
	For some $\gamma \in (0,1]$, 
		$\left| \lmod \mathbf{\Sigma}_{\mathbf{x}_{\mathcal{S}^c} \mathbf{x}_\mathcal{S}} \left( \mathbf{\Sigma}_{\mathbf{x}_\mathcal{S} \mathbf{x}_\mathcal{S}} \right)^{-1} \rmod\right|_{\infty} \leq 1-\gamma$. 
\end{assumption}
This assumption is not restrictive and has been used in various works related to support recovery \cite{wainwright2009sharp, ravikumar2011high,ravikumar2010high,ravikumar2009sparse,daneshmand2014estimating}. In addition, we also assume $\indinf{ \nrbr{\mathbf{\Sigma}_{\mathbf{x}_\mathcal{S} \mathbf{x}_\mathcal{S}}}^{-1}} = G_{\text{max}} \leq \infty$ 
and $\indinf{\mathbf{\Sigma}_{\mathbf{x}_{\mathcal{S}^c} \mathbf{x}_\mathcal{S}}} = H_{\max} \leq \infty$. 

We refer to $(r,\gamma,F_{\text{min}},F_{\text{max}},D_{\text{min}},D_{\text{max}},G_{\text{max}},H_{\text{max}})$ as the budget given to the adversary for corrupting each sample. With a brief discussion of our assumptions in this section, we present our main theoretical result in the next section. 

\section{Our Theoretical Analysis}
\label{sec:main}
In this section, we discuss the main theorem for the proposed support recovery problem mentioned in the previous section. In a nutshell, the theorem implies a sample complexity of $n = \Omega\left( k^2 \log(p)\right)$ for correct support recovery with high probability. We provide formal and intuitive implications of the proposed theorem first and further present its proof. 

The Lasso problem under the adversarial setting can be stated as:
	\begin{align}
	\hat{\mathbf{w}} = \argmin_{\mathbf{w} \in \mathbb{R}^p} l(\mathbf{w}) + \lambda ||\mathbf{w}||_1 = \argmin_{\mathbf{w} \in \mathbb{R}^p} \frac{1}{2n} ||\mathbf{y} - \mathbf{X} \mathbf{w}||_2^2 +  \lambda ||\mathbf{w}||_1 \label{eq:res_lasso}
	\end{align}
Our main theoretical result for the above problem is as follows.
\begin{theorem}
	\label{thm:main}
	If $n = \Omega(k^2 \log(p))$, assumption \ref{assum:ex_subG}, \ref{assum:Sigma_Xpsd}, and  \ref{assum:mutual_incoh} hold, and 

		\begin{align*}
		\lambda \geq \max\crbr{\frac{16b}{\gamma}, \frac{q_1 \sigma_{e_y}}{\gamma} \sqrt{\frac{2\log(p)}{n}}, \frac{16q}{\gamma}\sqrt{\frac{4\log(p)}{n}}} 
		\end{align*}
	where $q= r\sqrt{ \mathbf{w}^{\star\intercal}_{\mathcal{S}} \mathbf{\Sigma}_{\mathbf{e}_{\mathcal{S} \mathcal{S}}} \mathbf{w}^{\star}_{\mathcal{S}}}   \max\limits_{i \in [p]} \nrbr{\sigma \sqrt{\mathbf{ \Sigma}}_{ii} + r \sqrt{\mathbf{ \Sigma_e}}_{ii}}, \; b = \norm{\mathbf{\Sigma}_{\mathbf{x} \mathbf{e}_{\mathcal{S}}} \mathbf{w}^{\star}_{\mathcal{S}}  }_{\infty}, \; q^2_1 = 3  \nrbr{ C_{\max} + 2F_{\max} + D_{\max} }$,	then we claim the following with probability of at least $1 - \mathcal{O}\left( \frac{1}{p} \right)$
	\begin{enumerate}[leftmargin=0.7 cm]
		\setlength\itemsep{-1 pt}
		\item The true support is recovered, i.e., $\mathcal{S}(\hat{\mathbf{w}}) \subseteq \mathcal{S}(\mathbf{w}^{\star})$ or equivalently $\hat{\mathbf{w}}_i = 0, \forall i \notin \mathcal{S}(\mathbf{w}^{\star})$
		\item $\hat{\mathbf{w}}_{\mathcal{S}} $ is the unique solution for the Lasso problem stated in Eq. \eqref{eq:res_lasso}. 
		\item The estimated parameter vector satisfies 
		\begin{align}
		\norm{\hat{\mathbf{w}}_{\mathcal{S}}- \mathbf{w}_{\mathcal{S}}^{\star}}_{\infty} \leq \lambda \nrbr{1 + \frac{\gamma}{4}} \frac{3G_{\max} }{2}  = f(\lambda) \label{eq:thm3rdclaim}
		\end{align}
		\item $\hat{\mathbf{w}}_i \neq 0$ and furthermore $\text{sign}(\hat{\mathbf{w}}_i) = \text{sign}(\mathbf{w}^{\star}_i), \forall i \in \mathcal{S}(\mathbf{w}^{\star}) $ if $\min\limits_{i \in \mathcal{S}}\left|\mathbf{w}_i^{\star}\right| \geq 2 f\left(\lambda\right) $.
		\item Additionally if $\mathbf{\Sigma}_{\mathbf{x} \mathbf{e}_{\mathcal{S}}} = \mathbf{0}_{p \times k}$, and therefore $b = 0$, then $\lambda \geq \mathcal{O}\nrbr{\sqrt{\frac{\log(p)}{n}}}$, statement 1, 2, and 4 still hold true and statement 3 is modified to
		\begin{align*}
		\norm{\hat{\mathbf{w}}_{\mathcal{S}}- \mathbf{w}_{\mathcal{S}}^{\star}}_{\infty} \leq f(\lambda) = \mathcal{O}\nrbr{\sqrt{\frac{\log(p)}{n}}}
		\end{align*}
	\end{enumerate}
\end{theorem}
The first two claims of the above theorem imply that we can uniquely recover the true support with high probability, assuming we have a sufficient number of samples if we choose a regularization parameter $\lambda$ greater than a certain threshold. Note that choosing a very large value of $\lambda$ is not desirable as it would also increase the upper bound for $\lmod \hat{\mathbf{w}}_{\mathcal{S}}  - \mathbf{w}_{\mathcal{S}}^{\star} \rmod_{\infty}$, as per the above theorem (Eq. \eqref{eq:thm3rdclaim}). The fourth claim of the theorem states that the minimum magnitude among the support entries of the regression parameter vector should be greater than a certain function of $\lambda$ for correct sign recovery. 

Note that the adversary can increase the lower bound of $\lambda$. For example, the adversary can increase $b$, which can increase $\lambda$, resulting in higher value of $f(\lambda)$ in Eq. \eqref{eq:thm3rdclaim}. This may lead to violation of requirement in statement 4, $\min_{i \in \mathcal{S}}\left|\mathbf{w}_i^{\star}\right| \geq 2 f\left(\lambda\right) $, if the smallest entry in $\mathbf{w}^{\star}_{\mathcal{S}}$ is not large enough. In statement 5 with $b = 0$, such situation can be avoided by increasing the value of $n$ appropriately despite any efforts from adversary. This helps us to identify different regimes under which we can provide theoretical guarantees of LASSO for successful support recovery under adversarial attack and also the case which may be favorable to adversary. This is discussed in detail towards the end of section \ref{sec:w_bound}. 

\paragraph{Proof Sketch.} We use a constructive proof technique: primal-dual witness (PDW) \cite{wainwright2009sharp, ravikumar2011high,ravikumar2010high,ravikumar2009sparse,daneshmand2014estimating} method to prove Theorem \ref{thm:main}. 
The proof outline is summarized below:
\begin{itemize}[leftmargin=0.25 cm]
	\setlength\itemsep{0 pt}
	\item The PDW framework starts by allowing us to find sufficient conditions to estimate the elements of the non-support ($\mathcal{S}^c$) first by ensuring strict dual feasibility (Section \ref{sec:kkt_cond}). This step ensures that we correctly recover the zeroes, i.e., $\hat{\mathbf{w}}_i = 0$ for all $i \notin \mathcal{S}$. 
	This establishes the first claim of exact support recovery in Theorem \ref{thm:main}. 
	\item We then derive the sufficient conditions for uniqueness of $\hat{\mathbf{w}}_{\mathcal{S}}$ (Section \ref{sec:uniq}), 
	which proves the second claim of Theorem \ref{thm:main}. 
	\item The goodness of the estimated parameter vector is proved in Section \ref{sec:w_bound} by deriving an upper bound on $\lmod \hat{\mathbf{w}}_{\mathcal{S}}  - \mathbf{w}_{\mathcal{S}}^{\star} \rmod_{\infty} $  to justify the third claim of Theorem \ref{thm:main}.
	\item Armed with second and third claim, we prove the fourth claim of Theorem \ref{thm:main} for correctly recovering the non-zeros, i.e., $\hat{\mathbf{w}}_i \neq 0$ for all $i \in \mathcal{S}(\mathbf{w}^{\star})$.
\end{itemize}
\subsection{Exact Support Recovery}
\label{sec:kkt_cond}
In this subsection, we verify the first-order stationary, complementary slackness and strict dual feasibility conditions for the optimal solution $\hat{\mathbf{w}}_{\mathcal{S}} $. The first order stationarity condition are (algebraic computation in Appendix \ref{appn:lossfn_0,1,2stcond}):  
\begin{align}
\nabla l((\hat{\mathbf{w}}_{\mathcal{S}}, \mathbf{0})) + \lambda\hat{\mathbf{{z}}} = \mathbf{0}_{p\times 1} \label{eq:1st_ord}
\end{align}
where $\hat{\mathbf{{z}}} \in \partial\lmod \hat{\mathbf{w}} \rmod_1$ belongs to the sub-differential set of the $\ell_1$ norm at $\hat{\mathbf{w}}$. In the context of the  primal-dual witness framework \cite{ravikumar2010high,ravikumar2011high, ravikumar2009sparse,daneshmand2014estimating}, $\hat{\mathbf{w}}_{\mathcal{S}}$ and $\hat{\mathbf{{z}}} $ are referred as the primal and dual variables respectively. As $\hat{\mathbf{{z}}} $ belongs to the sub-differential set of the $\ell_1$ norm, we can claim that $\lmod \hat{\mathbf{{z}}}  \rmod_{\infty} \leq 1$ by norm duality but for strict dual feasibility we need $\lmod \hat{\mathbf{{z}}}_{\mathcal{S}^c}  \rmod_{\infty} < 1$ as stated in Lemma 1 of \cite{wainwright2009sharp}. In order to ensure this condition, we need to first derive $\hat{\mathbf{{z}}}_{\mathcal{S}^c} $ from the first order stationary condition in Eq. \eqref{eq:1st_ord} which is a $p-$dimensional vector equation and can be written for elements in $\mathcal{S} $ and $\mathcal{S}^c$ separately to derive $\hat{\mathbf{{z}}}_{\mathcal{S}^c}$. The final expression is presented here, whose derivation can be seen in Appendix \ref{appn:lossfn_0,1,2stcond}: 
	\begin{align}
	\hat{\mathbf{{z}}}_{\mathcal{S}^c}  &= \hat{\mathbf{{z}}}_{\mathcal{S}^c_{t_1}}  + \hat{\mathbf{{z}}}_{\mathcal{S}^c_{t_2}}, \quad \hat{\mathbf{{z}}}_{\mathcal{S}^c_{t_1}}  = \mathbf{X}_{\mathcal{S}^c}^{\intercal} 	\mathbf{X}_{\mathcal{S}} \left(\mathbf{X}_{\mathcal{S}}^{\intercal} \mathbf{X}_{\mathcal{S}} \right)^{-1} \hat{\mathbf{{z}}}_{\mathcal{S}}  \label{eq:Zsc_def} \\
	\hat{\mathbf{{z}}}_{\mathcal{S}^c_{t_2}} & = \mathbf{X}_{\mathcal{S}^c}^{\intercal}  \left({\mathbf{P}}/{\lambda n}\right) \left( \mathbf{e}_y - \mathbf{E_{x_{\mathcal{S}}}} \mathbf{w}_{\mathcal{S}}^{\star} \right) \label{eq:Zsct2_def} \\
	\mathbf{P} &= \left( \mathbf{I}_n - \mathbf{X}_{\mathcal{S}} \left(\mathbf{X}_{\mathcal{S}}^{\intercal} \mathbf{X}_{\mathcal{S}} \right)^{-1} \mathbf{X}_{\mathcal{S}}^{\intercal}  \right) \label{eq:projmat_def}
	\end{align}
where $\mathbf{I}_n$ represents an identity matrix of dimension $n \times n$, $\mathbf{E_{x}}$ is $n \times p$ matrix containing $n$ samples of adversarial perturbation. We have decomposed $\hat{\mathbf{{z}}}_{\mathcal{S}^c} $ in two terms $\hat{\mathbf{{z}}}_{\mathcal{S}^c_{t_1}}  $ and $\hat{\mathbf{{z}}}_{\mathcal{S}^c_{t_2}}  $ to bound them separately in next two sub-sections. 
\subsubsection{Analyzing adversarial attack on support ($\mathcal{S}$) and non-support ($\mathcal{S}^c$)}
\label{sec:Zsct1_b}
After applying the sub-multiplicative property of norms to  $\hat{\mathbf{{z}}}_{\mathcal{S}^c_{t_1}} $ in Eq. \eqref{eq:Zsc_def}, and using the fact $\lmod \hat{\mathbf{{z}}}  \rmod_{\infty} \leq 1$: 
\begin{align*}
\lmod\hat{\mathbf{{z}}}_{\mathcal{S}^c_{t_1}}   \rmod_{\infty} \leq \lmod \left| \mathbf{X}_{\mathcal{S}^c}^{\intercal}  \mathbf{X}_{\mathcal{S}} \left(\mathbf{X}_{\mathcal{S}}^{\intercal}  \mathbf{X}_{\mathcal{S}} \right)^{-1} \right|\rmod_{\infty}
\end{align*}
Let $\mathbf{R} = \frac{1}{n} \mathbf{X}_{\mathcal{S}^c}^{\intercal} \mathbf{X}_{\mathcal{S}}$ and $\mathbf{Q} = \frac{1}{n} \mathbf{X}_{\mathcal{S}}^{\intercal} \mathbf{X}_{\mathcal{S}} $, and hence $\E\sqbr{\mathbf{R}} = \mathbf{\Sigma}_{\mathbf{x}_{\mathcal{S}^c} \mathbf{x}_\mathcal{S}}$, $\E\sqbr{\mathbf{Q}} = \mathbf{\Sigma}_{\mathbf{x}_\mathcal{S} \mathbf{x}_\mathcal{S}}$. Then, the simplified expression obtained after some algebraic manipulations (in Appendix \ref{appn:zsc_simp}) is:
	\begin{align}
	\indinf{\mathbf{X}_{\mathcal{S}^c}^{\intercal} \mathbf{X}_{\mathcal{S}} \left(\mathbf{X}_{\mathcal{S}}^{\intercal} \mathbf{X}_{\mathcal{S}} \right)^{-1} } \leq  m_1 + m_2 + m_3 + m_4\label{eq:zt1_simpl} 
	\end{align}
	where $m_1$, $m_2$, $m_3$, and $m_4$ are defined as: 
	\begin{align}
	m_1 & = \indinf{ \nrbr{\mathbf{R} - \E\sqbr{\mathbf{R}}} \nrbr{\E\sqbr{\mathbf{Q}}}^{-1}} \nonumber \\ 
	m_2 &= \indinf{\E\sqbr{\mathbf{R}} \left( \mathbf{Q}^{-1} - \nrbr{\E\sqbr{\mathbf{Q}}}^{-1}\right)} \nonumber \\  
	m_3 &= \indinf{ \nrbr{\mathbf{R} - \E\sqbr{\mathbf{R}}} \left( \mathbf{Q}^{-1} - \nrbr{\E\sqbr{\mathbf{Q}}}^{-1}\right)} \nonumber \\
	m_4 &= \indinf{ \E\sqbr{\mathbf{R}} \nrbr{\E\sqbr{\mathbf{Q}}}^{-1}} \nonumber 
	\end{align}
This carefully constructed decomposition of $\hat{\mathbf{{z}}}_{\mathcal{S}^c_{t_1}}$ has given us the freedom to study the effect of the adversarial perturbation on the non-support entries and the support entries by analyzing $m_1$ and $m_2$ respectively. The term $m_4$ in Eq. \eqref{eq:zt1_simpl}  can be bounded using mutual incoherence assumption \ref{assum:mutual_incoh}. We propose Lemma \ref{lem:m1_bound} and Lemma \ref{lem:m2_bound} to bound the terms $m_1$, $m_2$, and $m_3$ as discussed below. 

It may be noted that $m_1$ in Eq. \eqref{eq:zt1_simpl} is a function of the adversarial perturbation in $\mathcal{S}^c$ for a fixed value of $\E\sqbr{\mathbf{Q}} = \mathbf{\Sigma}_{\mathbf{x}_\mathcal{S} \mathbf{x}_\mathcal{S}}$. To bound $m_1$, we use the sub-multiplicative property of norms: 
	\begin{align*}
	m_1 \leq \indinf{ \frac{\mathbf{X}_{\mathcal{S}^c}^{\intercal} \mathbf{X}_{\mathcal{S}}}{n}  - \mathbf{\Sigma}_{\mathbf{x}_{\mathcal{S}^c} \mathbf{x}_\mathcal{S}}} \indinf{ \nrbr{\mathbf{\Sigma}_{\mathbf{x}_\mathcal{S} \mathbf{x}_\mathcal{S}}}^{-1}} 
	\end{align*}
To bound the first term in the RHS of the above equation, we propose the following lemma. 
\begin{lemma}
	\label{lem:m1_bound}
	If $n = \Omega\nrbr{\frac{k^2 \xi^2}{\delta^2} \log(p)}$ and $0\leq\delta \leq 32\xi k$, then $\indinf{ \frac{1}{n} \mathbf{X}_{\mathcal{S}^c}^{\intercal} \mathbf{X}_{\mathcal{S}} - \mathbf{\Sigma}_{\mathbf{x}_{\mathcal{S}^c} \mathbf{x}_\mathcal{S}} }  \leq \delta$ with probability at least  $1- \mathcal{O}\left(\frac{1}{p}\right)$, where $\xi = \max\limits_{i \in \mathcal{S}}\nrbr{\sigma \sqrt{\mathbf{\Sigma}}_{ii} + r \sqrt{\mathbf{\Sigma_e}}_{ii}} \max\limits_{j \in \mathcal{S}^c}\nrbr{\sigma \sqrt{\mathbf{\Sigma}}_{jj} + r \sqrt{\mathbf{\Sigma_e}}_{jj}}$. 
\end{lemma}
The proof of the above lemma relies on properties of norms and sub-Gaussian distributions, union bound, and sub-exponential tail bounds. We can claim the following by substituting $\delta = \frac{\gamma}{ \indinf{ 16\nrbr{\mathbf{\Sigma}_{\mathbf{x}_\mathcal{S} \mathbf{x}_\mathcal{S}}}^{-1}} }$ in Lemma \ref{lem:m1_bound}: 
\begin{align}
m_1 = \indinf{ \nrbr{\mathbf{R} - \E\sqbr{\mathbf{R}}} \nrbr{\E\sqbr{\mathbf{Q}}}^{-1}} \leq  \gamma/16 \label{eq:m1_boundf}
\end{align}
if $n = \Omega\nrbr{\frac{k^2 \xi^2 G^2_{\max}}{\gamma^2} + \log(p)}$. It should be noted that the value of $\delta$ chosen to analyze the adversarial perturbation in $\mathcal{S}^c$ is a function of the adversarial perturbation in $\mathcal{S}$ as $\mathbf{\Sigma}_{\mathbf{x}_\mathcal{S} \mathbf{x}_\mathcal{S}} = \mathbf{\Sigma}_{\mathbf{x}^{\star}_\mathcal{S} \mathbf{x}^{\star}_\mathcal{S}} + 2 \mathbf{\Sigma}_{\mathbf{e}_\mathcal{S} \mathbf{x}^{\star}_\mathcal{S}} + \mathbf{\Sigma}_{\mathbf{e}_\mathcal{S} \mathbf{e}_\mathcal{S}}$. Hence if the adversarial perturbation in $\mathcal{S}$ is designed such that  $ G_{\max} = \indinf{ \nrbr{\mathbf{\Sigma}_{\mathbf{x}_\mathcal{S} \mathbf{x}_\mathcal{S}}}^{-1}} $ increases, then the learner is forced to choose a smaller value of $\delta $ in Lemma \ref{lem:m1_bound} for the adversarial perturbation in $\mathcal{S}^c$ which increases the sample complexity. This also demonstrates the counterintuitive point that an adversary can influence the sample complexity  by attack on the non-support entries. This result is counter-intuitive as the output $y^{\star} = \mathbf{w}^{\star}\mathbf{x}^{\star} = \mathbf{w}^{\star}_{\mathcal{S}}\mathbf{x}^{\star}_{\mathcal{S}} + \mathbf{w}^{\star}_{\mathcal{S}^c}  \mathbf{x}^{\star}_{\mathcal{S}^c}  = \mathbf{w}^{\star}_{\mathcal{S}}\mathbf{x}^{\star}_{\mathcal{S}}  + 0 = \mathbf{w}^{\star}_{\mathcal{S}} \mathbf{x}^{\star}_{\mathcal{S}} $ is affected by regressor in the support only. The next step is to bound the term $m_2$ in Eq. \eqref{eq:zt1_simpl}, for which we propose the following lemma. 
\begin{lemma}
	\label{lem:m2_bound}
	If $n = \Omega\nrbr{\frac{k^2 }{\delta^2 \nrbr{C_{\text{min}} + 2F_{\text{min}} + D_{\text{min}}}^4 } \log(p)}$, then we claim $\indinf{ \nrbr{\frac{1}{n} \mathbf{X}_{\mathcal{S}}^{\intercal} \mathbf{X}_{\mathcal{S}}}^{-1}  -  \nrbr{\mathbf{\Sigma}_{\mathbf{x}_\mathcal{S} \mathbf{x}_\mathcal{S}}}^{-1} } \leq \delta$ with probability at least $1- \mathcal{O}\left(\frac{1}{p}\right)$. 
\end{lemma}
The proof of the above lemma relies on properties of norms, sub-Gaussian tail bounds, and also the bound for $\lmod \frac{1}{n} \mathbf{E}^{\intercal}_{\mathcal{S}}  \mathbf{X}^{\star}_{\mathcal{S}} - \mathbf{\Sigma}_{\mathbf{e}_{\mathcal{S}} \mathbf{x}^{\star}_{\mathcal{S}} }\rmod_{2} $ derived in the following theorem.
\begin{theorem}
	\label{thm:B_2}
	For $0 < \delta < \frac{32 r \sigma ab_k }{n} $, $\mathbf{E}_{\mathcal{S}}, \mathbf{X}^{\star}_{\mathcal{S}}  \in \mathbb{R}^{n \times k}$, $a^2 = \max\limits_{j \in \mathcal{S}}\left\{ \mathbf{ \Sigma}_{jj} \right\}, b^2_k =k \sum\limits_{i \in \mathcal{S}} \mathbf{ \Sigma}^2_{\mathbf{e}_{ii}}$
		\begin{align}
		\mathbb{P}\left[ \lmod \frac{\mathbf{E}^{\intercal}_{\mathcal{S}}  \mathbf{X}^{\star}_{\mathcal{S}}}{n}  - \mathbf{\Sigma}_{\mathbf{e}_{\mathcal{S}} \mathbf{x}^{\star}_{\mathcal{S}} }\rmod_{2} \geq \delta  \right]  \leq 4 e^{ \frac{-n \delta^2}{256r^2 \sigma^2 a b_k  }} \label{eq:lemB2}
		\end{align}
\end{theorem}
The proof of this lemma is interesting as $\mathbf{E}^{\intercal}_{\mathcal{S}}  \mathbf{X}^{\star}_{\mathcal{S}}  $ is a non-symmetric matrix, which makes the problem slightly challenging as compared to symmetric matrices. Hence we use the following lemma to transform a non-symmetric matrix to a symmetric matrix without changing its spectral norm. 
\begin{lemma}
	\label{lem:QandQ2}
	For matrix $\mathbf{B} \in \mathbb{R}^{k \times k}$, we claim $\lmod \mathbf{B} \rmod_2 = \lmod \mathbf{M} \rmod_2$, where $\mathbf{M} =${\scriptsize $ \begin{bmatrix}
		\mathbf{0}_{k \times k} & \mathbf{B} \\
		\mathbf{B}^{\intercal} &  \mathbf{0}_{k \times k}
		\end{bmatrix}$}.
\end{lemma}
Further, we use Lemma 6.12 from \cite{wainwright2019high}, which states for a random symmetric matrix $\mathbf{M} $
\begin{align}
\mathbb{P}\left[ \lmod\mathbf{M} \rmod_{2}\geq \delta \right] 	\leq 2 \text{tr}\left(\Psi_{\mathbf{M} }(\lambda)\right) e^{-\lambda \delta} 
\end{align}
where $\delta > 0$, $\Psi_{\mathbf{M}}(\lambda)$ denotes the moment generating function (MGF) and $\text{tr}$ represents the trace. Hence, in order to derive the bound of $\lmod\mathbf{M} \rmod_{2}$, we bound the trace of  the MGF of $\mathbf{M}$. Using the properties of sub-Gaussian and sub-exponential distributions, we derive a matrix $\mathbf{V} \in \mathbb{R}^{2k \times 2k}$ such that
\begin{align*}
\Psi_{\mathbf{M}}(\lambda) \preccurlyeq e^{\frac{\lambda^2 \mathbf{V}}{2}}
\end{align*}
Hence, in order to bound the trace of $\Psi_{\mathbf{Q}}(\lambda)$, we focus on the eigenvalues of $\mathbf{V}$. We later observe that $\mathbf{V}$ is a matrix with only two non-zero eigenvalues. This helps us to derive improved bounds as compared to the case of all $2k$ eigenvalues being non-zero.  More details can be seen in Appendix \ref{appn:B_2}. It should be noted that Theorem \ref{thm:B_2} as well as the lemmas are proposed for a generalized setting of rows of  $\mathbf{E}_{\mathcal{S}}, \mathbf{X}^{\star}_{\mathcal{S}}$ following two sub-Gaussian distributions which may be dependent and hence can be used in other works as well. 

Returning to deriving bound of $m_2$ in Eq. \eqref{eq:zt1_simpl}, we substitute $\delta = \frac{\gamma}{8 \indinf{\mathbf{\Sigma}_{\mathbf{x}_{\mathcal{S}^c} \mathbf{x}_\mathcal{S}}}}$ in Lemma \ref{lem:m2_bound} to claim: 
\begin{align}
m_2 = \indinf{\E\sqbr{\mathbf{R}} \left( \mathbf{Q}^{-1} - \nrbr{\E\sqbr{\mathbf{Q}}}^{-1}\right)} \leq \frac{\gamma}{8} \label{eq:m2_boundf}
\end{align}
if $n = \Omega\nrbr{\frac{k^2 H^2_{\max}}{\gamma^2 \nrbr{C_{\text{min}} + 2F_{\text{min}} + D_{\text{min}}}^4 } \log(p)}$. This analysis provides insight into the nature of the dependence of sample complexity on dimensions ($k$ and $p$), and also other parameters like mutual incoherence ($\gamma$), minimum eigenvalues $\crbr{C_{\min}, D_{\min}, F_{\min}}$, and other constants like $H_{\max}$.

Further, we proceed to bound the term $m_3$ in Eq. \eqref{eq:zt1_simpl} by using the sub-multiplicative property of norms
and substituting $\delta = \frac{\sqrt{\gamma}}{4}$ in Lemma \ref{lem:m1_bound} and Lemma \ref{lem:m2_bound} to claim the following with high probability 
	\begin{align*}
	m_3 = \indinf{ \nrbr{\mathbf{R} - \E\sqbr{\mathbf{R}}}} \indinf{ \left( \mathbf{Q}^{-1} - \nrbr{\E\sqbr{\mathbf{Q}}}^{-1}\right)}\leq \frac{\gamma}{16} 
	\end{align*}
if $n = \Omega\nrbr{k^2 \log(p)}$. We substitute Eq. \eqref{eq:m1_boundf}, Eq. \eqref{eq:m2_boundf}, and above equation in Eq. \eqref{eq:zt1_simpl}  to arrive at: 
\begin{align}
\lmod\hat{\mathbf{{z}}}_{\mathcal{S}^c_{t_1}}   \rmod_{\infty} 
\leq 1 - \gamma + \frac{\gamma}{16} + \frac{\gamma}{8} + \frac{\gamma}{16} = 1 - \frac{3\gamma}{4} \label{eq:Zsc_t1_bound} 
\end{align}
In this sub-section, we derived an upper bound for the infinity norm of $\hat{\mathbf{{z}}}_{\mathcal{S}^c_{t_1}} $ which will be used later to bound the infinity norm of $\hat{\mathbf{{z}}}_{\mathcal{S}^c} $ defined in Eq. \eqref{eq:Zsc_def} to ensure strict dual feasibility. More importantly, our analysis also sheds light on the dependence of sample complexity on parameters like mutual incoherence, minimum eigenvalue, and other constants like $G_{\max}$ and $H_{\max}$, which helps us to study critical scientific limitations or behavior of the LASSO algorithm under adversarial attacks. 
\subsubsection{Choosing regularization parameter ($\lambda$)}
\label{sec:Zsct2_infb}
In this subsection, we continue the discussion on strict dual feasibility and focus on how the adversary affects regularization parameter. We start from $\hat{\mathbf{{z}}}_{\mathcal{S}^c_{t_2}} $ in Eq. \eqref{eq:Zsct2_def}, which is a $(p-k)$ dimensional random vector. Using properties of norms, whose details are mentioned in Appendix \ref{sec:appdn_zsct2} and using the bound derived in Eq. \eqref{eq:Zsc_t1_bound}, we arrive at:  
	\begin{align}
	\norm{\hat{\mathbf{{z}}}_{\mathcal{S}^c_{t_2}}}_{\infty} \leq \frac{1}{\lambda} \norm{\frac{1}{n} \mathbf{X}_{\mathcal{S}^c}^{\intercal} \mathbf{P} \mathbf{e}_y }_{\infty} + \frac{1}{\lambda} \norm{\frac{1}{n} \mathbf{X}_{\mathcal{S}^c}^{\intercal}  \mathbf{E_{x_{\mathcal{S}}}} \mathbf{w}^{\star}_{\mathcal{S}} }_{\infty}  + \frac{1}{\lambda} \nrbr{1-\frac{3\gamma}{4}}\norm{ \frac{1}{n}\mathbf{X}_{\mathcal{S}}^{\intercal} \mathbf{E_{x_{\mathcal{S}}}} \mathbf{w}^{\star}_{\mathcal{S}} }_{\infty} \label{eq:Zsct2_decom}
	\end{align}

This decomposition enables us to analyze the effect of the adversarial perturbation on various model parameters. For example, the first term on the RHS of the above equation is concerned with the interaction of adversarial perturbation in $\mathcal{S}^c$ with $\mathbf{e}_y$. This term can be bounded by choosing an appropriate value of $\lambda$, as shown in the following lemma. 
\begin{lemma}
	\label{lem:Xe_ey_lambda}
	If  $\lambda = \lambda_1 \geq \frac{q_1 \sigma_{e_y}}{\gamma}\sqrt{\frac{2\log(p)}{n}}$, where constant $q^2_1 = 3  \nrbr{ C_{\max} + 2F_{\max} + D_{\max} } $, then $\norm{\frac{\mathbf{X}_{\mathcal{S}^c}^{\intercal} \mathbf{P} \mathbf{e}_y}{n \lambda} }_{\infty} \leq \frac{\gamma}{8}$ with probability of at least $1-\mathcal{O}\nrbr{\frac{1}{p}}$.
\end{lemma} 
Also, note that $\mathbf{P}$ defined in Eq. \eqref{eq:projmat_def} contains the adversarial perturbation in $\mathcal{S}$, but it does not affect the variance as $\mathbf{P}$ is proved to be a projection matrix in Lemma \ref{lem:Proj_mat}. 
\begin{lemma}
	\label{lem:Proj_mat}
	$\mathbf{P}$ defined in Eq. \eqref{eq:projmat_def} is a projection matrix and hence $\lmod \mathbf{P} \rmod^2_{2} = 1$. 
\end{lemma}
Similarly the second term on the RHS of Eq. \eqref{eq:Zsct2_decom} signifies the interaction of adversarial perturbation with $\mathbf{X}^{\star}_{\mathcal{S}^c}$. It can be bounded by choosing a suitable value of $\lambda$, as shown in the following lemma. 
\begin{lemma}
	\label{lem:XscEsw}
	If $\lambda = \lambda_2 \geq  \frac{16}{\gamma} \max \crbr{b_2, q_2 \sqrt{\frac{4 \log(p)}{n}}}$, then $\norm{\frac{\mathbf{X}_{\mathcal{S}^c}^{\intercal}  \mathbf{E_{x_{\mathcal{S}}}} \mathbf{w}^{\star}_{\mathcal{S}}}{n\lambda}}_{\infty}   \leq \frac{\gamma}{8}$ with probability of at least $1-\mathcal{O}\nrbr{\frac{1}{p}}$, where $b_2 = \norm{\mathbf{\Sigma}_{\mathbf{x}_{\mathcal{S}^c} \mathbf{e}_{\mathcal{S}}} \mathbf{w}^{\star}_{\mathcal{S}}  }_{\infty}$ and $q_2  = r\sqrt{ \mathbf{w}^{\star\intercal}_{\mathcal{S}} \mathbf{\Sigma}_{\mathbf{e}_{\mathcal{S} \mathcal{S}}} \mathbf{w}^{\star}_{\mathcal{S}}}   \max\limits_{i \in \mathcal{S}^c} \nrbr{\sigma \sqrt{\mathbf{ \Sigma}}_{ii} + r \sqrt{\mathbf{ \Sigma_e}}_{ii}} $.
\end{lemma}
Similarly the third term in RHS of Eq. \eqref{eq:Zsct2_decom} is concerned with the interaction of adversarial perturbation in $\mathbf{X}^{\star}_{\mathcal{S}}$. It can be bounded by selecting a suitable value of $\lambda$, which is presented in Lemma \ref{lem:XsEsw} in the Appendix. Substituting the bounds derived in Lemma \ref{lem:Xe_ey_lambda}, Lemma \ref{lem:XscEsw}, and Lemma \ref{lem:XsEsw} in Eq. \eqref{eq:Zsct2_decom}, we obtain: 
\begin{align}
\norm{\hat{\mathbf{{z}}}_{\mathcal{S}^c_{t_2}}}_{\infty} \leq \frac{\gamma}{8} + \frac{\gamma}{8} + \frac{\gamma}{8} = \frac{3\gamma}{8} \label{eq:zsct2_ub}
\end{align}
It should be noted that this bound is derived under some lower bound constraint on the regularization parameter. The lower bound can be obtained by taking the maximum of $\lambda_1, \lambda_2$, and $\lambda_3$ presented in Lemma \ref{lem:Xe_ey_lambda}, Lemma \ref{lem:XscEsw}, and Lemma \ref{lem:XsEsw} in Appendix respectively: 
	\begin{align}
	\lambda \geq & \max\crbr{\lambda_1, \lambda_2, \lambda_3} \nonumber \\
	= & \max\crbr{\frac{16b}{\gamma}, \frac{q_1 \sigma_{e_y}}{\gamma} \sqrt{\frac{2\log(p)}{n}}, \frac{16q}{\gamma}\sqrt{\frac{4\log(p)}{n}}} \label{eq:lambda_f1}
	\end{align}
where $q = r\sqrt{ \mathbf{w}^{\star\intercal}_{\mathcal{S}} \mathbf{\Sigma}_{\mathbf{e}_{\mathcal{S} \mathcal{S}}} \mathbf{w}^{\star}_{\mathcal{S}}}   \max\limits_{i \in [p]} \nrbr{\sigma \sqrt{\mathbf{ \Sigma}}_{ii} + r \sqrt{\mathbf{ \Sigma_e}}_{ii}} $ and $b = \norm{\mathbf{\Sigma}_{\mathbf{x} \mathbf{e}_{\mathcal{S}}} \mathbf{w}^{\star}_{\mathcal{S}}  }_{\infty}$.  This completes the lower bound proof of $\lambda$ used in Theorem \ref{thm:main}. 

Note that for a fixed budget, the adversary can increase the lower bound of $\lambda$ by designing $\mathbf{\Sigma}_{\mathbf{e}_{\mathcal{S} \mathcal{S}}}$ such that the eigenvector corresponding to maximum eigenvalue of   $\mathbf{\Sigma}_{\mathbf{e}_{\mathcal{S} \mathcal{S}}} $ is parallel to $\mathbf{w}^{\star}_{\mathcal{S}}$ to increase $q$ in Eq. \eqref{eq:lambda_f1}. A higher value of the lower bound of $\lambda$ implies more penalization on the regression parameter vector, which might make the learning algorithm to incorrectly estimate the small non-zero parameters in $\mathcal{S}$ to be zero. To mitigate this adversarial effect, the learner requires more samples $n$, to decrease the lower bound on $\lambda$.

Returning to the strict dual feasibility condition, the bound for  $ \hat{\mathbf{{z}}}_{\mathcal{S}^c} $ defined in Eq. \eqref{eq:Zsc_def} is derived by using the bound for $\hat{\mathbf{{z}}}_{\mathcal{S}^c_{t_2}}$ in Eq. \eqref{eq:zsct2_ub} and the  bound for $\hat{\mathbf{{z}}}_{\mathcal{S}^c_{t_1}}$ in  Eq. \eqref{eq:Zsc_t1_bound} 
\begin{align}
\lmod \hat{\mathbf{{z}}}_{\mathcal{S}^c} \rmod_{\infty} \leq 1 - \frac{3\gamma}{4} + \frac{3\gamma}{8} = 1- \frac{3\gamma}{8 } < 1 \label{eq:strictdualfv_f}
\end{align}
In this sub-section, we have verified the strict dual feasibility condition by proving that $\lmod \hat{\mathbf{{z}}}_{\mathcal{S}^c} \rmod_{\infty}  < 1$ as $\gamma > 0$ in the above equation. This ensures that KKT conditions are met, which proves the first claim of Theorem \ref{thm:main}, i.e., $\mathcal{S}(\hat{\mathbf{w}}) \subseteq \mathcal{S}(\mathbf{w}^{\star})$. It should be noted that we derive the lower bound constraint on $\lambda$ for giving theoretical guarantees. For practical purposes, we choose $\lambda = \mathcal{O}\nrbr{\sqrt{\frac{\log(p)}{n}}}$ as done in the sparse regression literature \cite{wainwright2009sharp,ravikumar2009sparse, ravikumar2010high,ravikumar2011high, daneshmand2014estimating}.  
\subsection{Uniqueness of the solution}
\label{sec:uniq}
In this sub-section, we prove the uniqueness of the optimal solution $\hat{\mathbf{w}}_{\mathcal{S}}$.  We need the second order derivative, $\left[ \nabla^2l((\mathbf{w}_{\mathcal{S}}, \mathbf{0}))\right]_{\mathcal{S}, \mathcal{S}} = \frac{1}{n}\mathbf{X}_{\mathcal{S}}^{\intercal} \mathbf{X}_{\mathcal{S}}$ (computed in Appendix \ref{appn:lossfn_0,1,2stcond}) to be positive definite for the problem in Eq. \eqref{eq:res_lasso} to be strictly convex in the support space (see Eq.\eqref{eq:res_lassapn} in the appendix for a formal definition). The positive definiteness of a submatrix of the sample covariance is proved in the following lemma.
\begin{lemma}
	\label{lem:sampleHess_psd}
	If assumption \ref{assum:Sigma_Xpsd} holds and $n = \Omega\left( k \log(p)\right)$, then we claim 
	\begin{align*}
	\mathbb{P} \left[ \Lambda_{\text{min}} \left( \frac{\mathbf{X}_{\mathcal{S}}^{\intercal} \mathbf{X}_{\mathcal{S}}}{n}  \right)  \geq \frac{(C_{\text{min}} + 2F_{\text{min}} 
		+ D_{\text{min}})}{2}  \right] \geq 1 - \mathcal{O}\left(\frac{1}{p}\right)
	\end{align*}
\end{lemma}
Hence $\left[ \nabla^2l((\mathbf{w}_{\mathcal{S}}, \mathbf{0}))\right]_{\mathcal{S}, \mathcal{S}} = \frac{1}{n}\mathbf{X}_{\mathcal{S}}^{\intercal} \mathbf{X}_{\mathcal{S}}$ is positive definite. More importantly, as the Hessian matrix depends only on adversarial perturbation in the support $\mathcal{S}$, sample complexity in the above lemma is not impacted by perturbation in the non-support $\mathcal{S}^c$. But this does not imply allocating more budget to $\mathcal{S}$ to design perturbation is recommended from the adversary's perspective, as more budget to $\mathcal{S}$ may lead to increasing $D_{\min}$, which is advantageous for the learning algorithm. In a more formal way, we need to bound $\lambda_{\text{min}}\left( \frac{ \mathbf{E}^{\intercal}_{\mathbf{x}_{\mathcal{S}}} \mathbf{E}_{\mathbf{x}_{\mathcal{S}}}}{n}\right)$ while proving Lemma \ref{lem:sampleHess_psd}, which requires $n = \Omega\left( {\left( k + \log(p)\right)}/{D^2_{\text{min}} }\right)$ samples (Eq. \eqref{eq:mineig_Ex}). Hence, it is advisable for the adversary to design perturbations such that $D_{\text{min}}$ is small.

With a brief discussion on uniqueness in this sub-section, we provide theoretical guarantees for the estimated regression parameter vector in the next subsection.
\subsection{Quality of estimated regression parameter vector } 
\label{sec:w_bound}
In this subsection, we prove the third claim made in Theorem \ref{thm:main} and discuss how the adversarial perturbation in non-support $\mathcal{S}^c$ can  affect the theoretical guarantees for $\hat{\mathbf{w}}_{\mathcal{S}}$ (in support) indirectly through regularization parameter. We start with the computation of $\hat{\mathbf{w}}_{\mathcal{S}} - \mathbf{w}_{\mathcal{S}}^{\star}$ by using the first order stationary condition specified in Eq. \eqref{eq:1st_ord}. The algebraic steps are presented in Appendix \ref{appn:lossfn_0,1,2stcond} and the simplified expression is: 
	\begin{align}
	&\norm{\hat{\mathbf{w}}_{\mathcal{S}}- \mathbf{w}_{\mathcal{S}}^{\star}}_{\infty} \leq \indinf{\mathbf{A}^{-1}}\nrbr{\norm{\mathbf{w_1} }_{\infty} + \norm{\mathbf{w_2} }_{\infty} } \nonumber \\
	& \qquad \qquad \qquad \qquad + \lambda \indinf{\mathbf{A}^{-1}}  \norm{\hat{\mathbf{{z}}}_{\mathcal{S}}}_{\infty}\label{eq:wb_trnineq} \\ 
	&\mathbf{w_1} = \frac{\mathbf{X}_{\mathcal{S}}^{\intercal} \mathbf{e}_y}{n}, \; \mathbf{w_2}  = \frac{\mathbf{X}_{\mathcal{S}}^{\intercal}\mathbf{E_{x_{\mathcal{S}}}} \mathbf{w}^{\star}_{\mathcal{S}} }{n}, \; \mathbf{A} = \frac{\mathbf{X}_{\mathcal{S}}^{\intercal}  \mathbf{X}_{\mathcal{S}}}{n}  \label{eq:w2_def}
	\end{align}

The last term in RHS of Eq. \eqref{eq:wb_trnineq} can be easily bounded as $\norm{\hat{\mathbf{{z}}}_{\mathcal{S}}}_{\infty} \leq 1 $. To further bound $ \indinf{\mathbf{A}^{-1} } $, we use the triangle inequality: 
\begin{align*}
\indinf{\mathbf{A}^{-1}} \leq \indinf{\mathbf{A}^{-1} - \nrbr{\E\sqbr{\mathbf{ A}}}^{-1}} + \indinf{\nrbr{\E\sqbr{\mathbf{ A}}}^{-1}}
\end{align*}
The first term in the RHS of the above equation can be bounded using Lemma \ref{lem:m2_bound}. We can claim 
\begin{align}
\P\sqbr{\indinf{\mathbf{A}^{-1} - \nrbr{\E\sqbr{\mathbf{ A}}}^{-1}} \geq \frac{G_{\max} }{2}} \leq  \mathcal{O}\left(\frac{1}{p}\right) \label{eq:w2b_Aeinv}
\end{align}
by substituting $\delta = \frac{G_{\max}}{2}$ in Lemma \ref{lem:m2_bound} if $n = \Omega\nrbr{\frac{k^2 \log(p) }{G^2_{\text{max}}}}$. Using this, we can claim $\indinf{\mathbf{A}^{-1}} \leq \frac{3G_{\max} }{2}$. Further we proceed to bound $\norm{\mathbf{ w}_1}_{\infty}$ defined in Eq. \eqref{eq:w2_def}. Using an approach very similar to Lemma \ref{lem:Xe_ey_lambda}, we claim: 
\begin{align}
\norm{\mathbf{ w}_1}_{\infty} =\norm{\frac{\mathbf{X}_{\mathcal{S}}^{\intercal} \mathbf{e}_y}{n}}_{\infty} \leq \frac{\lambda \gamma}{8} \label{eq:w1_binf}
\end{align}
It should be noted that there is lower bound constraint on $\lambda$ for the above statement to hold with high probability, as specified in Lemma \ref{lem:Xe_ey_lambda}.  The lower bound value of $\lambda$ can be tightened slightly for this case specifically by changing the $\log(p)$ factor to $\log(k)$ as $\mathbf{ w}_1$ is a $k-$dimensional vector, and we need to take union bound over $k$ elements only instead of $p-k$, as done in Lemma \ref{lem:Xe_ey_lambda}.  But we take the $\lambda$ mentioned in Eq. \eqref{eq:lambda_f1}, so that the strict dual feasibility is also verified. 

Further, we proceed to bound  $\norm{\mathbf{ w}_2}_{\infty}$ defined in Eq. \eqref{eq:w2_def} by using the approach similar to  Lemma \ref{lem:XsEsw}  presented in Appendix. We claim: 
\begin{align}
\norm{\mathbf{w_2}}_{\infty}  = \norm{\frac{\mathbf{X}_{\mathcal{S}}^{\intercal}\mathbf{E_{x_{\mathcal{S}}}} \mathbf{w}^{\star}_{\mathcal{S}} }{n}}_{\infty} \leq \frac{\lambda \gamma}{8} \label{eq:w2_binf}
\end{align}
where the lower bound on $\lambda$ is specified in Eq. \eqref{eq:lambda_f1}. Substituting the bounds derived in Eq. \eqref{eq:w1_binf} and Eq. \eqref{eq:w2_binf} in Eq. \eqref{eq:wb_trnineq}, we obtain: 
\begin{align}
\norm{\hat{\mathbf{w}}_{\mathcal{S}}- \mathbf{w}_{\mathcal{S}}^{\star}}_{\infty} \leq \lambda \nrbr{1 + \frac{\gamma}{4}} \frac{3G_{\max} }{2}  = f(\lambda)
\end{align}
This proves the third claim in Eq. \eqref{eq:thm3rdclaim} of Theorem \ref{thm:main}. From the above equation, we observe that a large value regularization $\lambda$ is not desirable as it is directly proportional to the bound of $\lmod \hat{\mathbf{w}}_{\mathcal{S}} - \mathbf{w}_{\mathcal{S}}^{\star} \rmod_{\infty}$. But note that the lower bound of $\lambda$ can be controlled by the adversary due to the presence of constants $b$ and $q$ in Eq. \eqref{eq:lambda_f1}, and hence the adversary can control the quality of the estimated regression parameter vector as demonstrated shortly. Before proceeding to that discussion, we need to prove the fourth claim of sign matching in Theorem \ref{thm:main}, which can be seen as a direct consequence of Lemma \ref{lem:sign_p} in the Appendix.

Consider the case when $\lambda = \max\crbr{\lambda_1, \lambda_2, \lambda_3}= \frac{16b}{\gamma}$, then we need the following condition as per Theorem \ref{thm:main}:
\begin{align}
\min\limits_{i \in \mathcal{S}}\left|\mathbf{w}_i^{\star}\right| \geq 2 f(\lambda) = 12G_{\max}  \nrbr{1 + \frac{4}{\gamma}}b. \label{eq:case_1b}
\end{align}
This requirement on the lower bound of the absolute value of parameters in the support basically states that these coefficients should have significant values for detection. If the adversary is given more budget and designs a large value of  $b = \norm{\mathbf{\Sigma}_{\mathbf{x} \mathbf{e}_{\mathcal{S}}} \mathbf{w}^{\star}_{\mathcal{S}}  }_{\infty}$ to break the above requirement (Eq. \eqref{eq:case_1b}), then we may not be able to provide theoretical guarantees for successful support recovery. Our theoretical analysis has identified the critical condition under which the adversary can design malicious attacks such that the LASSO algorithm may not have a high probability of successful support recovery. It should be noted that the lack of theoretical guarantees for successful support recovery does not restrict a user from using the LASSO algorithm in practice. It may still do correct support recovery under this case, but we may not be able to provide reasonable bounds for probability of success.  

Consider the case with $\mathbf{\Sigma}_{\mathbf{x} \mathbf{e}_{\mathcal{S}}} = \mathbf{0}_{p \times k}$, and hence $b = 0$. Therefore $\lambda = \max\crbr{\lambda_1, \lambda_2, \lambda_3}=\max\crbr{\lambda_2, \lambda_3}= \mathcal{O}\nrbr{\sqrt{\frac{\log(p)}{n}}}$, then the same requirement is:
\begin{align}
\min\limits_{i \in \mathcal{S}}\left|\mathbf{w}_i^{\star}\right| \geq 2 f(\lambda) = \mathcal{O}\nrbr{\sqrt{\frac{\log(p)}{n}}} \label{eq:minw_1}
\end{align}
This condition can be easily fulfilled by increasing the value of $n$ sufficiently high depending on the value of  $\min_{i \in \mathcal{S}}\left|\mathbf{w}_i^{\star}\right|$, and hence, theoretical guarantees can be established. The adversary can still try to break the above condition by increasing the value of $q_1$ or $q$ in $\lambda_2$ or $\lambda_3$ respectively, but the user can increase the sample size ($n$) accordingly as derived in various lemmas to ensure a high probability of success.  For example, consider a scenario when the LASSO algorithm is performing satisfactorily, and the adversary tries to break Eq. \eqref{eq:minw_1} by increasing $\gamma$ to twice its value. Assuming other parameters are constant, the user can increase the value of $n$ to at least four times as compared to its previous value for Eq. \eqref{eq:minw_1} to hold. This quadratic dependence of $n$ on $\gamma$ can be seen from the sample complexity of bound derived in Eq. \eqref{eq:m1_boundf}. Similarly, we can see the dependence of $n$ on other adversarial parameters. 

In this subsection, we completed the proof of Theorem \ref{thm:main} and discussed the critical regimes which may be favorable to adversary or learning algorithm. We also discussed the counter-intuitive result of how the adversarial perturbation in $\mathcal{S}^c$ can affect the guarantees for $\hat{\mathbf{w}}_{\mathcal{S}}$ indirectly by influencing the lower bound on the regularization parameter.  
\section{Experiments}
\label{sec:exp}
In this section, we validate our proposed theoretical claims with empirical analysis on synthetic data and real-world data. Please refer to Appendix \ref{appn:exp} for more details. 

\begin{figure}[H]
	\centering
	\includegraphics[width=0.65\columnwidth]{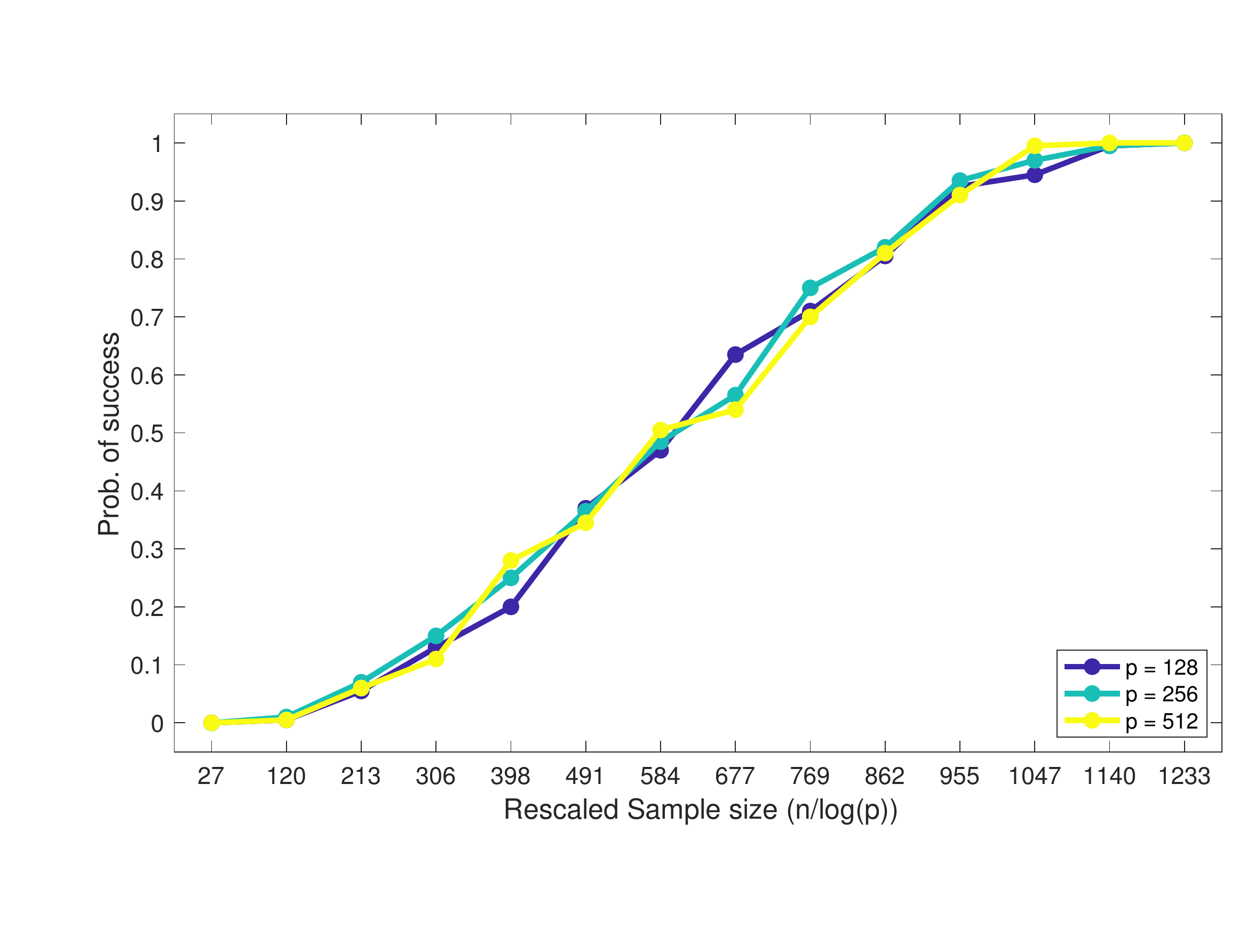} 
	\caption{Probability of support recovery vs rescaled sample size}
	\label{fig:sigmax_0.1}
\end{figure}

\paragraph{Synthetic data:} To verify the sample complexity result of $n = \Omega(\log(p))$ for fixed $k$, we repeat the support recovery experiment $200$ times for a particular value of $(n,p)$. We report the empirical probability of successful support recovery. 
We perform this experiment for $p \in \{128, 256, 512\}$ and vary $n$ such that $\frac{n}{\log(p)} \in \left(25, 1250\right)$. The results presented in Figure \ref{fig:sigmax_0.1}, show that $\frac{n}{\log(p)}$ is not a function of $p$ (hence a constant), as the plots are overlapping. We make the same observation for the more challenging case of adversarial perturbation being dependent on the data (refer Appendix \ref{appn:exp}). We also observe that the sample complexity may increase at least linearly with $\lmod \mathbf{\Sigma}_{\mathbf{e}_{\mathcal{S} \mathcal{S}}} \rmod_{2} $ as implied by Eq. \eqref{eq:lambda_f1} theoretically.

\paragraph{Real-World Data:} We used the BlogFeedback dataset \cite{buza2014feedback} which contains 52397 samples and 276 features. 
We first recover the ``true'' support with the given data and further estimate the support from adversarially corrupted data, which is generated by adding adversarial perturbation to all the features. Our approach recovers the support with F1-score of 0.94, which demonstrates the efficacy of our proposed method on real-world datasets.

\paragraph{Concluding Remarks:}
We hope our work insinuates similar ideas for other problems as learning under adversarial attacks is an interesting problem that is not restricted to a linear sparse regression model. For example, our contributions can be explored in other problems such as nonparametric sparse regression \cite{ravikumar2009sparse}, matrix factorization \cite{luo2020adversarial}, or model compression \cite{gui2019model}. 

\bibliographystyle{plain}
\bibliography{mybibfile}

\newpage

\appendix

\noindent\rule{\textwidth}{3pt}
\begin{center}
	{\Large \textbf{Supplementary Material: A Theoretical Study of The Effects of Adversarial Attacks on Sparse Regression}}
\end{center}
\noindent\rule{\textwidth}{1pt}

\section{Appendix}

\subsection{Proof of Lemma \ref{lem:zero_mean}}
\textbf{Lemma \ref{lem:zero_mean}}: If the adversary attacks only the support entries ($\mathbf{x}_{\mathcal{S}}$) or non-support entries ($\mathbf{x}_{\mathcal{S}^c}$) with non-zero mean adversarial perturbation, then the learner can guess the support trivially with probability at least  $1 - \mathcal{O}\left( \frac{1}{p} \right)$ if $n = \Omega\left(\log (p)\right)$. 
\begin{proof}
	Let the non-zero mean adversarial perturbation have the form $\bm{\mu} = \begin{bmatrix}
	\mu_1 \mathbf{1}_k \\ 
	\mu_2 \mathbf{1}_{p-k}
	\end{bmatrix}$. Also, for clarity, we assume that the first $k$ entries of $\mathbf{x}$ correspond to the support $\mathcal{S}$ and the rest correspond to the non-support $\mathcal{S}^c$. If this is not the case, the support and non-support entries will need to be properly interleaved. For the attack on $\mathcal{S}$ only, we consider the case $\mu_1 \neq 0$ and $\mu_2 = 0$. Similarly, for the attack on $\mathcal{S}^c$ only, we consider the case of $\mu_1 = 0$ and $\mu_2 \neq  0$. 
	
	We first analyze the sample mean of the entries in the support. For $n$ samples, that is $\mathbf{x}^{(j)} \in \mathbb{R}^p$ for $j \in [n]$, we can compute the sample mean. For the $i^{\text{th}}$ entry, denoted by $\mathbf{x}_i$, where $i \in \mathcal{S}$, we use the sub-Gaussian tail bound along with the union bound:
	\begin{align}
	\mathbb{P} \left[ \left( \exists i \in \mathcal{S} \right)  \left|	\frac{1}{n} \sum_{j = 1}^{n} \mathbf{x}_i^{(j)} -  \mu_1 \right| \geq t\right] \leq 2 \sum_{i \in \mathcal{S}}\exp\left\{ \frac{-nt^2}{2 \sigma_{x_i}^2}\right\} 
	\end{align}
	where $\sigma_{x_i}^2$ denotes the variance proxy parameter. Similarly, the mean of the entries in the non-support can be analyzed as 
	\begin{align*}
	\mathbb{P} \left[ \left( \exists l \in \mathcal{S}^c \right) \left|	\frac{1}{n} \sum_{j = 1}^{n} \mathbf{x}_l^{(j)} -  \mu_2 \right| \geq t\right] \leq 2 \sum_{l \in \mathcal{S}^c} \exp\left\{ \frac{-nt^2}{2 \sigma_{x_l}^2}\right\} 
	\end{align*}
	We substitute $t = \frac{\max(|\mu_1|, |\mu_2|)}{3}$ in the above equations. Now, if $\mu_1 = 0$ and $\mu_2 \neq 0$, we can claim: 
	\begin{align*}
	\left|	\frac{1}{n} \sum_{j = 0}^{n} \mathbf{x}_i^{(j)}\right| &\leq \frac{\max(|\mu_1|, |\mu_2|)}{3} = \frac{|\mu_2|}{3} \\
	\left|	\frac{1}{n} \sum_{j = 0}^{n} \mathbf{x}_l^{(j)} \right| &\geq |\mu_2| - \left|	\frac{1}{n} \sum_{j = 0}^{n} \mathbf{x}_l^{(j)} -  \mu_2 \right| \geq |\mu_2| - \frac{\max(|\mu_1|, |\mu_2|)}{3} = 2\frac{|\mu_2|}{3}
	\end{align*}
	with high probability of $1 - \mathcal{O}\left( \frac{1}{p}\right) $ if $n = \Omega\left( \log(p)\right)$.  Note that the sample mean of the entries in the support is upper-bounded by $\frac{|\mu_2|}{3}$, whereas the sample mean of the entries in the non-support is lower bounded by $2\frac{|\mu_2|}{3}$. Hence the learner can guess the support easily by observing the concentration of the sample mean.  Note that the case of $\mu_1 \neq 0$ and $\mu_2 = 0$ can be analyzed similarly. 
	
\end{proof}	
\subsection{Variance proxy parameter for adversarial perturbation}
\label{appn:ex_subG}
We derive the variance proxy parameter of a sub-Gaussian vector $\mathbf{e_x}$. This is done by first starting with a general random sub-Gaussian vector, $\mathbf{z}$ with variance proxy parameter $r^2$ and identity covariance matrix. The definition of sub-Gaussian vectors in \cite{hsu2012tail} states that for all $\mathbf{v} \in \mathbb{R}^p$:
\begin{align}
\mathbb{E}\left[ \mathbf{v}^{\intercal}  \mathbf{z}\right] \leq \exp\left( \frac{\lmod \mathbf{v}\rmod^2_{2} r^2}{2} \right)
\end{align}
Without loss of generality, we define $\mathbf{e_x} = \mathbf{\Sigma}^{1/2}_{\mathbf{e}} \mathbf{z}$, where $\mathbf{\Sigma}_{\mathbf{e}} $ is the covariance matrix of $\mathbf{e_x} $. Substituting $\mathbf{z} =  \mathbf{\Sigma}^{-1/2}_{\mathbf{e}} \mathbf{e_x}$ in the above equation: 
\begin{align*}
\mathbb{E}\left[ \mathbf{v}^{\intercal}  \mathbf{\Sigma}^{-1/2}_{\mathbf{e}} \mathbf{e_x}\right] \leq \exp\left( \frac{\lmod \mathbf{v}\rmod^2_{2} r{\tiny }^2}{2} \right)
\end{align*}
Substituting $\bm{\alpha}^{\intercal}  = \mathbf{v}^{\intercal}  \mathbf{\Sigma}^{-1/2}_{\mathbf{e}} $ in the above equation
\begin{align}
\mathbb{E}\left[ \bm{\alpha}^{\intercal}  \mathbf{e_x}\right] &\leq \exp\left( \frac{\bm{\alpha}^{\intercal}  \mathbf{\Sigma}_{\mathbf{e}} \bm{\alpha} r^2}{2} \right) 
\end{align}
which holds for all $\bm{\alpha} \in \mathbb{R}^p$.  

\subsection{First Order Stationarity condition}
\label{appn:lossfn_0,1,2stcond}
Consider the loss function
\begin{align}
l(\mathbf{w}) = \frac{1}{2n} ||\mathbf{y} - \mathbf{X} \mathbf{w}||_2^2. \label{eq:loss_fnappn}
\end{align}
The Lasso problem is given by:
\begin{align}
\hat{\mathbf{w}}_{\mathcal{S}} = \argmin_{\mathbf{w}_{\mathcal{S}} \in \mathbb{R}^{k}}l((\mathbf{w}_{\mathcal{S}}, \mathbf{0})) + \lambda ||\mathbf{w}_{\mathcal{S}}||_1.  \label{eq:res_lassapn}
\end{align}
We start with the first-order stationary condition.  Taking the first order derivative of Eq. \eqref{eq:loss_fnappn}, we get: 
\begin{align*}
\nabla l(\mathbf{w}) 
=& \frac{1}{n} \mathbf{X}^{\intercal}  \left( \mathbf{X}(\mathbf{w} - \mathbf{w}^{\star} ) + \left(\mathbf{E_x} \mathbf{w}^{\star}  - \mathbf{e}_y\right)   \right) \\
\left[\nabla l((\hat{\mathbf{w}}_{\mathcal{S}}, \mathbf{0})) \right]_{\mathcal{S}} =& \frac{1}{n}\mathbf{X}_{\mathcal{S}}^{\intercal}  \left( \mathbf{X}_{\mathcal{S}} (\hat{\mathbf{w}}_{\mathcal{S}} - \mathbf{w}_{\mathcal{S}}^{\star} ) + \left(\mathbf{E_{x_{\mathcal{S}}}} \mathbf{w}^{\star} _{\mathcal{S}} - \mathbf{e}_y\right)   \right) \\
\left[\nabla l((\hat{\mathbf{w}}_{\mathcal{S}}, \mathbf{0})) \right]_{\mathcal{S}^c} =& \frac{1}{n}\mathbf{X}_{\mathcal{S}^c}^{\intercal}  \left( \mathbf{X}_{\mathcal{S}} (\hat{\mathbf{w}}_{\mathcal{S}} - \mathbf{w}_{\mathcal{S}}^{\star} ) + \left(\mathbf{E_{x_{\mathcal{S}}}} \mathbf{w}^{\star} _{\mathcal{S}} - \mathbf{e}_y\right)   \right)
\end{align*}
The stationarity condition of Eq. \eqref{eq:res_lassapn}, after splitting into the support $\mathcal{S}$ and non-support $\mathcal{S}^c$, becomes:
\begin{align*}
\left[\nabla l((\hat{\mathbf{w}}_{\mathcal{S}}, \mathbf{0})) \right]_{\mathcal{S}} + \lambda \hat{\mathbf{{z}}}_{\mathcal{S}} &= \mathbf{0}_{k}\\ 
\left[\nabla l((\hat{\mathbf{w}}_{\mathcal{S}}, \mathbf{0})) \right]_{\mathcal{S}^c} + \lambda \hat{\mathbf{{z}}}_{\mathcal{S}^c} &= \mathbf{0}_{(p-k)}\\ 
\end{align*}
Using these equations, we arrive at:
\begin{align}
\frac{1}{n}\mathbf{X}_{\mathcal{S}}^{\intercal}  \left( \mathbf{X}_{\mathcal{S}} (\hat{\mathbf{w}}_{\mathcal{S}} - \mathbf{w}_{\mathcal{S}}^{\star} ) + \left(\mathbf{E_{x_{\mathcal{S}}}} \mathbf{w}^{\star} _{\mathcal{S}}  - \mathbf{e}_y\right)   \right)+ \lambda \hat{\mathbf{{z}}}_{\mathcal{S}} &= \mathbf{0} \nonumber \\
(\hat{\mathbf{w}}_{\mathcal{S}} - \mathbf{w}_{\mathcal{S}}^{\star} ) = \left(\mathbf{X}_{\mathcal{S}}^{\intercal}  \mathbf{X}_{\mathcal{S}}\right)^{-1}\left( \mathbf{X}_{\mathcal{S}}^{\intercal}  \left( \mathbf{e}_y - \mathbf{E_{x_{\mathcal{S}}}} \mathbf{w}^{\star} _{\mathcal{S}}  \right)  - n \lambda  \hat{\mathbf{{z}}}_{\mathcal{S}} \right) \label{eq:bdiff_gen}
\end{align} 
Further, using triangle inequality and sub-multiplicative property of norms, we arrive at: 
\begin{align}
\norm{\hat{\mathbf{w}}_{\mathcal{S}} - \mathbf{w}_{\mathcal{S}}^{\star} }_{\infty} \leq \norm{\nrbr{\frac{\mathbf{X}_{\mathcal{S}}^{\intercal}  \mathbf{X}_{\mathcal{S}}}{n}}^{-1}}_{\infty}\nrbr{ \norm{\frac{\mathbf{X}_{\mathcal{S}}^{\intercal}  \mathbf{e}_y }{n}}_{\infty} + \norm{\frac{\mathbf{X}_{\mathcal{S}}^{\intercal}  \mathbf{E_{x_{\mathcal{S}}}} \mathbf{w}^{\star} _{\mathcal{S}}  }{n}}_{\infty} + \lambda \norm{\hat{\mathbf{{z}}}_{\mathcal{S}}}_{\infty}}
\end{align}
Now, $\hat{\mathbf{{z}}}_{\mathcal{S}^c}  $ can be computed as
\begin{align}
\hat{\mathbf{{z}}}_{\mathcal{S}^c}  & = -\frac{1}{\lambda n}\mathbf{X}_{\mathcal{S}^c}^{\intercal}  \left( \mathbf{X}_{\mathcal{S}} (\hat{\mathbf{w}}_{\mathcal{S}} - \mathbf{w}_{\mathcal{S}}^{\star} ) + \left(\mathbf{E_{x_{\mathcal{S}}}} \mathbf{w}^{\star}_{\mathcal{S}}  - \mathbf{e}_y\right)   \right) \nonumber \\
& = -\frac{1}{\lambda n}\mathbf{X}_{\mathcal{S}^c}^{\intercal}  \left( \mathbf{X}_{\mathcal{S}} \left( \left(\mathbf{X}_{\mathcal{S}}^{\intercal}  \mathbf{X}_{\mathcal{S}} \right)^{-1}\left( \mathbf{X}_{\mathcal{S}}^{\intercal}  \left( \mathbf{e}_y - \mathbf{E_{x_{\mathcal{S}}}} \mathbf{w}^{\star}_{\mathcal{S}}  \right)  - n \lambda  \hat{\mathbf{{z}}}_{\mathcal{S}} \right) \right) + \left(\mathbf{E_{x_{\mathcal{S}}}} \mathbf{w}^{\star}_{\mathcal{S}} - \mathbf{e}_y\right)   \right) \nonumber \\
&= \mathbf{X}_{\mathcal{S}^c}^{\intercal}  \left\{ \mathbf{X}_{\mathcal{S}} \left(\mathbf{X}_{\mathcal{S}}^{\intercal}  \mathbf{X}_{\mathcal{S}}\right)^{-1} \hat{\mathbf{{z}}}_{\mathcal{S}}  + \left( \mathbf{I}_n - \mathbf{X}_{\mathcal{S}} \left(\mathbf{X}_{\mathcal{S}}^{\intercal}  \mathbf{X}_{\mathcal{S}} \right)^{-1} \mathbf{X}_{\mathcal{S}}^{\intercal}   \right) \frac{\left( \mathbf{e}_y - \mathbf{E_{x_{\mathcal{S}}}} \mathbf{w}^{\star}_{\mathcal{S}}  \right)  }{\lambda n}\right\}  \label{eq:Zsc_der_apn}
\end{align}
where $\mathbf{I}_n$ denotes an identity matrix of dimension $n \times n$.

The second order derivative of Eq. \eqref{eq:loss_fnappn} is: 
\begin{align}
\nabla^2 l(\mathbf{w}) &= \frac{1}{n}\mathbf{X}^{\intercal} \mathbf{X} \nonumber \\
\left[ \nabla^2l((\mathbf{w}_{\mathcal{S}}, \mathbf{0}))\right]_{\mathcal{S}, \mathcal{S}} &= \frac{1}{n}\mathbf{X}_{\mathcal{S}}^{\intercal} \mathbf{X}_{\mathcal{S}} \nonumber
\end{align}

\subsection{Simplification of $\lmod\hat{\mathbf{{z}}}_{\mathcal{S}^c_{t_1}}   \rmod_{\infty} $}
\label{appn:zsc_simp}
In this sub-section, we present the simplification of the term $\lmod\hat{\mathbf{{z}}}_{\mathcal{S}^c_{t_1}}   \rmod_{\infty} $, which basically uses the triangle inequality as shown below: 
\begin{align*}
\indinf{\mathbf{X}_{\mathcal{S}^c}^{\intercal} \mathbf{X}_{\mathcal{S}}\left(\mathbf{X}_{\mathcal{S}}^{\intercal} \mathbf{X}_{\mathcal{S}} \right)^{-1} } \leq& \indinf{ \nrbr{\frac{1}{n} \mathbf{X}_{\mathcal{S}^c}^{\intercal} \mathbf{X}_{\mathcal{S}} -\mathbf{\Sigma}_{\mathbf{x}_{\mathcal{S}^c} \mathbf{x}_\mathcal{S}} + \mathbf{\Sigma}_{\mathbf{x}_{\mathcal{S}^c} \mathbf{x}_\mathcal{S}} }\left( \frac{1}{n} \mathbf{X}_{\mathcal{S}}^{\intercal} \mathbf{X}_{\mathcal{S}} \right)^{-1} } \\
\leq &\indinf{\nrbr{\frac{1}{n} \mathbf{X}_{\mathcal{S}^c}^{\intercal} \mathbf{X}_{\mathcal{S}} -\mathbf{\Sigma}_{\mathbf{x}_{\mathcal{S}^c} \mathbf{x}_\mathcal{S}}  }\left( \frac{1}{n}\mathbf{X}_{\mathcal{S}}^{\intercal} \mathbf{X}_{\mathcal{S}} \right)^{-1}} + \indinf{ \mathbf{\Sigma}_{\mathbf{x}_{\mathcal{S}^c} \mathbf{x}_\mathcal{S}} \left( \frac{1}{n}\mathbf{X}_{\mathcal{S}}^{\intercal} \mathbf{X}_{\mathcal{S}} \right)^{-1}} \\
= &\indinf{\nrbr{\frac{1}{n} \mathbf{X}_{\mathcal{S}^c}^{\intercal} \mathbf{X}_{\mathcal{S}} -\mathbf{\Sigma}_{\mathbf{x}_{\mathcal{S}^c} \mathbf{x}_\mathcal{S}}  } \nrbr{\left(\frac{1}{n} \mathbf{X}_{\mathcal{S}}^{\intercal} \mathbf{X}_{\mathcal{S}}\right)^{-1} - \left( \mathbf{\Sigma}_{\mathbf{x}_\mathcal{S} \mathbf{x}_\mathcal{S}} \right)^{-1} + \left( \mathbf{\Sigma}_{\mathbf{x}_\mathcal{S} \mathbf{x}_\mathcal{S}} \right)^{-1}} } \\
& + \indinf{ \mathbf{\Sigma}_{\mathbf{x}_{\mathcal{S}^c} \mathbf{x}_\mathcal{S}} \nrbr{\left( \frac{1}{n} \mathbf{X}_{\mathcal{S}}^{\intercal} \mathbf{X}_{\mathcal{S}} \right)^{-1} - \left( \mathbf{\Sigma}_{\mathbf{x}_\mathcal{S} \mathbf{x}_\mathcal{S}} \right)^{-1} + \left( \mathbf{\Sigma}_{\mathbf{x}_\mathcal{S} \mathbf{x}_\mathcal{S}} \right)^{-1}}} \\
\leq & \indinf{\nrbr{\frac{1}{n} \mathbf{X}_{\mathcal{S}^c}^{\intercal} \mathbf{X}_{\mathcal{S}} -\mathbf{\Sigma}_{\mathbf{x}_{\mathcal{S}^c} \mathbf{x}_\mathcal{S}}  } \nrbr{\left(\frac{1}{n} \mathbf{X}_{\mathcal{S}}^{\intercal} \mathbf{X}_{\mathcal{S}} \right)^{-1} - \left( \mathbf{\Sigma}_{\mathbf{x}_\mathcal{S} \mathbf{x}_\mathcal{S}} \right)^{-1} } } \\
& + \indinf{\nrbr{\frac{1}{n} \mathbf{X}_{\mathcal{S}^c}^{\intercal} \mathbf{X}_{\mathcal{S}} -\mathbf{\Sigma}_{\mathbf{x}_{\mathcal{S}^c} \mathbf{x}_\mathcal{S}}  } \left( \mathbf{\Sigma}_{\mathbf{x}_\mathcal{S} \mathbf{x}_\mathcal{S}} \right)^{-1} } \\
& +  \indinf{ \mathbf{\Sigma}_{\mathbf{x}_{\mathcal{S}^c} \mathbf{x}_\mathcal{S}} \nrbr{\left( \frac{1}{n} \mathbf{X}_{\mathcal{S}}^{\intercal} \mathbf{X}_{\mathcal{S}} \right)^{-1} - \left( \mathbf{\Sigma}_{\mathbf{x}_\mathcal{S} \mathbf{x}_\mathcal{S}} \right)^{-1} }} + \indinf{ \mathbf{\Sigma}_{\mathbf{x}_{\mathcal{S}^c} \mathbf{x}_\mathcal{S}} \left( \mathbf{\Sigma}_{\mathbf{x}_\mathcal{S} \mathbf{x}_\mathcal{S}} \right)^{-1}} \\
\end{align*}
Let $\mathbf{R} = \frac{1}{n} \mathbf{X}_{\mathcal{S}^c}^{\intercal} \mathbf{X}_{\mathcal{S}}$ and $\mathbf{Q} = \frac{1}{n} \mathbf{X}_{\mathcal{S}}^{\intercal} \mathbf{X}_{\mathcal{S}}$, and hence $\E\sqbr{\mathbf{R}} = \mathbf{\Sigma}_{\mathbf{x}_{\mathcal{S}^c} \mathbf{x}_\mathcal{S}}$, $\E\sqbr{\mathbf{Q}} = \mathbf{\Sigma}_{\mathbf{x}_\mathcal{S} \mathbf{x}_\mathcal{S}}$. The above expression simplifies to: 
\begin{align*}
\indinf{\mathbf{X}_{\mathcal{S}^c}^{\intercal} \mathbf{X}_{\mathcal{S}}\left(\mathbf{X}_{\mathcal{S}}^{\intercal} \mathbf{X}_{\mathcal{S}} \right)^{-1} } &\leq \indinf{ \E\sqbr{\mathbf{R}} \nrbr{\E\sqbr{\mathbf{Q}}}^{-1}} + \indinf{\E\sqbr{\mathbf{R}} \left( \mathbf{Q}^{-1} - \nrbr{\E\sqbr{\mathbf{Q}}}^{-1}\right)} \\ 
+ & \indinf{ \nrbr{\mathbf{R} - \E\sqbr{\mathbf{R}}} \left( \mathbf{Q}^{-1} - \nrbr{\E\sqbr{\mathbf{Q}}}^{-1}\right)} + \indinf{ \nrbr{\mathbf{R} - \E\sqbr{\mathbf{R}}} \nrbr{\E\sqbr{\mathbf{Q}}}^{-1}} 
\end{align*}

\subsection{Proof of Lemma \ref{lem:m1_bound}} 
\textbf{Lemma \ref{lem:m1_bound}: } For $0\leq\delta \leq 32\xi k$, where $\xi = \max\limits_{i \in \mathcal{S}}\nrbr{\sigma \sqrt{\mathbf{\Sigma}}_{ii} + r \sqrt{\mathbf{\Sigma_e}}_{ii}} \max\limits_{j \in \mathcal{S}^c}\nrbr{\sigma \sqrt{\mathbf{\Sigma}}_{jj} + r \sqrt{\mathbf{\Sigma_e}}_{jj}}$, if $n = \Omega\nrbr{\frac{k^2 \xi^2}{\delta^2} \log(p)}$, then 
\begin{align*}
\P\sqbr{\indinf{ \frac{1}{n} \mathbf{X}_{\mathcal{S}^c}^{\intercal} \mathbf{X}_{\mathcal{S}} - \mathbf{\Sigma}_{\mathbf{x}_{\mathcal{S}^c} \mathbf{x}_\mathcal{S}} }  \leq \delta} \geq 1- \mathcal{O}\left(\frac{1}{p}\right)
\end{align*}
\begin{proof}
	We start by analyzing each entry of $\frac{1}{n} \mathbf{X}_{\mathcal{S}^c}^{\intercal} \mathbf{X}_{\mathcal{S}}$. As $\mathbf{x}= \mathbf{x}^{\star}+ \mathbf{e}_{\mathbf{x}}$, we can claim $\mathbf{x}_i \sim SG(0, \sigma \sqrt{\mathbf{\Sigma}}_{ii} + r\sqrt{\mathbf{\Sigma_e}}_{ii})$ using Lemma \ref{lem:sum_subG}. Further as $\mathbf{X}_{ki}$ and $\mathbf{X}_{kj}$  are sub-Gaussian, its product is sub-exponentially distributed, denoted by $SE(8\sqrt{2}c_{ij} , 4c_{ij})$ using Lemma \ref{lem:prod_subG} where $c_{ij}= \nrbr{\sigma \sqrt{\mathbf{\Sigma}}_{ii} + r \sqrt{\mathbf{\Sigma_e}}_{ii}} \nrbr{\sigma \sqrt{\mathbf{\Sigma}}_{jj} + r \sqrt{\mathbf{\Sigma_e}}_{jj}}$. By using properties of sub-exponential distributions, we can further claim the following for entry $(i,j)$: 
	\begin{align*}
	\nrbr{\frac{1}{n} \mathbf{X}_{\mathcal{S}^c}^{\intercal} \mathbf{X}_{\mathcal{S}} }_{ij}\sim SE\left( \frac{8\sqrt{2} \xi}{\sqrt{n}}, \frac{4\xi}{n}\right)\\
	\end{align*}
	where $\xi = \max\limits_{i \in \mathcal{S}}\nrbr{\sigma \sqrt{\mathbf{\Sigma}}_{ii} + r \sqrt{\mathbf{\Sigma_e}}_{ii}} \max\limits_{j \in \mathcal{S}^c}\nrbr{\sigma \sqrt{\mathbf{\Sigma}}_{jj} + r \sqrt{\mathbf{\Sigma_e}}_{jj}}$. Applying Lemma \ref{lem:indinf_subexp} for $\frac{1}{n} \mathbf{X}_{\mathcal{S}^c}^{\intercal} \mathbf{X}_{\mathcal{S}}$, we arrive at: 
	\begin{align*}
	\P\sqbr{\indinf{ \frac{1}{n} \mathbf{X}_{\mathcal{S}^c}^{\intercal} \mathbf{X}_{\mathcal{S}} - \mathbf{\Sigma}_{\mathbf{x}_{\mathcal{S}^c} \mathbf{x}_\mathcal{S}} }  \geq \delta} \leq 2(p-k) k  \exp\crbr{\frac{-n\delta^2}{256 k^2 \xi^2}}
	\end{align*}
	for $0\leq\delta \leq 32\xi k$. If we choose $n = \Omega\nrbr{\frac{k^2 \xi^2}{\delta^2} \log(p)}$, then we may claim:
	\begin{align*}
	\P\sqbr{\indinf{ \frac{1}{n} \mathbf{X}_{\mathcal{S}^c}^{\intercal} \mathbf{X}_{\mathcal{S}} - \mathbf{\Sigma}_{\mathbf{x}_{\mathcal{S}^c} \mathbf{x}_\mathcal{S}} }  \leq \delta} \geq 1- \mathcal{O}\left(\frac{1}{p}\right).
	\end{align*}
\end{proof}

\subsection{Proof of Lemma \ref{lem:m2_bound}}
\textbf{Lemma \ref{lem:m2_bound}: }	For any $\delta > 0$, if $n = \Omega\nrbr{\frac{k^2 }{\delta^2 \nrbr{C_{\text{min}} + 2F_{\text{min}} + D_{\text{min}} }^4 } \log(p)}$
\begin{align*}
\P\sqbr{\indinf{ \nrbr{\frac{1}{n} \mathbf{X}_{\mathcal{S}}^{\intercal} \mathbf{X}_{\mathcal{S}}}^{-1}  -  \nrbr{\mathbf{\Sigma}_{\mathbf{x}_\mathcal{S} \mathbf{x}_\mathcal{S}}}^{-1} } \leq \delta} \geq 1- \mathcal{O}\left(\frac{1}{p}\right) 
\end{align*}
\begin{proof}
	We start by applying norm inequalities to arrive to the spectral norm: 
	\begin{align}
	\indinf{ \nrbr{\frac{1}{n} \mathbf{X}_{\mathcal{S}}^{\intercal} \mathbf{X}_{\mathcal{S}}}^{-1}  -  \nrbr{\mathbf{\Sigma}_{\mathbf{x}_\mathcal{S} \mathbf{x}_\mathcal{S}}}^{-1} } \leq& \sqrt{k}  \norm{\nrbr{\frac{1}{n} \mathbf{X}_{\mathcal{S}}^{\intercal} \mathbf{X}_{\mathcal{S}}}^{-1}  -  \nrbr{\mathbf{\Sigma}_{\mathbf{x}_\mathcal{S} \mathbf{x}_\mathcal{S}}}^{-1}}_2 \nonumber \\
	\leq & \sqrt{k}  \norm{ \nrbr{\frac{1}{n} \mathbf{X}_{\mathcal{S}}^{\intercal} \mathbf{X}_{\mathcal{S}}}^{-1}  \nrbr{\mathbf{\Sigma}_{\mathbf{x}_\mathcal{S} \mathbf{x}_\mathcal{S}}- \frac{1}{n} \mathbf{X}_{\mathcal{S}}^{\intercal} \mathbf{X}_{\mathcal{S}}  } \nrbr{\mathbf{\Sigma}_{\mathbf{x}_\mathcal{S} \mathbf{x}_\mathcal{S}}}^{-1} }_2 \nonumber \\ 
	\leq & \sqrt{k}  \norm{ \nrbr{\frac{1}{n} \mathbf{X}_{\mathcal{S}}^{\intercal} \mathbf{X}_{\mathcal{S}}}^{-1}}_2  \norm{\frac{1}{n} \mathbf{X}_{\mathcal{S}}^{\intercal} \mathbf{X}_{\mathcal{S}}  - \mathbf{\Sigma}_{\mathbf{x}_\mathcal{S} \mathbf{x}_\mathcal{S}}}_2 \norm{ \nrbr{\mathbf{\Sigma}_{\mathbf{x}_\mathcal{S} \mathbf{x}_\mathcal{S}}}^{-1} }_2 \label{eq:m2_e1}
	\end{align}
	The term $\norm{ \nrbr{\mathbf{\Sigma}_{\mathbf{x}_\mathcal{S} \mathbf{x}_\mathcal{S}}}^{-1} }_2$ in the above equation can be bounded as shown below:
	\begin{align}
	\lambda_{\text{min}}\nrbr{\mathbf{\Sigma}_{\mathbf{x}_\mathcal{S} \mathbf{x}_\mathcal{S}}}&\geq \lambda_{\text{min}}\nrbr{\mathbf{\Sigma}_{\mathbf{x}^{\star}_\mathcal{S} \mathbf{x}^{\star}_\mathcal{S}}} + 2 \lambda_{\text{min}}\nrbr{\mathbf{\Sigma}_{\mathbf{e}_\mathcal{S} \mathbf{x}^{\star}_\mathcal{S}}} + \lambda_{\text{min}}\nrbr{\mathbf{\Sigma}_{\mathbf{e}_\mathcal{S} \mathbf{e}_\mathcal{S}}} \nonumber \\
	& = C_{\text{min}} +2F_{\text{min}} + D_{\text{min}}  \label{eq:m2_eq5}
	\end{align} 
	We use Lemma \ref{lem:sampleHess_psd} to claim $ \norm{ \nrbr{\frac{1}{n} \mathbf{X}_{\mathcal{S}}^{\intercal} \mathbf{X}_{\mathcal{S}}}^{-1}}_2 \leq \frac{2}{C_{\text{min}} +2F_{\text{min}}  + D_{\text{min}}}  $ with high probability of $1- \mathcal{O}\left(\frac{1}{p}\right) $ if $n = \Omega\nrbr{k \log(p)}$. Substituting this bound and Eq. \eqref{eq:m2_eq5} in Eq. \eqref{eq:m2_e1}:
	\begin{align}
	\indinf{ \nrbr{\frac{1}{n} \mathbf{X}_{\mathcal{S}}^{\intercal} \mathbf{X}_{\mathcal{S}}}^{-1}  -  \nrbr{\mathbf{\Sigma}_{\mathbf{x}_\mathcal{S} \mathbf{x}_\mathcal{S}}}^{-1} } \leq \sqrt{k} \frac{2}{\nrbr{C_{\text{min}} + 2F_{\text{min}} + D_{\text{min}}}^2} \norm{\frac{1}{n} \mathbf{X}_{\mathcal{S}}^{\intercal} \mathbf{X}_{\mathcal{S}}  - \mathbf{\Sigma}_{\mathbf{x}_\mathcal{S} \mathbf{x}_\mathcal{S}}}_2  \label{eq:m2_e6}
	\end{align}
	We further proceed to bound $\norm{\frac{1}{n} \mathbf{X}_{\mathcal{S}}^{\intercal} \mathbf{X}_{\mathcal{S}}  - \mathbf{\Sigma}_{\mathbf{x}_\mathcal{S} \mathbf{x}_\mathcal{S}}}_2  $  in Eq. \eqref{eq:m2_e6}:
	\begin{align}
	\norm{\frac{1}{n} \mathbf{X}_{\mathcal{S}}^{\intercal} \mathbf{X}_{\mathcal{S}}  - \mathbf{\Sigma}_{\mathbf{x}_\mathcal{S} \mathbf{x}_\mathcal{S}}}_2  \leq \lmod \frac{1}{n} \mathbf{X}_{\mathcal{S}}^{\star \intercal} \mathbf{X}_{\mathcal{S}}^{\star} - \mathbf{\Sigma}_{\mathbf{x}^{\star}_{\mathcal{S}} \mathbf{x}^{\star}_{\mathcal{S}}}\rmod_{2} + 2 \lmod \frac{1}{n} \mathbf{E}_{{\mathcal{S}}}^{\intercal} \mathbf{X}_{\mathcal{S}}^{\star} - \mathbf{\Sigma}_{\mathbf{e}_{\mathcal{S}} \mathbf{x}^{\star}_{\mathcal{S}}}\rmod_{2} + \lmod \frac{1}{n} \mathbf{E}_{{\mathcal{S}}}^{\intercal} \mathbf{E}_{\mathcal{S}} - \mathbf{\Sigma}_{\mathbf{e}_{\mathcal{S}} \mathbf{e}_{\mathcal{S}}}\rmod_{2} \label{eq:m2_eq2}
	\end{align}
	The first term in the RHS of the above equation can be easily bounded by substituting $\delta = \nrbr{C_{\text{min}} + 2F_{\text{min}} + D_{\text{min}}}^2\frac{\delta_1}{ 8\sqrt{k}} $ in Eq. \eqref{eq:eig_SS} to claim
	\begin{align}
	\P\sqbr{\lmod \frac{1}{n} \mathbf{X}_{\mathcal{S}}^{\star \intercal} \mathbf{X}_{\mathcal{S}}^{\star} - \mathbf{\Sigma}_{\mathbf{x}^{\star}_{\mathcal{S}} \mathbf{x}^{\star}_{\mathcal{S}}}\rmod_{2}  \leq \frac{\nrbr{C_{\text{min}} + 2F_{\text{min}} + D_{\text{min}} }^2}{8} \frac{\delta_1}{\sqrt{k}}   } \geq 1- \mathcal{O}\left(\frac{1}{p}\right) \label{eq:m2_eq3}
	\end{align} 
	if $n = \Omega\nrbr{\frac{k^2 }{\delta_1^2 \nrbr{C_{\text{min}} + F_{\text{min}}}^4 } + \log(p)}$. The third term, $\lmod \frac{1}{n} \mathbf{E}_{{\mathcal{S}}}^{\intercal} \mathbf{E}_{\mathcal{S}} - \mathbf{\Sigma}_{\mathbf{e}_{\mathcal{S}} \mathbf{e}_{\mathcal{S}}}\rmod_{2} $ in the RHS of Eq. \eqref{eq:m2_eq2} can also be bounded in similar manner with same sample complexity. The second term in  Eq. \eqref{eq:m2_eq2} can be bounded by substituting $\delta = \frac{\nrbr{C_{\text{min}} + 2F_{\text{min}} + D_{\text{min}} }^2 }{8   } \frac{\delta_1}{ \sqrt{k}}  $ in Theorem \ref{thm:B_2} 
	\begin{align}
	\P\sqbr{\lmod \frac{1}{n} \mathbf{E}_{{\mathcal{S}}}^{\intercal} \mathbf{X}_{\mathcal{S}}^{\star} - \mathbf{\Sigma}_{\mathbf{e}_{\mathcal{S}} \mathbf{x}^{\star}_{\mathcal{S}}}\rmod_{2} \leq  \frac{\nrbr{C_{\text{min}} + 2F_{\text{min}} + D_{\text{min}} }^2 }{8} \frac{\delta_1}{\sqrt{k}}   } \geq 1- \mathcal{O}\left(\frac{1}{p}\right)\label{eq:m2_eq4}
	\end{align} 
	if $n = \Omega\nrbr{\frac{k^2 }{\delta_1^2 \nrbr{C_{\text{min}} + 2F_{\text{min}} + D_{\text{min}} }^4 }\log(p)}$. Further, we substitute Eq. \eqref{eq:m2_eq3} and Eq. \eqref{eq:m2_eq4} in Eq. \eqref{eq:m2_eq2} to claim the following 
	\begin{align}
	\norm{\frac{1}{n} \mathbf{X}_{\mathcal{S}}^{\intercal} \mathbf{X}_{\mathcal{S}}  - \mathbf{\Sigma}_{\mathbf{x}_\mathcal{S} \mathbf{x}_\mathcal{S}}}_2 \leq \frac{\nrbr{C_{\text{min}} + 2F_{\text{min}} + D_{\text{min}} }^2}{2} \frac{\delta_1}{\sqrt{k}} \label{eq:m2_e8}
	\end{align}
	with probability $ 1- \mathcal{O}\left(\frac{1}{p}\right)$. Substituting Eq. \eqref{eq:m2_e8} in Eq. \eqref{eq:m2_e6} and replacing the dummy variable $\delta_1$ with $\delta$, we arrive at the claimed result. 
\end{proof}

\subsection{Proof of Theorem \ref{thm:B_2}}
\label{appn:B_2}
\textbf{Theorem \ref{thm:B_2}:} 
For $0 < \delta < \frac{r^2 \sigma^2  \max\limits_{j \in \mathcal{S}}\left\{ \mathbf{ \Sigma}_{jj} \right\} \sqrt{2k \sum\limits_{i \in \mathcal{S}} \mathbf{ \Sigma}^2_{\mathbf{e}_{ii}}} }{n} $, we have
\begin{align}
\mathbb{P}\left[ \lmod \frac{1}{n} \mathbf{E}^{\intercal}  \mathbf{X}^{\star} - \mathbf{\Sigma}_{\mathbf{e}_{\mathcal{S}} \mathbf{x}^{\star}_{\mathcal{S}}} \rmod_{2} \geq \delta  \right] \leq 4 \exp\left\{ \frac{-n \delta^2}{256r^2 \sigma^2  \max\limits_{j \in \mathcal{S}}\left\{ \mathbf{ \Sigma}_{jj} \right\} \sqrt{k \sum\limits_{i \in \mathcal{S}} \mathbf{ \Sigma}^2_{\mathbf{e}_{ii}}}   }\right\}\label{eq:lemB2_a}
\end{align}
\begin{proof}
	Let $\mathbf{B} = \mathbf{E}^{\intercal}  \mathbf{X}^{\star} $ and a matrix $\mathbf{Q}$ be defined as:
	\begin{align}
	\mathbf{Q} = \begin{bmatrix}
	\mathbf{0}_{k \times k} & \mathbf{B} \\
	\mathbf{B}^{\intercal}  &  \mathbf{0}_{k \times k}
	\end{bmatrix}
	\end{align}
	Using Lemma \ref{lem:QandQ2}, $\lmod \mathbf{B} \rmod_2 = \lmod \mathbf{Q} \rmod_2$. Hence we work with $\mathbf{Q}$ instead of $\mathbf{B}$. 
	
	Using Lemma 6.12 from \cite{wainwright2019high}, we have
	\begin{align}
	\mathbb{P}\left[ \lmod\mathbf{Q} - \E\sqbr{\mathbf{Q}}\rmod_{2}\geq \delta \right] 	\leq 2 \texttt{tr}\left(\Psi_{\mathbf{Q} }(\lambda)\right) e^{-\lambda \delta} \label{eq:matrix_chern}
	\end{align}
	where $\Psi_{\mathbf{Q}}$ is the moment generating function of a random matrix $\mathbf{Q} $ and can be seen as a mapping $\Psi_{\mathbf{Q}}: \mathbb{R} \rightarrow \mathcal{S}^{d \times d}$ defined as:
	\begin{align*}
	\Psi_{\mathbf{Q}}(\lambda) = \mathbb{E}\left[ e^{\lambda \mathbf{Q}}\right] = \sum_{k = 0}^{\infty } \frac{\lambda^k}{k!} \mathbb{E}\left[ \mathbf{Q}^k\right]
	\end{align*} 		
	Therefore, we have to compute the moment generating function $\Psi_{\mathbf{Q}}(\lambda)$ or compute the bound for $ \texttt{tr}\left(\Psi_{\mathbf{Q} }(\lambda)\right)$ in Eq \eqref{eq:matrix_chern}. To do that, we need to study the distribution of $\mathbf{B}$. Any entry $(i,j)$ of $\mathbf{B}$ can be expressed as the sum of products of pairs of sub-Gaussian random variables:
	\begin{align*}
	\mathbf{ B}_{ij} = \frac{1}{n} \sum_{k = 1}^n \mathbf{E}_{ki} \mathbf{X}^{\star}_{kj}
	\end{align*}
	Since $\frac{\mathbf{E}_{ki} }{r \sqrt{\mathbf{ \Sigma_e}}_{ii}}$ and $\frac{\mathbf{X}^{\star}_{kj}}{\sigma \sqrt{\mathbf{ \Sigma}}_{jj}}$ are zero-mean sub-Gaussian random variables with variance proxy $1$, their product is a sub-exponential random variable with parameter $\left(8\sqrt{2}, 4\right)$ by using Lemma \ref{lem:prod_subG}. Further, we define $q_{ij}= r \sigma \sqrt{\mathbf{ \Sigma_e}_{ii}} \sqrt{\mathbf{ \Sigma}_{jj}}$ by using properties of sub-exponential distributions:
	\begin{align*}
	\mathbf{E}_{ki} \mathbf{X}^{\star}_{kj} & \sim SE\left(8\sqrt{2} q_{ij} , 4q_{ij}\right) \\
	\sum_{k = 1}^n \mathbf{E}_{ki} \mathbf{X}^{\star}_{kj}  & \sim SE\left(8\sqrt{2} q_{ij} \sqrt{n} , 4q_{ij}\right) \\
	\frac{1}{n} \sum_{k = 1}^n \mathbf{E}_{ki} \mathbf{X}^{\star}_{kj} &\sim SE\left( \frac{8\sqrt{2} q_{ij}}{\sqrt{n}}, \frac{4q_{ij}}{n}\right)\\
	\frac{1}{n} \sum_{k = 1}^n \mathbf{E}_{ki} \mathbf{X}^{\star}_{kj} &\sim SE\left( \frac{8\sqrt{2} q_{i}}{\sqrt{n}}, \frac{4q_{i}}{n}\right)
	\end{align*}
	where $q_{i} = r \sigma \sqrt{\mathbf{ \Sigma_e}_{ii}}  \nrbr{\max\limits_{j \in \mathcal{S}}\left\{\mathbf{ \Sigma}_{jj}\right\}}^{1/2}$. Therefore $\mathbf{Q}$ follows sub-exponential distribution with parameter $\left( \mathbf{V}_1, 4q\right)$, where $q = r \sigma \nrbr{\max\limits_{i \in \mathcal{S}} \mathbf{ \Sigma_e}^{1/2}_{ii}}  \nrbr{\max\limits_{j \in \mathcal{S}}\left\{\mathbf{ \Sigma}_{jj}\right\}^{1/2} }$, and $\mathbf{V}_1$ is a matrix of dimension $k \times k$: 
	\begin{align}
	\mathbf{V}_1 = \frac{128 r^2 \sigma^2 \max\limits_{j \in \mathcal{S}}\left\{\mathbf{ \Sigma}_{jj}\right\} }{n}\begin{bmatrix}
	\mathbf{ \Sigma_e}_{11}  & \mathbf{ \Sigma_e}_{11} & \ldots & \mathbf{ \Sigma_e}_{11} \\
	\mathbf{ \Sigma_e}_{22}  & \mathbf{ \Sigma_e}_{22} & \ldots & \mathbf{ \Sigma_e}_{22} \\
	\vdots &\vdots & \ddots & \vdots \\
	\mathbf{ \Sigma_e}_{kk}  & \mathbf{ \Sigma_e}_{kk}  & \ldots & \mathbf{ \Sigma_e}_{kk}  \\
	\end{bmatrix} \label{eq:V_def_c2} 
	\end{align}
	Further, it is easy to observe that the random matrix $\mathbf{Q}$ is sub-exponential with parameter $(\mathbf{V}, \frac{4q}{n})$, where $\mathbf{V}$ is described as
	\begin{align}
	\mathbf{V} = \begin{bmatrix}
	\mathbf{0} & \mathbf{V}_1  \\
	\mathbf{V}_1^{\intercal}  & \mathbf{0} 
	\end{bmatrix} \label{eq:V_def}
	\end{align}
	
	The moment generating function $\Psi_{\mathbf{Q}}(\lambda) $ can be expressed as:
	\begin{align*}
	\Psi_{\mathbf{Q}}(\lambda) \preccurlyeq e^{\frac{\lambda^2 \mathbf{V}}{2}}
	\end{align*}
	Substituting the above in Eq. \eqref{eq:matrix_chern} and by replacing $\delta$ with $n \delta$, we get: 
	\begin{align*}
	\mathbb{P}\left[ \lmod\mathbf{Q}  - \E\sqbr{\mathbf{Q}} \rmod_{2}\geq n \delta \right] 	& \leq 2 \texttt{tr}\left(e^{\frac{\lambda^2 \mathbf{V}}{2}}\right) e^{-n\lambda \delta} \\
	\mathbb{P}\left[ \lmod\frac{1 }{n}\mathbf{Q}  - \E\sqbr{\mathbf{Q}} \rmod_{2}\geq  \delta \right]  &  \leq 2 \texttt{tr}\left( \sum_{i = 0}^{\infty } \frac{\lambda^i}{2^i i!} \mathbf{V}^i\right) e^{-n\lambda \delta}\\
	&  = 2 \sum_{i = 0}^{\infty } \texttt{tr}\left( \frac{\lambda^i}{2^i i!} \left( \mathbf{UDU}^{\intercal} \right)^i\right) e^{-n \lambda \delta}\\
	&  = 2 \sum_{i = 0}^{\infty } \texttt{tr}\left( \frac{\lambda^i}{2^i i!} \left( \mathbf{D}\right)^i\right) e^{-n\lambda \delta}\\
	&  = 2 \texttt{tr}\left( \sum_{i = 0}^{\infty } \frac{\lambda^i}{2^i i!} \left( \mathbf{D}\right)^i\right) e^{-n \lambda \delta}\\
	&  = 2 e^{-n \lambda \delta} \sum_{i = 1}^{2k} e^{\frac{\lambda^2}{2}d_i} 
	\end{align*}	
	The next step is to compute the eigenvalues of the matrix $\mathbf{V} $ which is done in Lemma \ref{lem:eig_V}. It can be easily observed that $\mathbf{V} $ has only two non-zero eigenvalues equal to $\frac{c_2}{n}$, where
	\begin{align}
	c_2 =  128r^2 \sigma^2  \max\limits_{j \in \mathcal{S}}\left\{ \mathbf{ \Sigma}_{jj} \right\} \sqrt{k \sum\limits_{i \in \mathcal{S}} \mathbf{ \Sigma}^2_{\mathbf{e}_{ii}}} \label{eq:c2_t1}
	\end{align}
	If we use all the zero eigenvalues of $\mathbf{V}$ to compute $\sum_{i = 1}^{2k} e^{\frac{\lambda^2}{2} d_i} = 2\exp\left\{ n \frac{\lambda^2}{2} c_2\right\} + 2k-2$, this would lead to ultimately non-optimal bounds. Hence the trick here is that the matrix $\mathbf{V}$ can be expressed as $\mathbf{V} = \mathbf{UDU}^{\intercal} $, where $\mathbf{U}$ is a $2k \times 2$ matrix instead of $2k \times 2k$ because we know $(2k-2)$ eigenvalues of $\mathbf{V}$ are zero. If we use the first two columns of $\mathbf{U}$, then $\sum_{i = 1}^{2k} e^{\frac{\lambda^2}{2}d_i} = 2\exp\left\{ n \frac{\lambda^2}{2} c_2\right\}$. Substituting this in Eq. \eqref{eq:matrix_chern}:
	\begin{align*}
	\mathbb{P}\left[ \lmod \frac{1}{n} \mathbf{E}^{\intercal}  \mathbf{X} - \mathbf{\Sigma}_{\mathbf{e}_{\mathcal{S}} \mathbf{x}^{\star}_{\mathcal{S}}} \rmod_{2} \geq \delta  \right] \leq 2 e^{-\lambda \delta } \times 2 \exp\left\{ \frac{\lambda^2 c_2}{2n} \right\} \qquad \qquad \forall \lambda < \frac{n}{4q}
	\end{align*}
	Substituting the optimal $\lambda = \frac{n \delta }{c_2}$, we get:
	\begin{align}
	\mathbb{P}\left[ \lmod \frac{1}{n} \mathbf{E}^{\intercal}  \mathbf{X} - \mathbf{\Sigma}_{\mathbf{e}_{\mathcal{S}} \mathbf{x}^{\star}_{\mathcal{S}}} \rmod_{2} \geq \delta  \right] \leq 4 e^{\frac{-n \delta^2}{2c_2}} = 4 \exp\left\{ \frac{-n \delta^2}{256 r^2 \sigma^2 \max\limits_{j \in \mathcal{S}}\left\{ \mathbf{ \Sigma}_{jj} \right\} \sqrt{k \sum\limits_{i \in \mathcal{S}} \mathbf{ \Sigma}^2_{\mathbf{e}_{ii}}}   }\right\}
	\end{align}
	for $0 < \delta < 32 \sqrt{k} r \sigma  \max\limits_{j \in \mathcal{S}}\left\{ \mathbf{ \Sigma}_{jj} \right\} \sqrt{ \sum_{i=1}^{k} \mathbf{ \Sigma}^2_{\mathbf{e}_{ii}}}$. Hence, a slightly simplified version of $\delta$ can be $0 < \delta < 32 \sqrt{k}r \sigma \nrbr{\max\limits_{i \in \mathcal{S}} \mathbf{ \Sigma_e}^{1/2}_{ii}}  \nrbr{\max\limits_{j \in \mathcal{S}}\left\{\mathbf{ \Sigma}_{jj}\right\}^{1/2} }$. 
	
\end{proof}

\subsection{Proof of Lemma \ref{lem:QandQ2}}
This lemma helps us to work with a symmetric matrix ($ \mathbf{M} $) instead of non-symmetric matrix ($ \mathbf{B} $). 
\textbf{Lemma \ref{lem:QandQ2}}
For matrix $\mathbf{B} \in \mathbb{R}^{k \times k}$, let $\mathbf{M}$ be defined as:
\begin{align*}
\mathbf{M} = \begin{bmatrix}
\mathbf{0}_{k \times k} & \mathbf{B} \\
\mathbf{B}^{\intercal} &  \mathbf{0}_{k \times k}
\end{bmatrix}
\end{align*}
we claim $\lmod \mathbf{B} \rmod_2 = \lmod \mathbf{M} \rmod_2$.
\begin{proof}
	Using $\mathbf{M}$ defined as above, $\mathbf{M}^2$ can be computed as:
	\begin{align*}
	\mathbf{M}^2 = \begin{bmatrix}
	\mathbf{B} \mathbf{B}^{\intercal} & \mathbf{0}_{k \times k} \\ 
	\mathbf{0}_{k \times k} & \mathbf{B}^{\intercal} \mathbf{B}  \\ 
	\end{bmatrix}
	\end{align*}
	The spectral norm of $	\mathbf{M}^2 $ can be computed as: 
	\begin{align*}
	\lmod 	\mathbf{M}^2  \rmod_{2} = \max \left\{ \lambda_{\text{max}}\left(\mathbf{B} \mathbf{B}^{\intercal}\right),  \lambda_{\text{max}}\left(\mathbf{B}^{\intercal} \mathbf{B} \right)\right\}
	\end{align*}
	From basic linear algebra properties, it is easy to observe that eigenvalues of $\mathbf{B} \mathbf{B}^{\intercal}$ and $\mathbf{B}^{\intercal} \mathbf{B}$ are the same: 
	\begin{align*}
	\mathbf{B} \mathbf{B}^{\intercal} \mathbf{x} &= \lambda \mathbf{x} \\
	\mathbf{B}^{\intercal}	\mathbf{B} \mathbf{B}^{\intercal} \mathbf{x} &= \lambda \mathbf{B}^{\intercal}\mathbf{x} \\
	\mathbf{B}^{\intercal} 	\mathbf{B} \left( \mathbf{B}^{\intercal} \mathbf{x} \right) &= \lambda \left( \mathbf{B}^{\intercal}\mathbf{x} \right) \\
	\mathbf{B}^{\intercal} 	\mathbf{B} \mathbf{y} &= \lambda \mathbf{y} 
	\end{align*}
	Using the above, we can claim, 
	\begin{align*}
	\lmod 	\mathbf{M}^2  \rmod_{2} =  \lambda_{\text{max}}\left(\mathbf{B} \mathbf{B}^{\intercal}\right)
	\end{align*}
	We also know that $\lmod \mathbf{M}  \rmod^2_{2}  = \lmod 	\mathbf{M}^2  \rmod_{2} $. Therefore $\lmod \mathbf{M}  \rmod_{2} = \sqrt{\lambda_{\text{max}}\left(\mathbf{B} \mathbf{B}^{\intercal}\right)} = \lmod\mathbf{B}  \rmod_{2}$. 
\end{proof}

\subsection{Simplification of $\hat{\mathbf{{z}}}_{\mathcal{S}^c_{t_2}} $ }
\label{sec:appdn_zsct2}
\begin{align}
\hat{\mathbf{{z}}}_{\mathcal{S}^c_{t_2}} = \mathbf{X}_{\mathcal{S}^c}^{\intercal}  \left({\mathbf{P}}/{\lambda n}\right) \left( \mathbf{e}_y - \mathbf{E_{x_{\mathcal{S}}}} \mathbf{w}_{\mathcal{S}}^{\star} \right), \qquad \text{where} \qquad \mathbf{P} = \left( \mathbf{I}_n - \mathbf{X}_{\mathcal{S}} \left(\mathbf{X}_{\mathcal{S}}^{\intercal} \mathbf{X}_{\mathcal{S}} \right)^{-1} \mathbf{X}_{\mathcal{S}}^{\intercal}  \right) 
\end{align}
Using triangle inequality and sub-multiplicative property of norms:
\begin{align*}
\norm{\hat{\mathbf{{z}}}_{\mathcal{S}^c_{t_2}}}_{\infty} &\leq \frac{1}{\lambda} \norm{\frac{1}{n} \mathbf{X}_{\mathcal{S}^c}^{\intercal} \mathbf{P} \mathbf{e}_y }_{\infty} + \frac{1}{\lambda} \norm{\frac{1}{n} \mathbf{X}_{\mathcal{S}^c}^{\intercal}  \mathbf{E_{x_{\mathcal{S}}}} \mathbf{w}^{\star}_{\mathcal{S}} }_{\infty}  + \frac{1}{\lambda} \norm{ \mathbf{X}_{\mathcal{S}^c}^{\intercal}  \mathbf{X}_{\mathcal{S}} \left(\mathbf{X}_{\mathcal{S}}^{\intercal} \mathbf{X}_{\mathcal{S}} \right)^{-1}}_{\infty}\norm{ \frac{1}{n}\mathbf{X}_{\mathcal{S}}^{\intercal} \mathbf{E_{x_{\mathcal{S}}}} \mathbf{w}^{\star}_{\mathcal{S}} }_{\infty}
\end{align*}
Further using the bound for $\norm{ \mathbf{X}_{\mathcal{S}^c}^{\intercal}  \mathbf{X}_{\mathcal{S}} \left(\mathbf{X}_{\mathcal{S}}^{\intercal} \mathbf{X}_{\mathcal{S}} \right)^{-1}}_{\infty}$ derived in section \ref{sec:Zsct1_b} or Eq. \eqref{eq:Zsc_t1_bound}:
\begin{align}
\norm{\hat{\mathbf{{z}}}_{\mathcal{S}^c_{t_2}}}_{\infty} &\leq \frac{1}{\lambda} \norm{\frac{1}{n} \mathbf{X}_{\mathcal{S}^c}^{\intercal} \mathbf{P} \mathbf{e}_y }_{\infty} + \frac{1}{\lambda} \norm{\frac{1}{n} \mathbf{X}_{\mathcal{S}^c}^{\intercal}  \mathbf{E_{x_{\mathcal{S}}}} \mathbf{w}^{\star}_{\mathcal{S}} }_{\infty}  + \frac{1}{\lambda} \nrbr{1-\frac{3\gamma}{4}}\norm{ \frac{1}{n}\mathbf{X}_{\mathcal{S}}^{\intercal} \mathbf{E_{x_{\mathcal{S}}}} \mathbf{w}^{\star}_{\mathcal{S}} }_{\infty}
\end{align}

\subsection{Proof of Lemma \ref{lem:Xe_ey_lambda}}

\textbf{ Lemma }\ref{lem:Xe_ey_lambda}: If the regularization parameter $\lambda = \lambda_1 \geq \frac{8q_1 \sigma_{e_y}}{\gamma}\sqrt{\frac{4\log(p)}{n}}$, where constant $q^2_1 = 3  \nrbr{ C_{\max} + 2F_{\max} + D_{\max} } $, then $\norm{\frac{\mathbf{X}_{\mathcal{S}^c}^{\intercal} \mathbf{P} \mathbf{e}_y}{n \lambda} }_{\infty} \leq \frac{\gamma}{8}$ with probability of at least $1-\mathcal{O}\nrbr{\frac{1}{p}}$.

\begin{proof}
	Consider the random vector of dimension $(p-k)$: 
	\begin{align}
	\mathbf{t}_1 = \frac{\mathbf{X}_{\mathcal{S}^c}^{\intercal} \mathbf{P} \mathbf{e}_y}{n \lambda} 
	\end{align}
	whose each entry is zero-mean sub-Gaussian conditioned on $\mathbf{X}$. The variance parameter for each entry is given by: 
	\begin{align}
	\sigma^2_{t_1} = \frac{1}{\lambda^2 n^2} \norm{\mathbf{X}_{\mathcal{S}^c}^{\intercal} \mathbf{P} \E\sqbr{\mathbf{e}_y \mathbf{e}^{\intercal}_y }\mathbf{P} \mathbf{X}_{\mathcal{S}^c} }_2&= \frac{\sigma^2_{e_y}}{\lambda^2 n} \norm{\frac{\mathbf{X}_{\mathcal{S}^c}^{\intercal} \mathbf{X}_{\mathcal{S}^c}}{n} }_2 \nonumber \\
	&\leq \frac{\sigma^2_{e_y}}{\lambda^2 n} \nrbr{\norm{\mathbf{ \Sigma}_{\mathbf{x}_{\mathcal{S}^c}  \mathbf{x}_{\mathcal{S}^c}} }_2 + \norm{\frac{\mathbf{X}_{\mathcal{S}^c}^{\intercal} \mathbf{X}_{\mathcal{S}^c}}{n} - \mathbf{ \Sigma}_{\mathbf{x}_{\mathcal{S}^c}  \mathbf{x}_{\mathcal{S}^c}}}_2} \nonumber \\
	& \leq \frac{3\sigma^2_{e_y}}{2\lambda^2 n} \norm{\mathbf{ \Sigma}_{\mathbf{x}_{\mathcal{S}^c}  \mathbf{x}_{\mathcal{S}^c}} }_2 \nonumber \\
	& \leq \frac{3\sigma^2_{e_y}}{2\lambda^2 n}  \nrbr{ C_{\max} + 2F_{\max} + D_{\max} }
	\end{align}
	where we have used Lemma \ref{lem:Proj_mat} in the first step. In the last step, we decompose $\norm{\frac{\mathbf{X}_{\mathcal{S}^c}^{\intercal} \mathbf{X}_{\mathcal{S}^c}}{n} - \mathbf{ \Sigma}_{\mathbf{x}_{\mathcal{S}^c}  \mathbf{x}_{\mathcal{S}^c}}}_2$ as done in Eq. \eqref{eq:m2_eq2} and further use Eq. \eqref{eq:eig_SS} and Theorem \ref{thm:B_2} to claim the resulting bound  with high probability of at least $1-\mathcal{O}\nrbr{\frac{1}{p}}$, if $n = k \log(p)$. 
	
	Further we use union-bound along with sub-Gaussian tail bounds to claim: 
	\begin{align}
	\P\sqbr{\norm{\mathbf{t}_1}_{\infty} \geq \delta} \leq 2\exp\crbr{\frac{-\delta^2}{2 \sigma^2_{t_1}} + \log(p-k)}
	\end{align}
	Substituting $\delta = \frac{\gamma}{16}$, we can claim the above state with high probability of at least $1-\mathcal{O}\nrbr{\frac{1}{p}}$ if 
	\begin{align}
	\lambda \geq \frac{\sigma_{e_y}}{\gamma}\sqrt{3  \nrbr{ C_{\max} + 2F_{\max} + D_{\max} }} \sqrt{\frac{2\log(p)}{n}}
	\end{align}
	which completes the proof. 
\end{proof}

\subsection{Proof of Lemma \ref{lem:Proj_mat}}
This lemma helps us to bound the spectral norm of $\mathbf{P}$.

\noindent \textbf{Lemma \ref{lem:Proj_mat}: }	$\mathbf{P}$ defined in Eq. \eqref{eq:projmat_def} is a projection matrix and hence $\lmod \mathbf{P} \rmod^2_{2} = 1$. 
\begin{proof}
	We use the fact that $\mathbf{P}$ is a projection matrix iff  $\mathbf{I} - \mathbf{P}$ is a projection matrix. Hence we focus only on $\mathbf{T} = \mathbf{I} - \mathbf{P} = \mathbf{X}_{\mathcal{S}} \left(\mathbf{X}_{\mathcal{S}}^{\intercal}  \mathbf{X}_{\mathcal{S}} \right)^{-1} \mathbf{X}_{\mathcal{S}}^{\intercal}   $
	\begin{align*}
	\mathbf{T}  &= \mathbf{X}_{\mathcal{S}} \left(\mathbf{X}_{\mathcal{S}}^{\intercal}  \mathbf{X}_{\mathcal{S}} \right)^{-1} \mathbf{X}_{\mathcal{S}}^{\intercal}   \\
	\mathbf{T}^2 &=  \mathbf{X}_{\mathcal{S}}\left(\mathbf{X}_{\mathcal{S}}^{\intercal}  \mathbf{X}_{\mathcal{S}} \right)^{-1} \mathbf{X}_{\mathcal{S}}^{\intercal}   \mathbf{X}_{\mathcal{S}} \left(\mathbf{X}_{\mathcal{S}}^{\intercal}  \mathbf{X}_{\mathcal{S}} \right)^{-1} \mathbf{X}_{\mathcal{S}}^{\intercal}  = \mathbf{X}_{\mathcal{S}} \left(\mathbf{X}_{\mathcal{S}}^{\intercal}  \mathbf{X}_{\mathcal{S}} \right)^{-1} \mathbf{X}_{\mathcal{S}}^{\intercal}    = \mathbf{T} 
	\end{align*} 
	Hence $\mathbf{P}$ defined in Eq. \eqref{eq:projmat_def} is a valid projection matrix. 
\end{proof}

\subsection{Proof of Lemma \ref{lem:XscEsw}}
\textbf{Lemma \ref{lem:XscEsw}: }
If $\lambda = \lambda_2 \geq  \frac{16}{\gamma} \max \crbr{\norm{\mathbf{\Sigma}_{\mathbf{x}^{\star}_{\mathcal{S}} \mathbf{e}_{\mathcal{S}}} \mathbf{w}^{\star}_{\mathcal{S}} + \mathbf{\Sigma}_{\mathbf{e}_{\mathcal{S}} \mathbf{e}_{\mathcal{S}}} \mathbf{w}^{\star}_{\mathcal{S}} }_{\infty}, q_2 \sqrt{\frac{4 \log(p)}{n}}}$, then 
\begin{align}
\P\sqbr{\norm{\frac{\mathbf{X}_{\mathcal{S}}^{\intercal}  \mathbf{E_{x_{\mathcal{S}}}} \mathbf{w}^{\star}_{\mathcal{S}}}{n\lambda}}_{\infty} \leq \frac{\gamma}{8}   } \geq 1 -\mathcal{O}\nrbr{\frac{1}{p}}
\end{align}
where $q_2  = r\sqrt{ \mathbf{w}^{\star\intercal}_{\mathcal{S}} \mathbf{\Sigma}_{\mathbf{e}_{\mathcal{S} \mathcal{S}}} \mathbf{w}^{\star}_{\mathcal{S}}}   \max\limits_{i \in \mathcal{S}^c} \nrbr{\sigma \sqrt{\mathbf{ \Sigma}}_{ii} + r \sqrt{\mathbf{ \Sigma_e}}_{ii}} $.

\begin{proof}
	Consider $\mathbf{t}_2 = \frac{1}{n} \mathbf{X}_{\mathcal{S}^c}^{\intercal}  \mathbf{E_{x_{\mathcal{S}}}} \mathbf{w}^{\star}_{\mathcal{S}} $ which is a $(p-k) \times 1$ random vector whose $i^{th}$ entry can be expressed as the mean of $n$ samples: 
	\begin{align}
	\mathbf{t}_{2i} &= \frac{1}{n} \sum_{j = 1}^n \mathbf{x}^{(j)}_{i}  \left(\sum_{l \in \mathcal{S}} \mathbf{e}^{(j)}_{\mathbf{x}_l} \mathbf{w}^{\star}_l\right) \\
	\mathbb{E}\sqbr{\mathbf{t}_{2i}} &= \frac{1}{n} \sum_{j = 1}^n \sum_{l \in \mathcal{S}} \E \sqbr{\mathbf{x}^{\star (j)}_{i}  \mathbf{e}^{(j)}_{\mathbf{x}_l} \mathbf{w}^{\star}_l } + \frac{1}{n} \sum_{j = 1}^n \sum_{l \in \mathcal{S}} \E \sqbr{\mathbf{e}^{(j)}_{\mathbf{x}_i} \mathbf{e}^{(j)}_{\mathbf{x}_l} \mathbf{w}^{\star}_l } \\
	&=  \mathbf{\Sigma}_{\mathbf{x}^{\star}_i \mathbf{e}_{\mathcal{S}}} \mathbf{w}^{\star}_{\mathcal{S}} + \mathbf{\Sigma}_{\mathbf{e}_i \mathbf{e}_{\mathcal{S}}} \mathbf{w}^{\star}_{\mathcal{S}} 
	\end{align}
	where $i \in \mathcal{S}^c$. Since $\frac{\mathbf{x}^{(j)}_{i} }{ \nrbr{\sigma \sqrt{\mathbf{ \Sigma}}_{ii} + r \sqrt{\mathbf{ \Sigma_e}}_{ii}}}$ and $\frac{\sum\limits_{l \in \mathcal{S}} \mathbf{e}^{(j)}_{\mathbf{x}_l} \mathbf{w}^{\star}_l }{\sigma_t}$, where $\sigma_t =   r\sqrt{ \mathbf{w}^{\star\intercal}_{\mathcal{S}} \mathbf{\Sigma}_{\mathbf{e}_{\mathcal{S} \mathcal{S}}} \mathbf{w}^{\star}_{\mathcal{S}}} $ are zero-mean sub-Gaussian random variables with variance proxy $1$, their product is a sub-exponential random variable with parameter $\left(8\sqrt{2}, 4\right)$ by using Lemma \ref{lem:prod_subG}. Therefore the sample mean will also be sub-exponential random variable with following parameters: 
	\begin{align*}
	\frac{1}{n \lambda} \sum_{j = 1}^n \mathbf{x}^{(j)}_{i}  \sum_{l \in \mathcal{S}} \mathbf{e}^{(j)}_{\mathbf{x}_l} \mathbf{w}^{\star}_l \sim SE\nrbr{\frac{8\sqrt{2}c_i}{\lambda\sqrt{n}}, \frac{4 c_i}{\lambda n}}
	\end{align*}
	where $c_i = \nrbr{\sigma \sqrt{\mathbf{ \Sigma}}_{ii} + r \sqrt{\mathbf{ \Sigma_e}}_{ii}} \sigma_t$. By using sub-exponential tail bounds and union bound, we further claim:
	\begin{align}
	\P\sqbr{\frac{1}{\lambda}\norm{\mathbf{t}_{2} - \E[\mathbf{t}_{2}]}_{\infty} > \delta} \leq 2\exp\nrbr{-\frac{n\delta^2 \lambda^2}{2 q^2_2} + \log(p-k)}
	\end{align} 
	for $0 < \delta \lambda \leq 32 q_2$, where $q_2  = \max\limits_{i \in \mathcal{S}^c} c_i $.
	Substituting $\delta  = \frac{\gamma}{16}$ in the above equation, we arrive at: 
	\begin{align}
	\frac{1}{\lambda}\norm{\mathbf{t}_{2} - \E[\mathbf{t}_{2}]}_{\infty} \leq \frac{\gamma}{16} \label{eq:lambda_b1subG}
	\end{align}
	with high probability of at least $1-\mathcal{O}\nrbr{\frac{1}{p}}$, if the regularization parameter satisfies: 
	\begin{align}
	\lambda \geq \frac{16q_2}{\gamma} \sqrt{\frac{2 \log(p)}{n}} \label{eq:lambda_1}
	\end{align}
	Using triangle inequality, we can claim:
	\begin{align}
	\frac{\norm{\mathbf{t}_2}_{\infty}}{\lambda} \leq \frac{1}{\lambda}\norm{\mathbf{t}_{2} - \E[\mathbf{t}_{2}]}_{\infty} + \frac{\norm{\E[\mathbf{t}_{2}]}_{\infty}}{\lambda} \leq \frac{\gamma}{16} + \frac{\gamma}{16} = \frac{\gamma}{8}
	\end{align}
	with high probability if the regularization parameter satisfies: 
	\begin{align}
	\lambda \geq \frac{16\norm{\E[\mathbf{z}_{1}]}_{\infty}}{\gamma} = \frac{16}{\gamma} \nrbr{\norm{\mathbf{\Sigma}_{\mathbf{x}^{\star}_{\mathcal{S}^c} \mathbf{e}_{\mathcal{S}}} \mathbf{w}^{\star}_{\mathcal{S}} + \mathbf{\Sigma}_{\mathbf{e}_{\mathcal{S}^c} \mathbf{e}_{\mathcal{S}}} \mathbf{w}^{\star}_{\mathcal{S}} }_{\infty}} \label{eq:lambda_2}
	\end{align}
	Combining Eq. \eqref{eq:lambda_1} and Eq. \eqref{eq:lambda_2} for the regularization parameter: 
	\begin{align}
	\lambda \geq  \frac{16}{\gamma} \max \crbr{\norm{\mathbf{\Sigma}_{\mathbf{x}^{\star}_{\mathcal{S}^c} \mathbf{e}_{\mathcal{S}}} \mathbf{w}^{\star}_{\mathcal{S}} + \mathbf{\Sigma}_{\mathbf{e}_{\mathcal{S}^c} \mathbf{e}_{\mathcal{S}}} \mathbf{w}^{\star}_{\mathcal{S}} }_{\infty}, q_2 \sqrt{\frac{4 \log(p)}{n}}}
	\end{align}
\end{proof}

\subsection{Proof of Lemma \ref{lem:sampleHess_psd}}
\textbf{Lemma \ref{lem:sampleHess_psd}} 
\textit{If assumption \ref{assum:Sigma_Xpsd} holds and $n = \Omega\left( k \log(p)\right)$, then we have $\Lambda_{\text{min}} \left( \frac{1}{n}\mathbf{X}_{\mathcal{S}}^{\intercal} \mathbf{X}_{\mathcal{S}}  \right) \geq \frac{C_{\text{min}} +2F_{\text{min}}  + D_{\text{min}}}{2} > 0$ with probability at least $1 - \mathcal{O}\left(\frac{1}{p}\right)$ }.
\begin{proof}
	The minimum eigenvalue of $\frac{1}{n}\mathbf{X}_{\mathcal{S}}^{\intercal} \mathbf{X}_{\mathcal{S}} $ can be expressed as: 
	\begin{align}
	\Lambda_{\text{min}} \left( \frac{1}{n}\mathbf{X}_{\mathcal{S}}^{\intercal} \mathbf{X}_{\mathcal{S}}  \right) & = \Lambda_{\text{min}} \left( \frac{1}{n}\mathbf{X}_{\mathcal{S}}^{\star \intercal} \mathbf{X}_{\mathcal{S}}^{\star} + \frac{1}{n} \mathbf{E}^{\intercal}_{{\mathcal{S}}} \mathbf{X}^{\star}_{\mathcal{S}} + \frac{1}{n} \mathbf{X}^{\star \intercal}_{\mathcal{S}} \mathbf{E}_{{\mathcal{S}}} + \frac{1}{n} \mathbf{E}^{\intercal}_{{\mathcal{S}}}  \mathbf{E}_{{\mathcal{S}}} \right) \nonumber \\
	&\geq \Lambda_{\text{min}} \left( \frac{1}{n}\mathbf{X}_{\mathcal{S}}^{\star \intercal} \mathbf{X}_{\mathcal{S}}^{\star}\right)   + 2 \Lambda_{\text{min}} \nrbr{ \frac{1}{n} \mathbf{E}^{\intercal}_{{\mathcal{S}}} \mathbf{X}^{\star}_{\mathcal{S}} }+ \lambda_{\text{min}}\left( \frac{1}{n} \mathbf{E}^{\intercal}_{{\mathcal{S}}} \mathbf{E}_{{\mathcal{S}}}\right) \label{eq:eigmin_decomp}
	\end{align}
	We need to further derive lower bounds for 	$\Lambda_{\text{min}} \left( \frac{1}{n}\mathbf{X}_{\mathcal{S}}^{\star \intercal} \mathbf{X}_{\mathcal{S}}^{\star}\right)  $ and $\lambda_{\text{min}}\left( \frac{1}{n} \mathbf{E}^{\intercal}_{{\mathcal{S}}} \mathbf{E}_{{\mathcal{S}}}\right)$. Substituting $\delta = \frac{1}{2}$ in Eq. \eqref{eq:mineig_Xstar} and Eq. \eqref{eq:mineig_Ex} of Lemma \ref{lem:sampleHess_psd_Xstar}, we can claim $\Lambda_{\text{min}} \left( \frac{1}{n}\mathbf{X}_{\mathcal{S}}^{\star \intercal} \mathbf{X}_{\mathcal{S}}^{\star} \right) \geq \frac{C_{\text{min}}}{2}$ and  $\lambda_{\text{min}}\left( \frac{1}{n} \mathbf{E}^{\intercal}_{{\mathcal{S}}} \mathbf{E}_{{\mathcal{S}}}\right) \geq \frac{D_{\text{min}}}{2}$ with probability $\left(1 - 2 \exp \left\{ \frac{-c_1 C^2_{\text{min}} n }{4} + k\right\}\right)$ and $\left(1 - 2 \exp \left\{ \frac{-c_2 D^2_{\text{min}} n }{4} + k\right\}\right)$  respectively. Using this information, we claim: 
	\begin{align}
	\Lambda_{\text{min}} \left( \frac{1}{n}\mathbf{X}_{\mathcal{S}}^{\intercal} \mathbf{X}_{\mathcal{S}}  \right) & \geq \frac{C_{\text{min}}}{2} + 2 \Lambda_{\text{min}} \nrbr{ \frac{1}{n} \mathbf{E}^{\intercal}_{{\mathcal{S}}} \mathbf{X}^{\star}_{\mathcal{S}} }+ \frac{D_{\text{min}}}{2}  \nonumber
	\end{align}
	To derive a lower bound for $\Lambda_{\text{min}} \nrbr{ \frac{1}{n} \mathbf{E}^{\intercal}_{{\mathcal{S}}} \mathbf{X}^{\star}_{\mathcal{S}} }$, we express it as follows: 
	\begin{align*}
	\Lambda_{\text{min}} \nrbr{ \frac{1}{n} \mathbf{E}^{\intercal}_{{\mathcal{S}}} \mathbf{X}^{\star}_{\mathcal{S}} } &\geq \lambda_{\text{min}}\nrbr{\mathbf{\Sigma}_{\mathbf{e}_{\mathcal{S}} \mathbf{x}^{\star}_{\mathcal{S}} }}  + \Lambda_{\text{min}} \nrbr{ \frac{1}{n} \mathbf{E}^{\intercal}_{{\mathcal{S}}} \mathbf{X}^{\star}_{\mathcal{S}} - \mathbf{\Sigma}_{\mathbf{e}_{\mathcal{S}} \mathbf{x}^{\star}_{\mathcal{S}} }}\\ 
	& \geq \lambda_{\text{min}}\nrbr{\mathbf{\Sigma}_{\mathbf{e}_{\mathcal{S}} \mathbf{x}^{\star}_{\mathcal{S}} }} - \norm{ \frac{1}{n} \mathbf{E}^{\intercal}_{{\mathcal{S}}} \mathbf{X}^{\star}_{\mathcal{S}} - \mathbf{\Sigma}_{\mathbf{e}_{\mathcal{S}} \mathbf{x}^{\star}_{\mathcal{S}} }}_2
	\end{align*}
	The next step is to bound $\norm{ \frac{1}{n} \mathbf{E}^{\intercal}_{{\mathcal{S}}} \mathbf{X}^{\star}_{\mathcal{S}} - \mathbf{\Sigma}_{\mathbf{e}_{\mathcal{S}} \mathbf{x}^{\star}_{\mathcal{S}} }}_2$ which is done in Theorem \ref{thm:B_2}. Substituting $\delta = \frac{\lambda_{\text{min}}\nrbr{\mathbf{\Sigma}_{\mathbf{e}_{\mathcal{S}} \mathbf{x}^{\star}_{\mathcal{S}} }} }{2}$ in Eq. \eqref{eq:lemB2}, we can claim the following with high probability 
	\begin{align}
	\Lambda_{\text{min}} \nrbr{ \frac{1}{n} \mathbf{E}^{\intercal}_{{\mathcal{S}}} \mathbf{X}^{\star}_{\mathcal{S}} } &\geq \lambda_{\text{min}}\nrbr{\mathbf{\Sigma}_{\mathbf{e}_{\mathcal{S}} \mathbf{x}^{\star}_{\mathcal{S}} }} - \frac{\lambda_{\text{min}}\nrbr{\mathbf{\Sigma}_{\mathbf{e}_{\mathcal{S}} \mathbf{x}^{\star}_{\mathcal{S}} }} }{2} = \frac{\lambda_{\text{min}}\nrbr{\mathbf{\Sigma}_{\mathbf{e}_{\mathcal{S}} \mathbf{x}^{\star}_{\mathcal{S}} }} }{2} \label{eq:mineig_EsXs}
	\end{align}
	if $n = \Omega(k \log(p))$. Hence we claim 
	\begin{align*}
	\Lambda_{\text{min}} \left( \frac{1}{n}\mathbf{X}_{\mathcal{S}}^{\intercal} \mathbf{X}_{\mathcal{S}}  \right)
	&\geq \frac{C_{\text{min}}  + 2\lambda_{\text{min}}\nrbr{\mathbf{\Sigma}_{\mathbf{e}_{\mathcal{S}} \mathbf{x}^{\star}_{\mathcal{S}} }} + D_{\text{min}}}{2} > 0 
	\end{align*}
	with probability $1 - \mathcal{O}\left(\frac{1}{p}\right)$ if $n = \Omega\left( k \log(p)\right)$. 
	
\end{proof}

\begin{lemma}
	\label{lem:eig_V}
	The two non-zero eigenvalues of the matrix $\mathbf{V}$ defined in Eq \eqref{eq:V_def} are equal to $\frac{c_2}{n}$, where $c_2 = 128 r^2 \sigma^2  \max\limits_{j \in \mathcal{S}}\left\{ \mathbf{ \Sigma}_{jj} \right\} \sqrt{k \sum\limits_{i \in \mathcal{S}} \mathbf{ \Sigma}^2_{\mathbf{e}_{ii}}} $. The rest of the $2k-2$ eigenvalues are zero.  
\end{lemma}
\begin{proof}
	We leave the multiplicative factor $\frac{128r^2 \sigma^2 \max_{j \in [p]}\left\{ \mathbf{ \Sigma}_{jj} \right\}}{n}$ aside and focus on the matrix structure now. Let $a_i = \mathbf{ \Sigma_e}_{ii}$ for the ease of notation. Hence the transformed matrix $\mathbf{V}'_1$ has the following form: 
	\begin{align*}
	\mathbf{V}^{'}_1 = \begin{bmatrix}
	a_1  & a_1 & \ldots & a_1 \\
	a_2  & a_2 & \ldots & a_2 \\
	\vdots &\vdots & \ddots & \vdots \\
	a_k  & a_k & \ldots & a_k \\
	\end{bmatrix} 
	\end{align*}
	We use the idea used in Lemma \ref{lem:QandQ2} and compute the eigenvalues of $\mathbf{V}'^{2}$ instead of $\mathbf{V}'$ directly: 
	\begin{align*}
	\mathbf{V}^{'2} &= \begin{bmatrix}
	\mathbf{V}^{'}_1 \mathbf{V}_1^{' \intercal} & \mathbf{0}_{k \times k} \\
	\mathbf{0}_{k \times k} & \mathbf{V}_1^{' \intercal}  \mathbf{V}^{'}_1 \\ 
	\end{bmatrix} \\
	\mathbf{V}^{'}_1 \mathbf{V}_1^{'\intercal} &= k\begin{bmatrix}
	a^2_1 & a_1a_2& \ldots & a_1a_k\\ 
	a_1a_2 & a^2_2& \ldots & a_2a_k \\ 
	\ldots & \ldots &  \ddots& \vdots \\ 
	a_k a_1& a_k a_2 & \ldots & a^2_k\\ 
	\end{bmatrix}
	\end{align*} 
	To compute the eigenvalues of $\mathbf{V}'^2$, we focus on $\mathbf{V}'_1 \mathbf{V}_1^{'\intercal} $ and 
	$ \mathbf{V}_1^{'\intercal} \mathbf{V}'_1$ separately. To compute the eigenvalues of $\mathbf{V}'_1 \mathbf{V}_1^{' \intercal} $, we first determine its rank by using some elementary row operations: $R_i \rightarrow R_i - R_1 \frac{a_{i}}{a_{1}}$ for $i \in [2,3,\ldots,k]$. The resulting matrix becomes:
	\begin{align*}
	k\begin{bmatrix}
	a^2_1 & a_1a_2& \ldots & a_1a_k\\ 
	0 & 0 & \ldots & 0 \\
	\vdots & \vdots & \ddots & \vdots \\
	0 & 0 &  \ldots & 0 \\
	\end{bmatrix}
	\end{align*}
	Therefore, $\mathbf{V}^{'}_1 \mathbf{V}_1^{'\intercal} $ is a rank 1 matrix and hence the one non-zero eigenvalue can be computed using the trace of the matrix, which is $\sqrt{ k \sum\limits_{i \in \mathcal{S}} a^2_i}$. By using Lemma \ref{lem:QandQ2}, we can claim that the eigenvalues of $\mathbf{V}^{'}_1 \mathbf{V}_1^{' \intercal} $ and $ \mathbf{V}_1^{' \intercal} \mathbf{V}^{'}_1$ are the same, and hence the two non-zero eigenvalues of $\mathbf{V}^{'}$ can be derived as:
	\begin{align*}
	\lambda(\mathbf{V}^{'}) = \sqrt{ k \sum\limits_{i \in \mathcal{S}} a^2_i} = \sqrt{k \sum\limits_{i \in \mathcal{S}} \mathbf{ \Sigma}^2_{\mathbf{e}_{ii}}} 
	\end{align*}
	Accounting for the scaling factor that was kept aside in the first step: 
	\begin{align*}
	\lambda(\mathbf{V}) = \frac{128 r^2 \sigma^2  \max\limits_{j \in \mathcal{S}}\left\{ \mathbf{ \Sigma}_{jj} \right\}}{n} \sqrt{k \sum\limits_{i \in \mathcal{S}} \mathbf{ \Sigma}^2_{\mathbf{e}_{ii}}} 
	\end{align*}
\end{proof}

\begin{lemma}
	\label{lem:indinf_subexp}
	Let each entry of $\mathbf{X}  \in \mathbb{R}^{k_1 \times k_2}$ be sub-exponentially distributed, denoted by $SE(\nu, \alpha)$, then for any $0\leq\delta \leq k_2 \frac{\nu^2}{\alpha}$.
	\begin{align*}
	\P\sqbr{\indinf{X - \mathbb{E}[\mathbf{X}]} > \delta} \leq 2k_1 k_2  \exp\crbr{-\frac{\delta^2}{2k_2^2 \nu^2}}.
	\end{align*} 
\end{lemma}
\begin{proof}
	We start with the use of basic norm inequalities and further use a union bound. 
	\begin{align*}
	\P\sqbr{\indinf{X - \mathbb{E}[\mathbf{X}]} > \delta} &\leq \P\sqbr{k_2\norm{\mathbf{X} - \mathbb{E}[\mathbf{X}]}_{\infty} > \delta} \\
	&\leq \P\sqbr{(\forall i \in [k_1], j \in [k_2]) \quad  |\mathbf{X}_{ij} - \mathbb{E}[\mathbf{X}_{ij}]| > \frac{\delta}{k_2}}  \\
	& \leq k_1 k_2 \P\sqbr{ |\mathbf{X}_{ij} - \mathbb{E}[\mathbf{X}_{ij}]| > \frac{\delta}{k_2}}  \\
	& \leq 2k_1 k_2  \exp\crbr{-\frac{\delta^2}{2k_2^2 \nu^2}}
	\end{align*}
	for $0\leq\delta \leq k_2 \frac{\nu^2}{\alpha}$, where we have used sub-exponential tail bounds in the last step. 
\end{proof}

\begin{lemma}
	\label{lem:XsEsw}
	If $\lambda = \lambda_3 \geq  \frac{16}{\gamma} \nrbr{1-\frac{3\gamma}{4}}\max \crbr{\norm{\mathbf{\Sigma}_{\mathbf{x}^{\star}_{\mathcal{S}} \mathbf{e}_{\mathcal{S}}} \mathbf{w}^{\star}_{\mathcal{S}} + \mathbf{\Sigma}_{\mathbf{e}_{\mathcal{S}} \mathbf{e}_{\mathcal{S}}} \mathbf{w}^{\star}_{\mathcal{S}} }_{\infty}, q_3 \sqrt{\frac{4 \log(p)}{n}}}$, then 
	\begin{align}
	\P\sqbr{\norm{\frac{\mathbf{X}_{\mathcal{S}}^{\intercal}  \mathbf{E_{x_{\mathcal{S}}}} \mathbf{w}^{\star}_{\mathcal{S}}}{n\lambda}}_{\infty} \leq \frac{\gamma}{8}  \frac{1}{\nrbr{1-\frac{3\gamma}{4}}}  } \geq 1 -\mathcal{O}\nrbr{\frac{1}{p}}
	\end{align}
	where $q_3  = r\sqrt{ \mathbf{w}^{\star\intercal}_{\mathcal{S}} \mathbf{\Sigma}_{\mathbf{e}_{\mathcal{S} \mathcal{S}}} \mathbf{w}^{\star}_{\mathcal{S}}}   \max\limits_{i \in \mathcal{S}} \nrbr{\sigma \sqrt{\mathbf{ \Sigma}}_{ii} + r \sqrt{\mathbf{ \Sigma_e}}_{ii}} $.
\end{lemma}
\begin{proof}
	The proof of this lemma is analogous to proof of Lemma \ref{lem:XsEsw}. We need to take union bound over $k$ terms only as we are working with $\mathcal{S}$. Also, we substitute $\delta = \frac{\gamma}{16} \frac{1}{\nrbr{1-\frac{3\gamma}{4}}}$ in Eq. \eqref{eq:lambda_b1subG}, which is the reason we see the scaling factor of $\nrbr{1-\frac{3\gamma}{4}}$.  
\end{proof}

\begin{lemma}
	\label{lem:sampleHess_psd_Xstar}
	If assumption \ref{assum:Sigma_Xpsd} holds, then for some $0\leq \delta <1$
	\begin{align}
	\P \left[ \Lambda_{\text{min}} \left( \frac{1}{n}\mathbf{X}_{\mathcal{S}}^{\star \intercal} \mathbf{X}_{\mathcal{S}}^{\star} \right) \leq \left( 1 - \delta \right) C_{\text{min}} \right] &\leq  2 \exp \left\{ -c_1 C^2_{\text{min}} \delta^2n  + k\right\} \label{eq:mineig_Xstar}\\
	\text{or equivalently}\quad \P \left[ \lmod \left( \frac{1}{n}\mathbf{X}_{\mathcal{S}}^{\star \intercal} \mathbf{X}_{\mathcal{S}}^{\star} \right)^{-1} \rmod_{2} \geq \frac{1}{\left( 1 - \delta \right) C_{\text{min}} }\right] &\leq  2 \exp \left\{ -c_1 C^2_{\text{min}} \delta^2n  + k\right\} \label{eq:Ainv_2}\\
	\text{and independently,}\quad \P \left[ \lambda_{\text{min}}\left( \frac{1}{n} \mathbf{E}^{\intercal}_{{\mathcal{S}}} \mathbf{E}_{{\mathcal{S}}}\right) \leq \left( 1 - \delta \right) D_{\text{min}} \right] &\leq  2 \exp \left\{ -c_2 D^2_{\text{min}} \delta^2n  + k\right\} \label{eq:mineig_Ex}
	\end{align}
	where $c_1$, $c_2$ are some positive constants. If $n = \Omega\left( k + \log(p)\right)$, the probability bound $1-2 \exp \left\{ -c_1 C^2_{\text{min}} \delta^2n  + k\right\} $ and $1-2 \exp \left\{ -c_1 D^2_{\text{min}} \delta^2n  + k\right\} $ simplify to $1  - \mathcal{O}\left( \frac{1}{p}\right)$.
\end{lemma} 
\begin{proof}
	Let $\mathbf{ A} = \frac{1}{n}\mathbf{X}_{\mathcal{S}}^{\star \intercal} \mathbf{X}_{\mathcal{S}}^{\star}$. To derive an upper bound on the maximum eigenvalue of $\mathbf{A}^{-1}$, we derive a lower bound on the minimum eigenvalue of $\mathbf{ A}$: 
	\begin{align*}
	\Lambda_{\text{min}} \left( \mathbf{ A}\right) &= \Lambda_{\text{min}} \left( \mathbf{ A} - \mathbf{ \Sigma}_{\mathcal{S} \mathcal{S}} + \mathbf{ \Sigma}_{\mathcal{S} \mathcal{S}} \right)  \\ 
	&\geq  \Lambda_{\text{min}}  \left(  \mathbf{ \Sigma}_{\mathcal{S} \mathcal{S}}  \right) -  \max \left( \Lambda_{\text{max}}  \left( \mathbf{A} - \mathbf{ \Sigma}_{\mathcal{S} \mathcal{S}} \right), - \Lambda_{\text{min}}  \left( \mathbf{A} - \mathbf{ \Sigma}_{\mathcal{S} \mathcal{S}} \right)\right) \\
	& = C_{\text{min}} - \lmod \mathbf{A} - \mathbf{ \Sigma}_{\mathcal{S} \mathcal{S}}  \rmod_{2}
	\end{align*} 
	Using Proposition 2.1 of \cite{vershynin2012close}, we can bound $\lmod \mathbf{A} - \mathbf{ \Sigma}_{\mathcal{S} \mathcal{S}}  \rmod_{2}$ as follows: 
	\begin{align}
	\mathbb{P} \left[ \lmod \mathbf{A} - \mathbf{ \Sigma}_{\mathcal{S} \mathcal{S}}  \rmod_{2} \geq \epsilon \right] \leq 2 \exp \left\{ -c\epsilon^2 n + k\right\} \label{eq:eig_SS}
	\end{align}
	where $c$ is a constant. Substituting $\epsilon = \delta C_{\text{min}} $ in the above equation, we get
	\begin{align*}
	\mathbb{P} \left[ \lmod \mathbf{A} - \mathbf{ \Sigma}_{\mathcal{S} \mathcal{S}}  \rmod_{2} \geq  \delta C_{\text{min}}  \right] \leq 2 \exp \left\{ -c  \delta^2 C^2_{\text{min}} n + k\right\} 
	\end{align*}
	Hence,  we can claim $\Lambda_{\text{min}} \left( \mathbf{ A}\right) \geq  \left( 1 - \delta \right) C_{\text{min}} $ with probability $1  -  2 \exp \left\{ -c C^2_{\text{min}} \delta^2n  + k\right\} $. If $n > C(k + \log(p))$, then we claim  $\Lambda_{\text{min}} \left( \mathbf{ A}\right) \geq \frac{C_{\text{min}} }{2}$ with probability $1  - \mathcal{O}\left( \frac{1}{p}\right)$. Therefore $\lmod \mathbf{ A}^{-1}\rmod_{2} \leq \frac{2}{C_{\text{min}}} $. 
	
	The bound on $\lambda_{\text{min}}\left( \frac{1}{n} \mathbf{E}^{\intercal}_{{\mathcal{S}}} \mathbf{E}_{{\mathcal{S}}}\right)$ can be proved using the same approach.
\end{proof}

\begin{lemma}
	\label{lem:sign_p}
	For any $a, b \in \mathbb{R}$, fix $\epsilon > 0$. If we have $|a-b| \leq \epsilon \wedge |b| > 2 \epsilon$, then $\text{sign}(a) = \text{sign}(b)$
\end{lemma}
\begin{proof}
	Consider the two cases for $|b| > 2 \epsilon$
	
	Case 1: if $b > 2 \epsilon$ and $|a-b| \leq \epsilon$, then $a \geq \epsilon$. This implies $a$ and $b$ are both positive and have the same sign. 
	
	Case 2: if $b < -2 \epsilon$ and $|a-b| \leq \epsilon$, then $a \leq -\epsilon$. This implies $a$ and $b$ are both negative and have the same sign. 
\end{proof}

\begin{lemma}
	\label{lem:sum_subG}
	Let $X \sim SG(0, \sigma_x)$ and $Y \sim SG(0, \sigma_y)$, then 
	\begin{enumerate}
		\item $X + Y \sim SG(0, \nrbr{\sigma^2_x + \sigma^2_y}^{1/2})$  if $X$ and $Y$ are mutually independent. 
		\item $X + Y \sim SG(0, (\sigma_x + \sigma_y))$  if $X$ and $Y$ are dependent.
	\end{enumerate}
	where $SG(\mu, \sigma_z)$ denotes a sub-Gaussian distribution with mean $\mu$ and parameter $\sigma_z$.
\end{lemma}
\begin{proof}
	We start with the easier case of $X$ and $Y$ being independent. We compute the moment generating function for $X + Y$: 
	\begin{align*}
	\E\sqbr{e^{\lambda (X + Y)}} = \E\sqbr{ e^{\lambda X} e^{\lambda Y} } 
	&= \E\sqbr{ e^{\lambda X}  }\E\sqbr{ e^{\lambda Y} }\\
	&\leq \exp\nrbr{\frac{\lambda^2 \sigma^2_x}{2}} \exp\nrbr{\frac{\lambda^2\sigma^2_y}{2}} = \exp\nrbr{ \frac{\lambda^2(\sigma^2_x + \sigma^2_y)}{2}} 
	\end{align*}
	which completes the proof for mutually independent random variables $X$ and $Y$. 
	
	Further proceeding to the general case and writing the moment generating function: 
	\begin{align}
	\E\sqbr{e^{\lambda (X + Y)}} = \E\sqbr{ e^{\lambda X} e^{\lambda Y} } 
	& \stackrel{\text{\tiny (i)}}{\leq} \nrbr{\E\sqbr{e^{\lambda p X}}}^{1/p} \nrbr{\E\sqbr{e^{\lambda q X}}}^{1/q} \nonumber\\
	& \leq \exp\nrbr{\frac{\lambda^2\sigma^2_xp^2}{2} \frac{1}{p}} \exp\nrbr{\frac{\lambda^2\sigma^2_yq^2}{2} \frac{1}{q}} = \exp\nrbr{\frac{\lambda^2 \nrbr{\sigma^2_xp + \sigma^2_yq} }{2} } \label{eq:exp_pq}
	\end{align}
	where (i) uses H$\ddot{\text{o}}$lder's inequality where $\frac{1}{p} + \frac{1}{q} = 1$. To upper bound the above, we optimize with respect to variable $p$ and solve:
	\begin{align*}
	\max  f(p) = \max \nrbr{\sigma^2_xp + \sigma^2_yq} = \max \nrbr{\sigma^2_xp + \sigma^2_y \frac{p}{p-1}} 
	\end{align*}
	Taking the first order derivative: 
	\begin{align*}
	\frac{df(p)}{dp} = \sigma^2_x - \sigma^2_y \frac{1}{(p-1)^2} = 0 
	\end{align*}
	which gives $p = 1 + \frac{\sigma_y}{\sigma_x}$, and therefore $q = 1 + \frac{\sigma_x}{\sigma_y}$.  Substituting this in Eq. \eqref{eq:exp_pq}, we arrive at: 
	\begin{align*}
	\E\sqbr{e^{\lambda (X + Y)}} \leq \exp\nrbr{\frac{\lambda^2 \nrbr{\sigma^2_x +  \sigma_x \sigma_y + \sigma^2_y + \sigma_x \sigma_y } }{2} } = \exp\nrbr{\frac{\lambda^2 \nrbr{\sigma_x + \sigma_y }^2 }{2} } 
	\end{align*}
	which completes the proof for the general case. 
\end{proof}

\begin{lemma}
	\label{lem:prod_subG}
	Let $X \sim SG(0,1)$ and $Y \sim SG(0,1)$, then the product
	\begin{enumerate}
		\item $X  Y \sim SE(4\sqrt{2}, 4)$ if $X$ and $Y$ are independent 
		\item $X  Y \sim  SE(8\sqrt{2}, 4)$ if $X$ and $Y$ are dependent.  
	\end{enumerate}
	where $SG(\mu,\sigma_z)$ denotes a sub-Gaussian distribution with mean $\mu$ and parameter $\sigma_z$, and $SE(\nu,\alpha)$ denotes a sub-exponential distribution with parameters $\nu,\alpha$.
\end{lemma}
\begin{proof}
	We first start with the case of mutually independent $X$ and $Y$. Their product can be expressed as: 
	\begin{align}
	XY = \frac{\nrbr{X + Y}^2- \nrbr{X - Y}^2 }{4}
	\end{align}
	So, we derive the distribution of $\nrbr{X + Y}^2$ and $\nrbr{X - Y}^2$. We use Lemma \ref{lem:sum_subG} to derive the distribution for the sum of a pair of independent random variables
	\begin{align*}
	X + Y \sim SG(0, \sqrt{2})
	\end{align*}
	Further, by scaling of sub-Gaussian random variables, we claim: 
	\begin{align*}
	\frac{X + Y}{\sqrt{2}} \sim SG(0, 1) 
	\end{align*}
	In the next step, we use Lemma 8 from \cite{barik2019provable} to derive the distribution of the square of a sub-Gaussian random variable: 
	\begin{align*}
	\nrbr{\frac{X + Y}{\sqrt{2}}}^2 \sim SE(4\sqrt{2}, 4)
	\end{align*}
	In a similar manner, we can claim the following for the difference of two sub-Gaussian random variables: 
	\begin{align*}
	\nrbr{\frac{X - Y}{\sqrt{2}}}^2 \sim SE(4\sqrt{2}, 4) 
	\end{align*}
	By scaling of sub-exponential random variables, we claim: 
	\begin{align*}
	\nrbr{X + Y}^2 &\sim SE(8\sqrt{2}, 8) \\
	\nrbr{X - Y}^2 &\sim SE(8\sqrt{2}, 8) 
	\end{align*}
	To derive the distribution of the sum of $\nrbr{X + Y}^2$ and $\nrbr{X - Y}^2$, we use Lemma \ref{lem:sum_subG} for dependent variables: 
	\begin{align*}
	\nrbr{X + Y}^2- \nrbr{X - Y}^2 \sim SE(16\sqrt{2}, 8) 
	\end{align*}
	Further, by scaling of sub-exponential random variables:
	\begin{align*}
	\frac{\nrbr{X + Y}^2-\nrbr{X - Y}^2}{4} &\sim SE(4\sqrt{2}, 2) 
	\end{align*}
	This completes the proof for the first claim of the lemma. Proceeding in a similar manner for the general case, we use Lemma \ref{lem:sum_subG} for dependent variables to claim the following: 
	\begin{align*}
	X + Y \sim SG(0, 2)
	\end{align*}
	Proceeding in a similar manner as done for the case of independent random variables, but now for dependent random variables, we arrive at: 
	\begin{align*}
	\frac{\nrbr{X + Y}^2-\nrbr{X - Y}^2}{4} &\sim SE(8\sqrt{2}, 4) 
	\end{align*}
\end{proof}
\subsection{Experiments}
\label{appn:exp}

\paragraph{Synthetic Data:}
Continuing the discussion in Section \ref{sec:exp} of the main manuscript, we present the experimental settings in more detail here. 

First we discuss the settings used for generating Figure \ref{fig:sigmax_0.1} shown in the main manuscript. We start with the data generation process: 
\begin{enumerate}
	\item We randomly generate the support $\mathcal{S}$ of size $k = 20$, and hence $\mathcal{S}^c = [p]\setminus\mathcal{S}$.
	\item We generate a random regression parameter vector, $\mathbf{w}^{\star}_{\mathcal{S}}$. We generate a random regression parameter vector by choosing $\mathbf{w}_i$ uniformly over $[-1, -0.1]\cup[0.1, 1]$ for $i \in \mathcal{S}$ and $\mathbf{w}_j = 0, \forall j \in \mathcal{S}^c$. 
	\item We generate the noise-free features, denoted by $\mathbf{x}^{\star} \in \mathbb{R}^p$. For the ease of analysis, we chose $\mathbf{x}_i^{\star} = \mathcal{N}(0,1), \forall i \in [p]$ and generate $n$ independent samples. The next step is to generate $y^{\star(j)}$ by using  $y^{\star(j)} = \mathbf{w}^{\star \intercal}\mathbf{x}^{\star(j)}$ for $j \in [n]$. 
	\item We corrupt the measurements using Eq. \eqref{eq:ex_assum}, where $e_y \sim \mathcal{N}(0, \sigma^2_1)$ and $\mathbf{e_x} \sim \mathcal{N}(\mathbf{0}, \sigma^2_2 \mathbf{I})$. We chose the values of $\sigma_1 = 0.05$ and $\sigma_2 = 0.1$. 
	\item Further we estimate the parameter vector, denoted by $\hat{\mathbf{w}}$ using LASSO and check if $\mathcal{S}(\hat{\mathbf{w}}) =  \mathcal{S} (\mathbf{w}^{\star})$ by setting $\lambda$ twice of the lower bound derived in Eq. \eqref{eq:lambda_f1}.
	\item We repeat the above five steps 200 times and count the number of success for $\mathcal{S}(\hat{\mathbf{w}}) = \mathcal{S} (\mathbf{w}^{\star})$ in step 5, which helps to compute the probability of success. 
	\item We repeat the above six steps  for different values of $n$ for a given value of $p$. We consider a rescaled sample size $\frac{n}{\log(p)} \in \{a, b\}$, where $a$, $b$ are some constants. 
	\item We repeat all the seven steps for different values of $p \in \crbr{128, 256, 512}$. 
\end{enumerate}
Note that for the plot in Figure \ref{fig:sigmax_0.1}, we have $a = 25$ and $b = 1250$. The plot basically shows that the probability of support recovery increases as we increase the number of samples. Note that the probability reaches $1$ when the rescaled sample size $\frac{n}{\log(p)} = 1150$. More importantly, the plot for each value of $p$ overlaps which confirms the hypothesis of sample complexity being logarithmic in the dimension of the regression vector. 

\begin{figure}
	\centerline{\includegraphics[width=0.5\columnwidth]{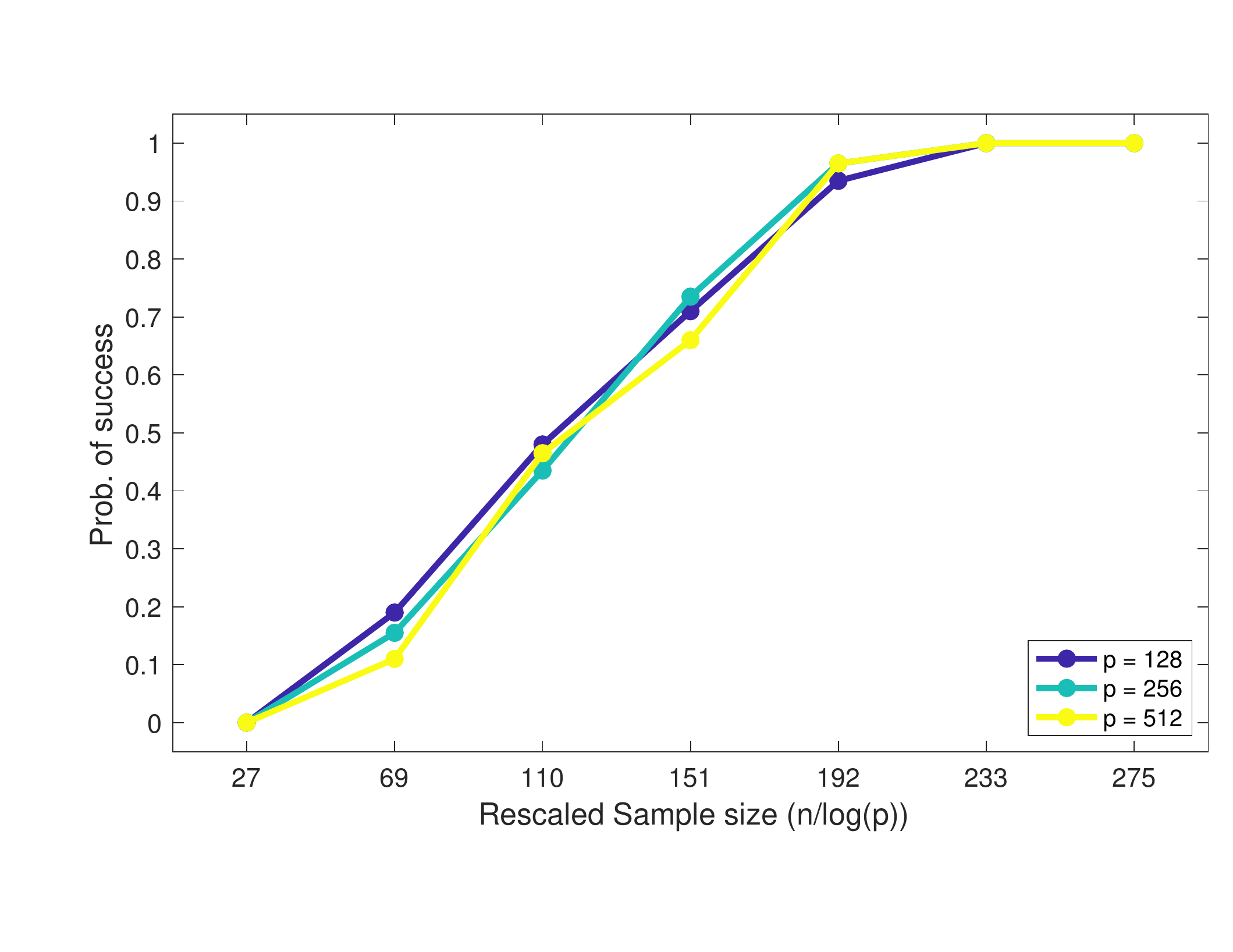}}
	\caption{Probability of support recovery vs rescaled sample size for $\mathbf{e_x} = \mathbf{0}$}
	\label{fig:sigmax_MW}
\end{figure}

We compare our results to the classical support recovery problem with no adversarial attack \cite{wainwright2009sharp} by making $\mathbf{e_x} = \mathbf{0}$ via $\sigma_2 = 0$ in our experiments. We repeat the above experiment for the same value $k$ and $p$ with different values of sample size. The results are presented in Figure \ref{fig:sigmax_MW}, which shows the similar trajectory as in Figure \ref{fig:sigmax_0.1}. The key difference is that we reach the probability of $1$ in Figure \ref{fig:sigmax_MW} when the rescaled sample size $\frac{n}{\log(p)} = 275$ which was $1150$ for Figure \ref{fig:sigmax_0.1}. Comparison of Figure \ref{fig:sigmax_0.1} and Figure \ref{fig:sigmax_MW} helps us to understand the effect of an adversary. 
\begin{figure}[H]
	\centerline{\includegraphics[width=0.5\columnwidth]{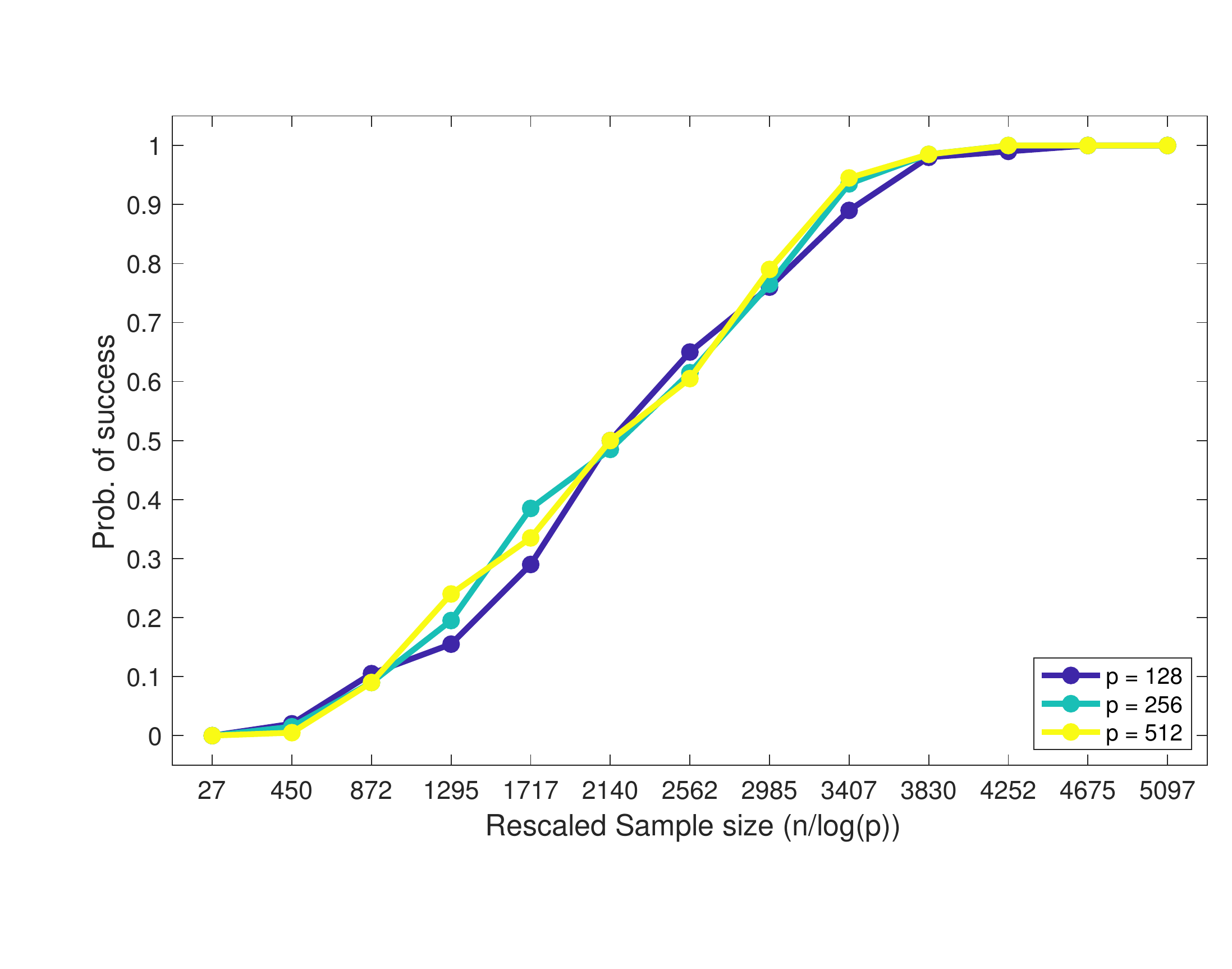}}
	\caption{Probability of support recovery vs rescaled sample size for $\sigma_2 = 0.2$}
	\label{fig:sigmax_02}
\end{figure}

In order to understand the effect of $\mathbf{e_x}$, we increase $\sigma_2 = 0.2$ in the step 4 of the procedure mentioned above. Note that we have doubled $\sigma_2$ as compared to our default analysis. From Eq. \eqref{eq:lambda_f1}, we can observe that for a constant lower bound on $\lambda$, the sample size ($n$) has to increase linearly with $\lmod \mathbf{\Sigma}_{\mathbf{e}_{\mathcal{S} \mathcal{S}}} \rmod_{2}$. We are discussing the case of constant lower bound because $\min\limits_{i \in \mathcal{S}}\left|\mathbf{w}_i^{\star}\right| \geq 2 f\left(\lambda\right) $ as per the fourth claim of Theorem \ref{thm:main}. As we double $\sigma_2 $, $\lmod \mathbf{\Sigma}_{\mathbf{e}_{\mathcal{S} \mathcal{S}}} \rmod_{2}$ will increase $4$ times and hence the minimum the number of samples required should also increase $4$ times. 

We actually observe this phenomena in our experiments. The results for $\sigma_2 = 0.2$ are presented in Figure \ref{fig:sigmax_02}. Note that the probability of success reaches 1 when the rescaled sample size $\frac{n}{\log(p)} = 4600 \approxeq 4 \times 1150$ in Figure \ref{fig:sigmax_02} which is four times the rescaled sample size needed $(1150)$ for success probability one for the case of  $\sigma_2 = 0.1$ presented in Figure \ref{fig:sigmax_0.1}. Hence our theoretical claim is justified empirically.

We further conduct experiments with more complicated forms of adversarial perturbation ($\mathbf{e}_{\mathbf{x}}$). These cases are discussed below.
\paragraph{Mixture of two distributions:} The adversarial perturbation for $j^{\text{th}}$ sample is chosen as a combination of Bernoulli distribution and Gaussian distribution as shown below: 
\begin{align}
\mathbf{e}^{(j)}_{\mathbf{x}} = \frac{r\mathbf{v}^{(j)}}{\norm{\mathbf{v}^{(j)}}_2}, \quad 
\text{where } \quad \mathbf{v}^{(j)} \sim \begin{cases} \mathbf{v}^{(j)}_i \sim 2\text{Bernoulli}(0.5) - 1 & \text{with probability } 0.5 \text{ for } i \in [p]\\
\mathcal{N}(\mathbf{0}, \mathbf{I}) & \text{with probability } 0.5 \\
\end{cases} \label{eq:advgen_comb}
\end{align}
where $r$ denotes the per sample budget for adversarial perturbation and $j \in [n]$. Compared to the previous case of all adversarial samples being drawn from Gaussian distribution, now $50\%$ of the samples will be drawn from scaled Bernoulli distribution such that each entry is $+1$ or $-1$ with equal probability. As Bernoulli distribution is bounded, we can claim it is sub-Gaussian, and the final distribution of $\mathbf{e}_\mathbf{x}$  is sub-Gaussian. Note that $\mathbf{e}_\mathbf{x}$ is designed in such a way that $\norm{\mathbf{e}^{(j)}_\mathbf{x}}_2 = r$ to respect the budget constraint. We chose $r = 0.1$ in our simulations. 

After generating the adversarial perturbation, we repeat the same exercise as described previously and present the plot for the probability of support recovery in Figure \ref{fig:Gaussbern_c1}. The plot confirms that the proposed algorithm does successful support recovery and also confirms that the sample complexity is logarithmic with respect to the size of the regression parameter vector. Further, we move to another method for adversarial perturbation generation. 
\begin{figure}[H]
	\centerline{\includegraphics[width=0.5\columnwidth]{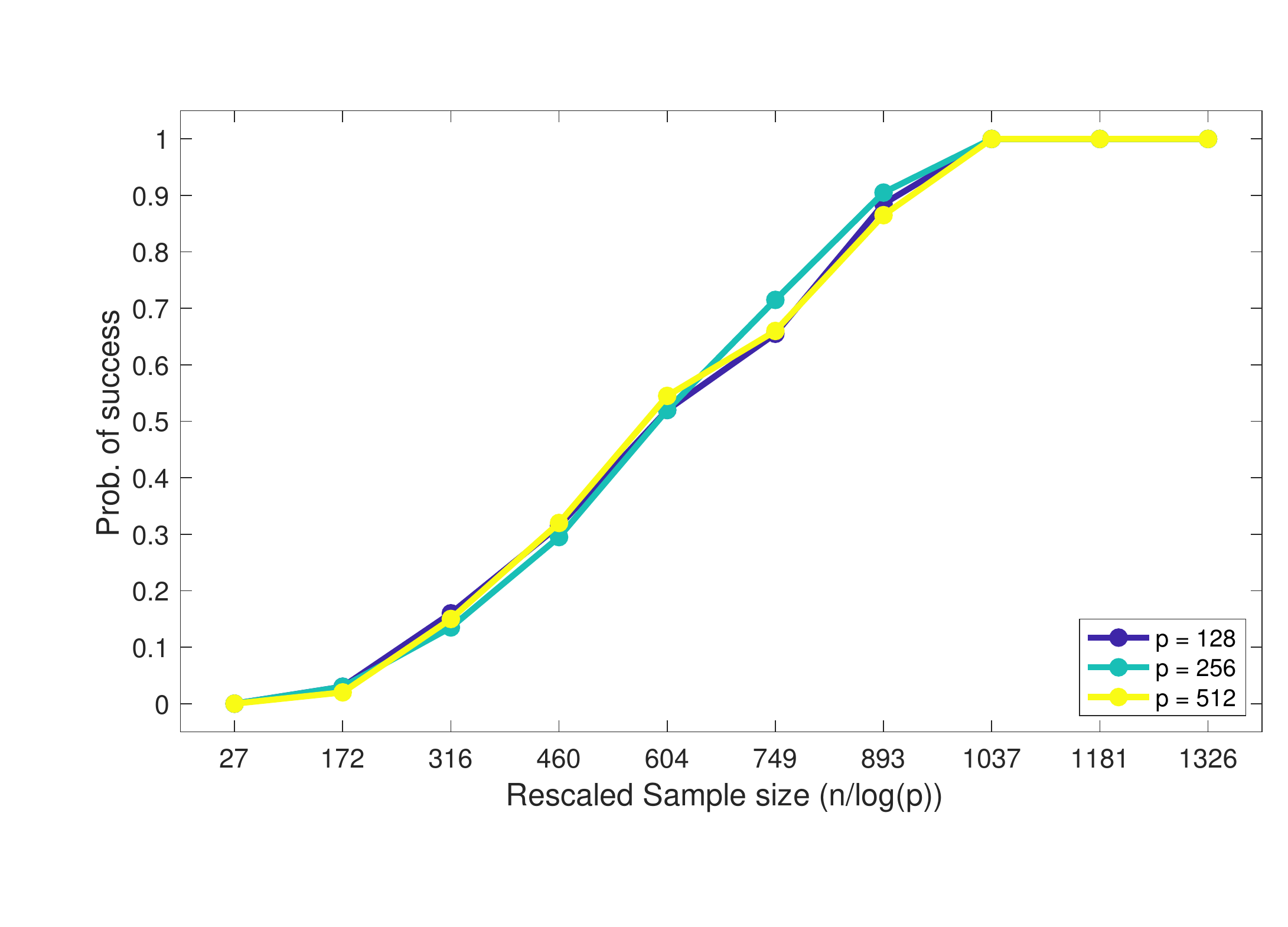}}
	\caption{Probability of support recovery vs rescaled sample size when adverarial perturbations are drawn from mixture of Bernoulli distribution and Gaussian distribution}
	\label{fig:Gaussbern_c1}
\end{figure}

\paragraph{Correlated with uncorrupted data:} In this approach, we design the adversarial perturbation in such a way that it is correlated with uncorrupted regressors ($\mathbf{x}^{\star}$) in $50\%$ of the samples as shown below: 
\begin{align}
\mathbf{e}^{(j)}_{\mathbf{x}}  = \frac{r\mathbf{v}^{(j)}}{\norm{\mathbf{v}^{(j)}}_2}, \quad
\text{where } \qquad 
\mathbf{v}^{(j)} \sim \begin{cases} \nrbr{2\text{Bernoulli}(0.5) - 1}\mathbf{x}^{\star}& \text{with probability } 0.5\\
\mathcal{N}(\mathbf{0}, \mathbf{I}) & \text{with probability } 0.5 \\
\end{cases} \label{eq:advgen_corr}
\end{align}
where $j \in [n]$. The above equation basically indicates the adversarial perturbation may be positively or negatively correlated with uncorrupted regressors with the probability of $0.5$. We further repeat the experiment as discussed at the beginning of Section \ref{sec:exp} and present the support recovery plot in Figure \ref{fig:corr_c2}. The plot verifies that the algorithm can successfully recover the support even when the adversarial perturbation is correlated with uncorrupted features. 

\begin{figure}[H]
	\centerline{\includegraphics[width=0.5\columnwidth]{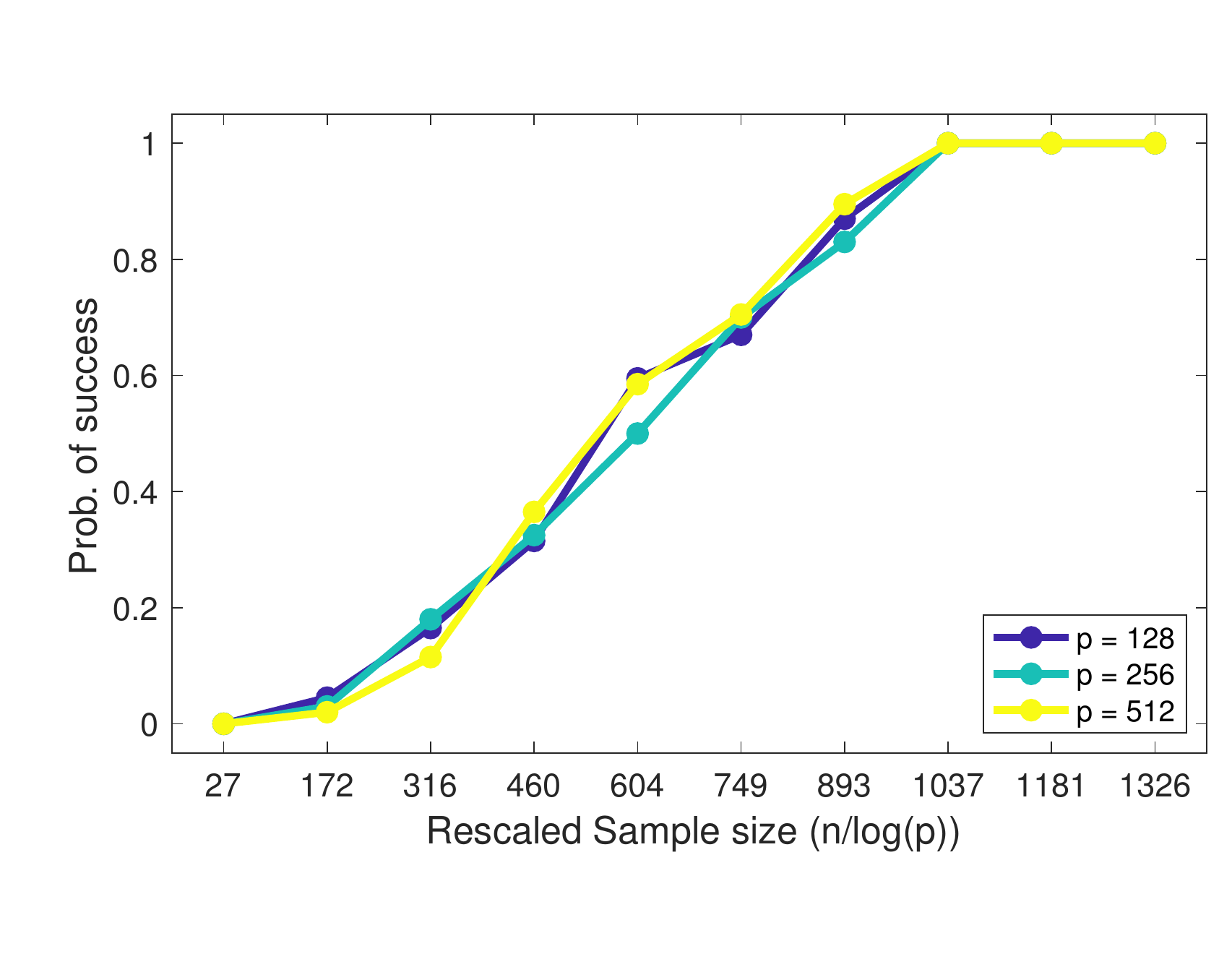}}
	\caption{Probability of support recovery vs rescaled sample size when adversarial perturbations are correlated with uncorrupted regressors}
	\label{fig:corr_c2}
\end{figure}

\paragraph{Real-World Data:}
We used the BlogFeedback dataset \cite{buza2014feedback} which contains 52397 samples and 276 features extracted from blog posts and the task is to predict how many comments a post will receive using these features. 

First, the ``true'' support is obtained by solving LASSO on the original Blogfeedback dataset \cite{buza2014feedback}. Let the ``perturbed'' support be defined as the support obtained by solving LASSO on the perturbed Blog-Feedback dataset. To construct the perturbed dataset, we add zero mean Gaussian white noise in each feature. The variance of Gaussian noise is chosen in proportion to the feature variance of the original data. After obtaining the ``true'' and ``perturbed'' support, we compute the standard F1-score defined below: 
\begin{align}
\text{Recall} & = \frac{\text{Number of elements in the ``true'' support that are in the ``perturbed'' support}}{\text{Number of elements in the ``perturbed'' support}}\\
\text{Precision} &= \frac{\text{Number of elements in the ``true'' support that are in the ``perturbed'' support}}{\text{Number of elements in the ``true'' support}}\\
\text{F1-score} &= 2 \frac{ \text{Recall} \times  \text{Precision}}{\text{Recall} + \text{Precision}} \label{eq:F1-score}
\end{align}
The F1-score of the recovered support from the perturbed data is $0.9462$, which effectively implies that the proposed approach is able to recover most of the support in real-world data as well. Further, we test the algorithm against other approaches of generating adversarial perturbation. 

We modify the approach of a mixture of two distributions in Eq. \eqref{eq:advgen_comb} by scaling with standard deviations in regressors to handle large variations in regressors as shown below:
\begin{align}
\mathbf{e}^{(j)}_{\mathbf{x}} = \frac{r\mathbf{v}^{(j)}}{\norm{\mathbf{v}^{(j)}}_2}, \quad 
\text{where } \quad \mathbf{v}^{(j)} \sim \begin{cases} \mathbf{v}^{(j)}_i \sim (2\text{Bernoulli}(0.5) - 1)\text{std}(\mathbf{x}^{\star}_i) & \text{with probability } 0.5 \text{ for } i \in [p]\\
\mathcal{N}(\mathbf{0}, \mathbf{\Sigma}) & \text{with probability } 0.5 \\
\end{cases} \label{eq:advgen_combreal}
\end{align}
where $j \in [n]$. The F1-score is reported to be $0.9393$ for $r = 1000$, proving that the algorithm can recover the support. 

We further test the algorithm against the correlated adversarial perturbation by modifying Eq. \eqref{eq:advgen_corr} to handle large variations in regressors as shown below: 
\begin{align}
\mathbf{e}^{(j)}_{\mathbf{x}}  = \frac{r\mathbf{v}^{(j)}}{\norm{\mathbf{v}^{(j)}}_2}, \quad
\text{where } \qquad 
\mathbf{v}^{(j)} \sim \begin{cases} \nrbr{2\text{Bernoulli}(0.5) - 1}\mathbf{x}^{\star}& \text{with probability } 0.5\\
\mathcal{N}(\mathbf{0}, \mathbf{\Sigma}) & \text{with probability } 0.5 \\
\end{cases} \label{eq:advgen_corr1}
\end{align}
where $j \in [n]$. We repeat the experiment in the same procedure and report the F1-score to be $0.9485$, which confirms that the proposed algorithm performs successful support recovery even when the adversarial perturbation is correlated with uncorrupted regressors. Note that F1-score is reported to be 1 in all the cases if we do not use the standard deviation scaling to normalize the adversarial perturbation. Hence, by modifying the procedure of adversarial perturbation introduction, we are solving a more challenging problem.

Note that we do not need to verify the assumptions mentioned in Section \ref{sec:assumpt} to run the algorithm. They are only needed for theoretical analysis to derive the sample complexity for support recovery. 

\subsection{Gaussian Adversarial Error}
\label{sec:proof_Gaussadv}
In this section, we prove that the sample complexity for Gaussian adversarial perturbation improves to $\Omega(k\log(p))$ as compared to the sub-Gaussian case where it is $\Omega(k^2\log(p))$ as presented in Theorem \ref{thm:main}. 
Since $\mathbf{x}^{\star} \sim \mathcal{N}(\mathbf{0}, \mathbf{\Sigma})$ and $\mathbf{e_x} \sim \mathcal{N}(\mathbf{0}, \mathbf{\Sigma}_{\mathbf{e}})$, we can claim that $\mathbf{x} \sim \mathcal{N}(\mathbf{0}, \mathbf{\Sigma}^a)$, where 
\begin{align*}
\mathbf{\Sigma}^a = \mathbf{\Sigma} + \mathbf{\Sigma}_{\mathbf{e}} 
\end{align*}
The first step is to verify the strict dual feasibility condition by bounding the infinity norm of $\hat{\mathbf{{z}}}_{\mathcal{S}^c}$ defined in Eq. \eqref{eq:Zsc_def}. For the case of the Gaussian distribution, we can 
express $\mathbf{X}_{\mathcal{S}^c}$ in Eq. \eqref{eq:Zsc_def} in terms of $\mathbf{X}_{\mathcal{S}}$ using the conditional expectation of jointly normal distribution:
\begin{align}
\mathbf{X}_{\mathcal{S}^c}^{\intercal} = \mathbf{\Sigma}^a_{\mathcal{S}^c \mathcal{S}} \left( \mathbf{\Sigma}^a_{\mathcal{S} \mathcal{S}}\right)^{-1}\mathbf{X}_{\mathcal{S}}^{\intercal} + \mathbf{E}^{\intercal}_{\mathcal{S}^c} \label{eq:Xsc_cond}
\end{align} 
where $\mathbf{E}_{\mathcal{S}^c}(i,j) \sim \mathcal{N}(0, [\mathbf{\Sigma}^a_{\mathcal{S}^c|\mathcal{S}}]_{jj})$ and  
\begin{align}
\mathbf{\Sigma}^a_{\mathcal{S}^c|\mathcal{S}} = \mathbf{\Sigma}^a_{\mathcal{S}^c \mathcal{S}^c} -  \mathbf{\Sigma}^a_{\mathcal{S}^c \mathcal{S}} \left( \mathbf{\Sigma}^a_{\mathcal{S} \mathcal{S}}\right)^{-1}\mathbf{\Sigma}^a_{\mathcal{S} \mathcal{S}^c} 
\end{align}

This simplifies the expression of $\hat{\mathbf{{z}}}_{\mathcal{S}^c}$ to: 
\begin{align}
\hat{\mathbf{{z}}}_{\mathcal{S}^c} = \mathbf{\Sigma}^a_{\mathcal{S}^c \mathcal{S}} \left( \mathbf{\Sigma}^a_{\mathcal{S} \mathcal{S}}\right)^{-1}\hat{\mathbf{{z}}}_{\mathcal{S}} + \mathbf{E}^{\intercal}_{\mathcal{S}^c} \left\{ \mathbf{X}_{\mathcal{S}} \left(\mathbf{X}_{\mathcal{S}}^{\intercal} \mathbf{X}_{\mathcal{S}}\right)^{-1} \hat{\mathbf{{z}}}_{\mathcal{S}}  + \mathbf{P} \frac{\left( \mathbf{e}_y - \mathbf{E_{x_{\mathcal{S}}}} \mathbf{w}^{\star}_{\mathcal{S}} \right)  }{\lambda n}\right\}  \label{eq:Gass_tmp1}
\end{align}
The first term can be bounded using mutual incoherence assumption. The second term is similar to Eq. 37(a) in \cite{wainwright2009sharp} and can be bounded in $\mathcal{O}\nrbr{k \log(p)}$ samples using the same approach Gaussian tail bounds and $\chi^2$ tail bounds (Appendix J in \cite{wainwright2009sharp}). This will ensure strict dual feasibility. Similarly, the uniqueness of the solution can be claimed in $\mathcal{O}\nrbr{k \log(p)}$ samples by using Lemma 9 from \cite{wainwright2009sharp}. 

For bounding  $\norm{\hat{\mathbf{w}}_{\mathcal{S}}- \mathbf{w}_{\mathcal{S}}^{\star}}_{\infty}$, we need to bound $\indinf{\mathbf{A}^{-1}}$ in $\mathcal{O}\nrbr{k \log(p)}$ samples, where $\mathbf{ A}$ is defined in Eq. \eqref{eq:w2_def}. 
This bound took $\mathcal{O}\nrbr{k^2 \log(p)}$ samples for the sub-Gaussian case. It can be bounded in $\mathcal{O}\nrbr{k \log(p)}$ samples for Gaussian case by using Lemma 5 of \cite{wainwright2009sharp}. Bounds for $\norm{\mathbf{w_1}}_{\infty}$ and $\norm{\mathbf{w_2}}_{\infty}$ can be guaranteed with high probability by choosing appropriate value of $\lambda$. Hence the sample complexity is $\mathcal{O}\nrbr{k \log(p)}$.

\end{document}